\documentclass{article} 
\usepackage{arxiv}
\usepackage{relsize}
\usepackage{multirow}
\usepackage{setspace}
\usepackage{soul}

\let\oldwedge\wedge
\renewcommand\wedge{\mathlarger{\mathlarger{\oldwedge}}}
\let\oldvee\vee
\renewcommand\vee{\mathlarger{\mathlarger{\oldvee}}}

\title{\LARGE\bf Nonlinear Meta-learning Can Guarantee Faster Rates\vspace{1em}}

\author{ 
Dimitri Meunier\thanks{Equal Contribution.} $^,$\thanks{Gatsby Computational Neuroscience Unit, University College London, London.} \\ {\footnotesize\em dimitri.meunier.21@ucl.ac.uk} \and  Zhu Li\footnotemark[1] $^{,}$\thanks{Department of Mathematics, Imperial College London, London.} \\ {\footnotesize\em michael.lzy2013@gmail.com} \\ \and Arthur Gretton\footnotemark[2] \\ {\footnotesize\em arthur.gretton@gmail.com} \\ \and Samory Kpotufe\thanks{Department of Statistics, Columbia University, New York.}\\ {\footnotesize\em samory@columbia.edu} \\ $ $ \\  
}

\date{}
\begin{document}
\maketitle
\begin{abstract}
\noindent Many recent theoretical works on \emph{meta-learning} aim to achieve guarantees in leveraging similar representational structures from related tasks towards simplifying a target task. The main aim of theoretical guarantees on the subject is to establish the extent to which convergence rates---in learning a common representation---\emph{may scale with the number $N$ of tasks} (as well as the number of samples per task). 
First steps in this setting demonstrate this property when both the shared representation amongst tasks, and task-specific regression functions, are linear. This linear setting readily reveals the benefits of aggregating tasks, e.g., via averaging arguments. In practice, however, the representation is often highly nonlinear,
introducing nontrivial biases in each task that cannot easily be averaged out as in the linear case. In the present work, we derive theoretical guarantees for meta-learning with nonlinear representations. In particular, assuming the shared nonlinearity maps to an infinite dimensional reproducing kernel Hilbert space, we show that additional biases can be mitigated with careful regularization that leverages the smoothness of task-specific regression functions, yielding improved rates that scale with the number of tasks as desired.   
\end{abstract}

\section{Introduction}
Meta-learning refers colloquially to the problem of inferring a deeper internal structure---beyond a specific task at hand, e.g., a regression task---that may be leveraged towards speeding up other similar tasks. This arises for instance in practice with neural networks where, in pre-training, multiple apparently dissimilar tasks may be aggregated to learn a \emph{representation} that enables \emph{faster} training on unseen target tasks (i.e., requiring relatively fewer target data). 

Notwithstanding the popularity of meta-learning in practice, the theoretical understanding and proper formalism for this setting is still in its early stages. We consider a common approach in the context of regression, which posits an unknown target-task function of the form $f(x) = g(\Gamma(x))$ and $N$ unknown related task-functions of the form $f_i(x) = g_i(\Gamma(x)), i \in [N]$, i.e., all sharing a common but unknown \emph{representation} $\Gamma(x)$; it is assumed that all {\it link functions} $g$ and $\{g_i\}_{i=1}^N$ are {\it simpler} --- for instance linear or at least lower-dimensional --- than the corresponding regression functions $f$ and $\{f_i\}_{i=1}^N$. As all these objects are a priori unknown, recent research has aimed to establish how the target regression problem may benefit from the $N$ related tasks. In particular, if $\Gamma(x)$ may be approximated by some $\hat \Gamma(x)$ at a rate that scales with $N$ (and the number $n$ of samples per task), then presumably, the target regression function $f$ may be subsequently learned as $\hat g(\hat \Gamma(x))$ at a faster rate commensurate with the {\it simplicity} of $g$. 

Recent theoretical results \citep{kong2020meta,du2021fewshot,tripuraneni2021provable,tian2023learning,niu2024collaborative} have provided significant new insights in this area by considering an idealized linear setting where $x\in \R^d$, $g$ and $\{g_i\}_{i=1}^N$ are linear functions in $\R^s (s \ll d$), and $\Gamma(x)$ denotes a linear projection to $\R^s$. These results show that $\Gamma$ can be learned at a rate of $\tilde O(\sqrt{ds/nN})$---under suitable subspace-distance measures, and where $\tilde O$ omits log terms ---which then allows for the target task to be learned at a rate of $\tilde O(\sqrt{s/n})\ll \tilde O(\sqrt{d/n})$. Here, it is emphasized that the representation learning rate of $\tilde O(\sqrt{ds/nN})$ scales with the number of tasks $N$ rather than just with $n$, establishing the benefit of related tasks in improving the target rate.

In practice, however, the representation $\Gamma$ is in general a nonlinear transformation of $x$, as when \emph{reproducing kernel Hilbert space} (RKHS) or neural net representations are used. While the importance of the nonlinear setting is well understood, fewer works have so far addressed this more challenging scenario \citep{maurer2016benefit,du2021fewshot}.  

In the present work, we consider the case where $\Gamma$ maps $x$, \emph{nonlinearly}, into an RKHS $\cal H$, possibly of infinite dimension; more precisely, $\Gamma$ {\it projects} the feature maps $K(x, \cdot)$ into an $s$-dimensional subspace $\cH_s$ of $\cH$. The link functions $g$ and $\{g_i\}_{i=1}^N$ are assumed to be {\it simple} in the sense that they are linear in $\Gamma$, hence we also have that $f$ and $\{f_i\}_{i=1}^N$ belong to $\cal H$. In other words, if we knew $\Gamma$ (or $\cH_s = \cH_s(\Gamma)$), the target problem would reduce to linear regression in $\R^s,$ and therefore would admit ($L_2$) convergence rates of the form $\tilde O(\sqrt{s/n})$, significantly faster than usual nonparametric rates for regression over infinite dimensional $\cal H$ (see discussion after Theorem \ref{th:inference_bound} and Corollary~\ref{cor:learning_rate}). As in the case of linear $\Gamma$ discussed above, this improved rate will turn out to require estimating $\Gamma$ at a fast rate scaling in both $N$ and $n$. 

When moving from linear to nonlinear, nonparametric $\Gamma$, a significant new challenge arises due to the bias inherent in the learning procedure. For a high-level intuition, note that a main appeal of meta-learning is that the aggregate of $N$ tasks should help reduce {\it variance} over using a single task, by carefully combining task-specific statistics computed on each of the $N$ samples; {\it crucially, such statistics ought to introduce little bias, since bias cannot be averaged out}. Task-specific biases are harder to avoid in nonparametric settings, however, if we wish to avoid overfitting task-specific statistics. This is in contrast to the case of linear projections in $\mathbb{R}^d$, where we have unbiased statistics with no overfitting (one may think e.g., of OLS). 

Fortunately, as we show in this work, nonlinear meta-learning remains possible with rate guarantees improving in both $N$ and $n$. Our approach relies on the following initial fact: if the links $\{g_i\}_{i=1}^N$ are linear in $\cH$, it easily follows that the individual regression functions $\{f_i\}_{i=1}^N$ all live in the span $\cH_s\subset \cH$ of the shared representation $\Gamma$ (see setup Section \ref{setup}). Thus, under a {\it richness assumption} where $\{f_i\}_{i=1}^N$ span $\cH_s$ \citep[extending usual assumptions in the linear case, e.g. of][]{du2021fewshot}, we may estimate $\cH_s$ by estimating the span of regularized estimates $\hat f_i$ of $f_i$. In order to guarantee fast rates that scale with $N$ and $n$, we need to \emph{under-regularize}, i.e., overfit task-specific estimates $\{\hat f_i\}_{i=1}^N$ to suitably decrease bias, at the cost of increased task-specific (hence overall) variance. Such under-regularization necessarily implies suboptimal regression in each task, but improves estimation of the representation defined by $\Gamma$. 

We demonstrate that these trade-offs may be satisfied, depending on the {\it smoothness} level of regression functions $\{f_i\}_{i=1}^N$, as captured by complementary regularity conditions on $\{f_i\}_{i=1}^N$ and the interaction between the kernel and data distributions $\{\mu_i\}_{i=1}^N$ defined on $\cX \times \R$ (see Section \ref{sec:regularity}), where we view $\cX$ and $\R$ as the input and output spaces, respectively. In the process, some interesting subtleties emerge: meta-learning benefits from {\it regularity beyond usual saturation points} that were established in traditional RKHS regression (please refer to Remark~\ref{rem:beyondregression}). This further illustrates how the meta-learning goal of estimating $\Gamma$ inherently differs from regression, even when relying on regression estimates. This is discussed in further detail in Section \ref{ta}.

Fast rates scaling in $N$ and $n$ for estimating $\cH_s = \cH_s(\Gamma)$ from $\operatorname{span}\{\hat f_i\}$ are established in Theorem \ref{th:pre_training}. This requires, among other tools, {a basic variation on Wedin’s $\sin-\Theta$ Theorem \cite{wedin1972perturbation} for infinite dimensional operators} (Proposition~\ref{prop:davis_kahan_meta}). As a consequence, we show that by operating in $\hat \cH_s$ (the estimation of $\cH_s$) for the target regression problem, we can achieve \emph{parametric} target $L_2$ rates of $\tilde O(\sqrt{s/n})$ (see Corollary~\ref{cor:learning_rate}), which are much faster than the usual nonparametric rates for $f\in \cH$. This last step requires us to establish closeness of projections onto the estimated $\hat \cH_s$ vs $\cH_s$. Moreover, when the feature map $K(x,\cdot)$ is finite dimensional, our results (see Example~\ref{ex:finite}) recover the learning rates obtained in earlier studies (e.g.~\cite{du2021fewshot,tripuraneni2021provable}), where $\Gamma$ is a linear projection.

Finally, although much of the analysis and involved operations pertain to infinite dimensional $\cH$ space, the entire approach can be instantiated in input data space via suitable representation theorems (see Section~\ref{sec:algo}). {This realization supports our theoretical findings with complementary experiments on simulated data, as detailed in Section~\ref{sec:exp}.}

\subsection*{Related Work}
Meta-learning is an umbrella term for a rich variety of learning settings, where we are provided with a set of distributions pertaining to relevant training tasks, and obtain a functional to speed learning on a target task. In this work, we focus on the case where this functional defines {\it a representation $\Gamma$ of the data}, and where the target regression function is of the form $f(x) = g(\Gamma(x))$. 
We begin this section with the closest work to our setting (namely linear and nonlinear projections $\Gamma$), then briefly touch on alternative meta-learning definitions for completeness (although these will be outside the scope of the present study).

We start with works in the {\em linear setting}, which study generalization error where $\Gamma$ is a learned linear projection  $\mathbb{R}^d \rightarrow \mathbb{R}^s$, obtained from $N$ training tasks \citep{kong2020meta,du2021fewshot,tripuraneni2021provable,thekumparampil2021statistically,konobeev2021distribution,tian2023learning,yukselfirst,niu2024collaborative}. \cite{tripuraneni2021provable} study low-dimensional linear representation learning under the assumption of isotropic inputs for all tasks, and obtain the learning rate of $\tilde O(\sqrt{ds^2/nN}+\sqrt{s/n})$ on the target task.  \cite{du2021fewshot} achieve a similar rate while relaxing the isotropic assumption with a different algorithm. In the linear representation case, they obtain an $\tilde O(\sqrt{ds/nN}+\sqrt{s/n})$ rate. \cite{kong2020meta} study a somewhat different scenario, where the number of samples per task may differ (and is smaller than the dimension $d$ of the data); the aim is to determine how many tasks must be undertaken in order to achieve consistency. 
The work of \cite{kong2020meta} is most closely related to our work, as our procedure, after linearization in $\cH$, is quite similar to their procedure in $\R^d$, notably in its reliance on outer-products of regression estimates. However, many technical issues arise in the infinite dimensional setting considered here, both on the algorithmic and analytical fronts. These are detailed in Remark \ref{rem:complinear} of Section \ref{sec:meta_learning}. {\cite{thekumparampil2021statistically} consider an alternate gradient descent algorithm, where they jointly minimize the within task loss and the aggregate loss across all tasks. Under the assumption that the data is Gaussian with the same variance across all tasks, they obtain the learning rate of $\tilde O(\sqrt{ds/nN}+\sqrt{s/n})$.} \cite{konobeev2021distribution} consider a distribution dependent analysis of meta-learning in the setting of fixed design finite dimensional linear regression, with Gaussian noise and a Gaussian parameter distribution. In the case where the covariance matrix of the parameter is assumed to be known, the authors provide  matching upper and lower bounds, which demonstrates a precise characterization of the benefit of meta-learning. While there is no theoretical analysis in the case where the covariance matrix is unknown, the authors provide a detailed description of how the EM algorithm can be employed to solve the meta-learning problem. \citet{tian2023learning} consider a generalization where tasks share similar but not identical linear representations and account for outlier tasks. \cite{niu2024collaborative,yukselfirst} also study the linear representation setting and provide refined theoretical analysis on learning the common representation.

We next consider the case where the representation $\Gamma$ is nonlinear. \cite{maurer2016benefit}  evaluate the performance of a method for learning a nonlinear representation $\Gamma\in \mathcal{F}$ which is  $s$-dimensional, addressing in particular the case of a projection onto a subspace of a reproducing kernel Hilbert space. They focus on a {\em learning to learn} (LTL) scenario, where  excess risk is evaluated   {\em in expectation over a distribution of tasks} \citep[Section 2.2][]{maurer2016benefit}: we emphasize that this is a fundamentally different objective to the performance on a specific novel test task, as in our setting. The loss they propose to minimize \citep[Eq.~1][]{maurer2016benefit} is an average over $N$ training tasks, where each task involves a different linear weighting of the common subspace projection (the work does not propose an algorithm, but concerns itself solely with the statistical analysis). Theorem 5 in \cite{maurer2016benefit} shows that for an RKHS subspace projection, one can achieve an LTL excess risk for Lipschitz losses (in expectation over the task distribution) that decreases as $\tilde O(s/\sqrt{N} + \sqrt{s/n})$. This requires $N\geq n$ in order to approach the parametric rate. \citet[note 2, p. 8]{maurer2016benefit} demonstrate that the factor $1/\sqrt{N}$ is an unavoidable consequence of the LTL setting.

\cite{du2021fewshot} consider the case of nonlinear representation learning, using the same training loss as \citet[Eq.~1]{maurer2016benefit}, but with performance evaluation on a single test task, as in our setting. Again defining $\Gamma\in \mathcal{F}$, they obtain a learning rate of $\tilde O(\mathcal{G}(\mathcal{F})/\sqrt{nN}+\sqrt{s/n})$ for the excess risk \citep[Theorem 5.1]{du2021fewshot}, where $\mathcal{G}(\cdot)$ measures the Gaussian width of  $\mathcal{F}$ (a data-dependent complexity measure, and consequently a function of $n, N$; see e.g., \cite{Maurer2014suprema}, for further details). The instantiation of $\mathcal{G}(\mathcal{F})$  for specific instances of $\mathcal{F}$   was not pursued further in this work, however \cite{Maurer2014suprema} shows that the Gaussian width is of order $\sqrt {nN}$ in $n$ and $N$, in the case where $\mathcal{F}$ is a projection onto a subspace of an RKHS with Lipschitz kernel.

The problem of learning a ``meaningful'' low-dimensional representation $\Gamma$ has also been addressed in the field of sufficient dimension reduction. \cite{fukumizu2009kernel,li2009dimension,yin2008successive} give different criteria for obtaining such $\Gamma$ and establishing consistency, however they do not address the risk analysis of downstream learning algorithms that employ  $\Gamma$.
\cite{li2011principal} introduce the so-called  principal support vector machine approach for learning both linear and nonlinear $\Gamma$. The idea is to learn a set of support vector regression functions, each mapping to different ``features'' of the output $Y$ (e.g., restrictions to intervals, nonlinear transforms). The estimator $\hat{\Gamma}$ of $\Gamma$ is then constructed from the principal components of these solutions. In the linear setting, the authors provide the $\sqrt{n}$-consistency of $\hat{\Gamma}$. \cite{wu2007regularized} provide a kernelization of sliced inverse regression, which yields a subspace $\Gamma$ in an RKHS (the so-called effective dimension reduction space). Consistency of the projection by $\hat{\Gamma}$ of an RKHS feature map $\phi (x)$ is established; and an $O(n^{-1/4})$ convergence rate is obtained, under the assumption that all $\Gamma$ components can be expressed in terms of a finite number of covariance operator eigenfunctions. The learning risk of downstream estimators using $\hat{\Gamma}$ remains to be established, however.

Outside of the regression setting, meta-learning has been studied for classification:  \cite{galanti2022generalization} investigate the generalization error in this setting, with the representation $\Gamma$ being a fully connected ReLU neural net of depth $Q$, common to all tasks. {\cite{aliakbarpour2024metalearning} study the sample complexity per task when the task-specific classifiers are halfspaces in $\R^{s}$ and the samples per task are extremely low.} Finally, there are analyses for other meta-learning schemes such as domain adaption \cite{ben2006analysis,mansour2009domain}, domain generalization \cite{blanchard2021domain} and covariate shift \cite{ma2023optimally}, as well as alternative gradient-based approaches to refine algorithms on novel test domains, e.g.,  \cite{denevi2019learning,finn2017model,finn2019online,khodak2019adaptive, meunier2021meta}. 

\section{Background \& Notations}
\label{bn}

{\bf Function Spaces \& Basic Operators.} Let $\mu$ be a probability measure on $\cX \times \R$, $\mu_{\cX}$ denotes the marginal distribution of $\mu$ on $\cX$, and $\mu(\cdot|x)$ the conditional distribution on $\R$ given $x \in \cX$. Let $K: \cX \times \cX \rightarrow \mathbb{R}$ be a symmetric and positive definite kernel function and $\mathcal{H}$ be a vector space of $\cX \rightarrow \mathbb{R}$ functions, endowed with a Hilbert space structure via an inner product $\langle\cdot, \cdot\rangle_{\mathcal{H}}$. $K$ is a reproducing kernel of $\mathcal{H}$ if and only if:~1.~$\forall x \in \cX, \phi(x)\doteq K(\cdot,x) \in \mathcal{H};~2.~\forall x \in \cX$ and $\forall f \in \mathcal{H}, f(x)=\left\langle f, \phi(x) \right\rangle_{\mathcal{H}}$. A space $\mathcal{H}$ which possesses a reproducing kernel is called a reproducing kernel Hilbert space (RKHS) \citep{berlinet2011reproducing}. $L_2(\cX,\mu_{\cX})$, abbreviated $L_2(\mu)$, denotes the Hilbert space of square-integrable functions with respect to (w.r.t.) $\mu_{\cX}$.\footnote{To simplify notations, when we  integrate over $\mu_{\cX}$ a function defined on $\cX$, we use $\E_{\mu}$ instead of $\E_{\mu_{\cX}}$.}

$\|A\|$ and $\|A\|_{HS}$ denote respectively the operator and  Hilbert-Schmidt norm of a linear operator $A$ on $\cH$. For $f,g \in \cH$, $g \otimes f\doteq\langle f, \cdot\rangle_{\cH} g$ is the generalization of the Euclidean outer product. The {covariance operator} is defined as $\Sigma \doteq \E_{X \sim \mu}[K(X,\cdot) \otimes K(X,\cdot)]$. 

We require some standard technical assumptions on the previously defined RKHS and kernel: 1. $\mathcal{H}$ is separable; this is satisfied if $\cX$ is a Polish space and $K$ is continuous \citep[Lemma 4.33][]{steinwart2008support}; 2. $\phi(x)$ is measurable for all $x \in \cX$; 3. $\sup_{x,x' \in \cX}K(x,x') \doteq \kappa^2 < \infty$. Note that those assumptions are not restrictive in practice, as well-known kernels such as the Gaussian, Laplacian and Mat{\'e}rn kernels satisfy all of the above assumptions on $\mathbb{R}^d$ \citep{sriperumbudur2011universality}.

{\bf Matrix Notation of Basic Operators.} For a set of vectors $\{u_1,\ldots, u_n\} \in \cH$, $U\doteq[u_1, \ldots, u_n]$ denotes the operator with the vectors as ``columns'', formally $U: \R^n \to \cH, \al \mapsto \sum_{i=1}^n u_i\al_i$. Its adjoint is $U^*: \cH \to \Rn, u \mapsto (\langle u_i, u \rangle_{\cH})_{i=1}^n$. 

{\bf Kernel Ridge Regression \& Regularization.} Given a data set $D=\left\{\left(x_i, y_i\right)\right\}_{i=1}^n$ independently sampled from $\mu$, kernel ridge regression aims to estimate the {\it regression function} $f_\mu=\E_{\mu}\left[Y \mid X \right]$, with the following kernel-based regularized least-squares procedure
\begin{equation} \label{eq:krr}
\hat f_{\lambda}=\underset{f \in \cH}{\operatorname{argmin}}\left\{\frac{1}{n} \sum_{i=1}^n\left(y_i-f\left(x_i\right)\right)^2 + \lambda\|f\|_{\cH}^2\right\},
\end{equation}
with $\lambda>0$ the regularization parameter. $\cR_{\mu}(f) \doteq \E_{\mu}\left[(Y - f(X))^2\right]$ is the squared expected risk and the excess risk is given by $\cE_{\mu}(f) \doteq \sqrt{\cR_{\mu}(f) - \cR_{\mu}(f_\mu)} = \E_{\mu}\left[(f(X) - f_\mu(X))^2\right]^{1/2}.$ We also introduce the population version of $\hat f_{\la}$ as
\begin{equation} \label{eq:krr_pop}
f_{\lambda}=\underset{f \in \cH}{\operatorname{argmin}}\left\{\E_{\mu}\left[(Y - f(X))^2\right] + \lambda\|f\|_{\cH}^2\right\}.
\end{equation}
The normed difference $\hat f_{\lambda}-f_{\lambda}$ is referred to as the estimation error and is a central object for the study of kernel ridge regression (see e.g., \cite{fischer2020sobolev}).

{\bf Further Notations.} For $n,m \in \mathbb{N}^*, n \leq m, [n] \doteq \{1, \ldots, n\}, [n,m] \doteq \{n, \ldots, m\}$. For two real numbers $a$ and $b$, we denote $a \vee b =\max\{a,b\}$ and $a \wedge b = \min\{a,b\}$.

\section{Nonlinear Meta-learning} \label{sec:meta_learning}

\subsection{Population Set-up}\label{setup}
We consider a setting with $N$ source distributions $\{\mu_i\}_{i \in [N]}$ defined on $\cX \times \R$, with corresponding regression functions of the form $f_i(x) = g_i (\Gamma (x))$. We are interested in minimizing the excess risk for a target distribution $\mu_T$, with regression function $f_T(x) = g_T(\Gamma(x))$. In the mostly common linear case, it is assumed that $\Gamma$ {\it projects} into a subspace of $\R^d = \cX$. However, in this manuscript, we assume that $\Gamma$ is a projection of nonlinear feature maps in an infinite dimensional space. 

\begin{assumption} \label{asst:meta}
We let $\Gamma: \cX \mapsto \cH$ be a map from $x\in \cX$ to a subspace $\cH_s$ of dimension $s \geq 1$ of an RKHS $\cH$ as follows: given a projection operator $P$ onto $\cH_s$, $\Gamma (x) \doteq P K(x, \cdot)$. Furthermore, all link functions $g_T$, $\{g_i\}_{i=1}^N$ are assumed linear $\cH \mapsto \R$, i.e., $\exists w_T, w_i \in \cH_s$ s.t. 
$g_T(\Gamma(x)) = \langle w_T, \Gamma(x) \rangle_{\cH}$, and $g_i(\Gamma(x)) = \langle w_i, \Gamma(x) \rangle_{\cH}$. 
\end{assumption}

\begin{rem}
Given an orthonormal basis (ONB) $V = [v_1, \ldots, v_s]$ of $\cH_s$, we may rewrite 
$g_T(\Gamma(x)) = \alpha_T^\top V^*K(x, \cdot)$, i.e., for $\alpha_T\in \R^s$, for an $s$-dimensional (nonlinear) representation $V^*\Gamma(x)= V^* K(x, \cdot)$ of $x$. The same is true for $\{g_i\}_{i=1}^N$ with respective $\{\alpha_i\}_{i=1}^N$. The representations are non-unique, although their corresponding regression functions and $\cH_s$ are unique (see Remark \ref{rem:RKHS} below).
\end{rem}

\begin{rem} \label{rem:h_s}
Since $P$ is self-adjoint, we have $f_T(x) \doteq \langle Pw_T, K(x, \cdot) \rangle_{\cH}$, hence by the reproducing property, $f_T = Pw_T \in \cH_s$. Similarly, we have that all $\{f_i\}_{i=1}^N$ are in $\cH_s$. 
\end{rem}

Remark~\ref{rem:h_s} indicates that $\operatorname{span}\left(\{f_i\}_{i \in [N]}\right) \subseteq \cH_s$. We therefore need the following {\it richness condition}, similar to previous works on meta-learning in the linear representation case \citep{du2021fewshot}, without which we cannot hope to learn $\cH_s$.

\begin{assumption}[Source Richness]\label{asst:rich}
 We have that $\operatorname{span}\left(\{f_i\}_{i \in [N]}\right) = \cH_s$. 
\end{assumption}

\begin{rem}\label{rem:RKHS} For any projection $P$ onto some complete subspace $\cH_s$, $\langle \cdot, PK(x, \cdot)\rangle_{\cH}$ evaluates every function in $\cH_s$ at $x$, and in fact is well-known as the {\it kernel} of the sub-RKHS defined by $\cH_s$. The same fact implies uniqueness of $\cH_s$ and in particular that it equals $\overline{\operatorname{span}} \{\Gamma(x) \doteq P K(x, \cdot)\}$. 
\end{rem} 

\subsection{Learning Set-up}\label{sec:sub_sec3.2}
In this section we present the high level ideas of our meta-learning strategy with nonlinear representation. The first step is to learn a subspace approximation $\hat \cH_s \approx \cH_s$ from source tasks. This process aims to find a suitable representation that facilitates the learning of the target task. We refer to this step as \textbf{pre-training}. The second step involves directly learning the target task within the subspace $\hat \cH_s$. We refer to this step as \textbf{inference}.

\textbf{Source Tasks - pre-training.} Our approach to approximate $\cH_s$ is inspired by \cite{kong2020meta}, who focused on finite-dimensional linear meta-learning. We extend this strategy to encompass (potentially infinite dimensional) nonlinear meta-learning. Under the source richness assumption (Assumption \ref{asst:rich}), $\cH_s$ is equal to the range of the rank-$s$ operator (see Proposition~\ref{prop:op_C} in Appendix) 
\begin{equation} \label{eq:task_cov}
    C_N \doteq \frac{1}{N} \sum_{i=1}^N f_i \otimes f_i, \qquad  \operatorname{ran} C_N = \mathcal{H}_s.
\end{equation}
Therefore, we estimate $\cH_s$ via the range of
\begin{equation} \label{eq:task_cov_approx}
    \hat C_{N,n,\lambda} \doteq \frac{1}{N} \sum_{i=1}^N \hat f'_{i,\lambda} \otimes \hat f_{i,\lambda}
\end{equation}
where $\hat f'_{i,\la}, \hat f_{i,\la}$ are  i.i.d copies of a ridge regression estimator for source task $i \in [N]$. Here, we use a data-splitting strategy to obtain the following
\begin{equation*}
    \mathbb{E}[\hat C_{N,n,\lambda}] = \frac{1}{N} \sum_{i=1}^N  \mathbb{E}[\hat f_{i,\lambda}'] \otimes   \mathbb{E}[\hat f_{i,\lambda}].
\end{equation*}
This property plays a crucial role in deriving approximation rates for $\cH_s$. Data-splitting is similarly employed in \cite{kong2020meta}. Avoiding data-splitting remains an open problem even in the finite-dimensional linear representation setting.

Each source task is learned from a dataset $\mathcal{D}_i = \{(x_{i,j}, y_{i,j})_{j=1}^{2n}\}, i \in [N]$ of i.i.d observations sampled from $\mu_{i}$,  via regularized kernel regression as in Eq.~\eqref{eq:krr},
\begin{equation} \label{eq:def_source_estim}
    \hat f_{i,\lambda} = \underset{f \in \cH}{\operatorname{argmin}} \sum_{j=1}^n \left(y_{i,j} - f(x_{i,j}) \right)^2 + n\lambda\|f\|^2_{\mathcal{H}}, \qquad \hat f'_{i,\lambda} = \underset{f \in \cH}{\operatorname{argmin}}\sum_{j=n+1}^{2n} \left(y_{i,j} - f(x_{i,j}) \right)^2 + n\lambda\|f\|^2_{\mathcal{H}}
\end{equation}
For task $i \in [N]$, let $K_i, L_i \in \R^{n \times n}$ be the Gram matrices such that $(K_i)_{j,l} = K(x_{i,j}, x_{i,l})$, $(j,l) \in [n]$ and $(L_i)_{j,l} = K(x_{i,j}, x_{i,l})$, $(j,l) \in [n+1:2n]$. Then for all $x \in \cX,$
\begin{equation} 
    \hat{f}_{i,\la}(x) = Y_i^{\top}\left(K_i + n \la I_n\right)^{-1}k_{i,x}, \qquad  \hat{f}_{i,\la}'(x) = (Y_i')^{\top}\left(L_i + n \la I_n\right)^{-1}\ell_{i,x},\label{eq:split_estimator}
\end{equation}
where $k_{i,x} = (K(x_{i,1}, x), \ldots, K(x_{i,n}, x))^{\top} \in \Rn$, $\ell_{i,x} = (K(x_{i,n+1}, x), \ldots, K(x_{i,2n}, x))^{\top} \in \Rn$, \\ $Y_i = (y_{i,1}, \ldots, y_{i,n})^{\top} \in \Rn$ and $Y_i' = (y_{i,n+1}, \ldots, y_{i,2n})^{\top} \in \Rn$.

After obtaining $\hat C_{N,n,\lambda}$, we cannot directly compare $\operatorname{ran}C_N$ to $\operatorname{ran}\hat C_{N,n,\lambda}$, since the latter is not guaranteed to be of rank $s$. We therefore consider the singular value decomposition of $\hat C_{N,n,\lambda}$:
$$
\hat C_{N,n,\lambda} = \sum_{i=1}^{N} \hat \gamma_i \hat u_i \otimes \hat v_i = \hat U \hat D \hat V^*,
$$
where $\hat \ga_1 \geq \dots \geq \hat \ga_N \geq 0$ are the singular values and stored in the diagonal matrix $\hat D \in \R^{N \times N}$. The right and left singular vectors are stored as $\hat V = [\hat v_1, \ldots, \hat v_N]$ and $\hat U = [\hat u_1, \ldots, \hat u_N]$, respectively. We use the right singular vectors to construct the approximation of $\cH_s$ as follows (note that a similar approach can be applied to the left singular vectors),
$$
\hat \cH_s \doteq \operatorname{span}\{\hat v_1, \ldots, \hat v_s\}.
$$
We define the orthogonal projection onto $\hat \cH_s$ as $\hat P$. 
\begin{rem}
In nonparametric regression, as employed in this approach, regularization becomes necessary. This leads to biased estimators since $\E[\hat f_{i,\la}] \ne f_i$. For subspace approximation, it is crucial to effectively control this bias since it cannot be averaged out.
\end{rem}

\textbf{Target task - inference.} We are given a target task dataset $\mathcal{D}_{T} = \{(x_{T,j}, y_{T,j})_{j=1}^{n_{T}}\} \in (\cX \times \R)^{n_T}$ sampled from $\mu_T$ in order to approximate $f_T$. As mentioned in Remark~\ref{rem:RKHS}, $\hat \cH_s = \hat P(\cH) \subseteq \cH$ forms a RKHS on $\cX$ having the same inner product as $\cH$ and with reproducing kernel $\hat K(x,y) = \langle \hat P \phi(x), \phi(y) \rangle_{\cH}, (x,y) \in \cX^2$. Consequently, we can estimate $f_T$ via regularized kernel regression within $\hat \cH_s$, as shown in Eq.~\eqref{eq:krr}. For $\la_*>0$,
\begin{align} \label{eq:krr_inference}
    \hat{f}_{T,\lambda_*} 
    &\doteq \argmin_{f \in \hat{\mathcal{H}}_{s}} 
    \sum_{j=1}^{n_T} \left(f(x_{T,j}) - y_{T,j}\right)^2 
    + n_T \lambda_* \|f\|_{\mathcal{H}}^2.
\end{align}
Since $\hat \cH_s$ is $s-$dimensional, it can be treated as a standard regularized regression in $\R^s$ (see Section~\ref{sec:algo}). The following remark highlights the main technical difficulties over the linear case. 

\begin{rem}[Differences from Linear Case] \label{rem:complinear} We point out that, while the algorithm used in our meta-learning approach draws inspiration from \cite{kong2020meta}, there are significant differences due to the complexities of the nonlinear setting, as opposed to the linear one, as outlined below. 

--- First, from the algorithmic perspective, proper regularization is crucial in an infinite dimensional space to prevent overfitting. \cite{kong2020meta} did not employ a regularization scheme, but instead relied on OLS regression, which does not directly extend to infinite dimension where some form of regularization is needed to control a learner's capacity. A second algorithmic difference arises in the instantiation of the procedure in input space $\mathbb{R}^d$: while our procedure appears similar to \cite{kong2020meta}'s when described in the RKHS $\cH$ i.e., after embedding, its instantiating in $\mathbb{R}^d$ is nontrivial, as it involves translating operations in $\cH$---e.g., projections onto subspaces of $\cH$---into operations in $\R^d$. Section~\ref{sec:algo} below addresses such technicality in depth.

--- Second, many crucial difficulties arise in the analysis of the infinite dimensional setting, which are not present in the finite-dimensional case. Importantly, in infinite dimensional space, the analysis effectively concerns two separate spaces: the RKHS $\cH$ which encodes the nonlinear representation, and the $L_2$ regression space. Thus a main technical difficulty is to relate rates of convergence in $\cH$ (where all operations are taking place) to rates in $L_2$, in particular via the \emph{covariance operator} which links the two norms $\| \cdot \|_{\cH}$ and $\|\cdot \|_{L_2}$; this is relatively easy in finite dimension by simply assuming an identity covariance (or bounds on its eigenvalues) as done in \cite{kong2020meta,du2021fewshot,tripuraneni2021provable}, but such assumptions do not extend to infinite dimension where concepts such as "identity covariance" are not defined. Namely, an infinite dimensional covariance operator must be compact, which implies that its eigenvalues decay to zero. Our analysis reveals that the speed of that decay (encoded in Assumptions~\ref{asst:evd} and~\ref{asst:emb}) determines the rate at which we can learn. Furthermore, unlike in \cite{kong2020meta,du2021fewshot,tripuraneni2021provable}, where there was no need to regularize the task-specific regressors, much of our analysis focuses on understanding the bias-variance trade-offs induced by the choice of regularizers. This is nontrivial but is crucial for guaranteeing gains in our nonlinear case, as explained in the paper's introduction. Thus, in the present infinite dimensional setting, as we will see, such crucial trade-offs will depend on specific measures of smoothness---of the RKHS $\cH$ and the regression functions therein---as introduced in the main results Section \ref{sec:mainresults} (see 
Assumptions \ref{asst:evd}, \ref{asst:emb}, \ref{asst:src}). 
\end{rem}

\subsection{Instantiation in Data Space} \label{sec:algo}
In this section, we describe in detail the steps outlined in Section~\ref{sec:sub_sec3.2} to offer a comprehensive understanding of the computational process. In particular, we focus on the computation of the right singular vectors of $\hat C_{N,n,\la}$, which plays a crucial role in constructing $\hat \cH_s$. Additionally, we provide insights into the projection of new data points onto $\hat \cH_s$, which is essential during the inference stage. We emphasize that such instantiations were not provided for kernel classes in the nonlinear settings addressed by \cite{maurer2016benefit,du2021fewshot}; given the nonconvexity of the loss (Eq.~(1) in both papers), this task is nontrivial. 

\textbf{Singular Value Decomposition of $\hat C_{N,n,\la}$.} We start by explaining how we can compute the SVD of $\hat C_{N,n,\la}$ in closed form from data. Let $\{\hat v_i\}_{i=1}^s$ and $\{\hat u_i\}_{i=1}^s$ be the right and left singular vectors corresponding to the largest $s$ singular values, and denote $\hat V_s = [\hat v_1, \ldots, \hat v_s]$ and $\hat U_s = [\hat u_1, \ldots, \hat u_s]$. The next proposition shows that $(\hat U_s, \hat V_s)$ can be obtained through the solution of a generalized eigenvalue problem associated to the matrices $J,Q \in \R^{N \times N}$ where for $(i,j) \in [N]^2$
$$
\begin{aligned}
    J_{i,j} &= \langle \hat f_i, \hat f_j \rangle_{\cH} =  nY_i^{\top}\left(K_i + n \la I_n\right)^{-1}K_{ij} \left(K_j + n \la I_n\right)^{-1}Y_j, \\
    Q_{i,j} &= \langle \hat f_i', \hat f_j' \rangle_{\cH} =  n(Y_i')^{\top}\left(L_i + n \la I_n\right)^{-1}L_{ij} \left(L_j + n \la I_n\right)^{-1}Y_j',
\end{aligned}
$$
\begin{prop}\label{prop:gen_eigen}
    Consider the generalized eigenvalue problem which consists of finding generalized eigenvectors $(\alpha^{\top}, \beta^{\top})^{\top} \in \R^{2N}$ and generalized eigenvalues $\gamma \in \R$ such that
    \begin{gather*}
    \begin{bmatrix} 0 & QJ \\ JQ & 0 \end{bmatrix}
    \begin{bmatrix} \al \\ 
    \beta\end{bmatrix}
    =
    \gamma
    \begin{bmatrix} Q & 0 \\ 0 & J \end{bmatrix}
    \begin{bmatrix} \al \\ 
    \beta\end{bmatrix}
    \end{gather*}
    Define  $A \doteq [\hat f_1', \ldots, \hat f_N']$ and $B \doteq [\hat f_1, \ldots, \hat f_N]$ and let $\{(\hat \alpha^{\top}_i, \hat \beta^{\top}_i)^{\top}\}_{i=1}^s$ be the generalized eigenvectors associated to the $s-$largest generalized eigenvalues of the above problem and re-normalized such that $\alpha^{\top}_i Q\alpha_i =  \beta_i^{\top} J \beta_i =1, i \in [s]$. The following two families of vectors $\{\hat{u}_i\}_{i=1}^s$ and $\{\hat{v}_i\}_{i=1}^s$ are orthonormal systems, and correspond to top-$s$ left and right singular vectors of $\hat C_{N,n,\la}$:
    $$
        \hat u_i = A\hat \al_i = \sum_{j=1}^N (\al_i)_j\hat{f}_j', \quad \hat v_i = B\hat \be_i =  \sum_{j=1}^N (\hat\be_i)_j\hat{f}_j, \quad i \in [s].
    $$
    In other words, we can define the projection onto the subspace $\hat{\cH}_s$ via $\{\hat{v}_i\}_{i=1}^s$:
    $$
    \hat \cH_s \doteq \operatorname{span}\{\hat v_1, \ldots, \hat v_s\} = \operatorname{span}\{B \hat \be_1, \ldots, B\hat \be_s\}.
    $$
\end{prop}

\textbf{Projection onto $\hat \cH_s$ and inference.} Next, we explain how we can project a new point onto $\hat \cH_s$ and perform inference on such representations. The projection onto $\hat \cH_s$ satisfies $\hat P = \hat V_s \hat V_s^*$. A new point $x \in \cX$ can be projected into $\hat \cH_s$ as $\hat P\phi(x)$ and identified to $\R^s$ via 
\begin{equation} \label{eq:embedding}
    \tilde{x} = \hat V_s^*\phi(x) = (\langle \hat v_1, \phi(x) \rangle_{\cH}, \ldots , \langle \hat v_s, \phi(x) \rangle_{\cH})^{\top} = (\hat v_1(x), \ldots, \hat v_s(x))^{\top} \in \R^s.
\end{equation}
By Proposition~\ref{prop:gen_eigen}, $\tilde{x}$ can be computed as 
$$
\tilde{x}_i = \hat v_i(x) = \langle \hat v_i, \phi(x) \rangle_{\cH} = \langle B \hat \be_i, \phi(x) \rangle_{\cH} = \hat \be_i^{\top}B^*\phi(x), \quad i \in [s],
$$
where  $B^*\phi(x) \doteq (\hat f_1(x), \ldots, \hat f_N(x))^{\top} \in \R^N$. Recall that after pre-training, at inference, we receive a target task dataset $\mathcal{D}_{T} = \{(x_{T,j}, y_{T,j})\}_{j=1}^{n_{T}} $. We denote by $\tilde x_{T,j} \in \R^s$ the embedding of the covariate $x_{T,j}$ into $\hat \cH_s$ according to Eq.~\eqref{eq:embedding}, and by $X_T \doteq [\tilde x_{T,1}, \ldots, \tilde x_{T,n_T}] \in \R^{s \times n_T}$ the data matrix that collects the embedded points as columns, $K_T \doteq X_T^{\top}X_T \in \R^{n_T \times n_T}$ is the associated Gram matrix and $n_T^{-1}X_TX_T^{\top} \in \R^{s \times s}$ the associated empirical covariance. 
\begin{prop}\label{prop:compute_inference}
$\hat{f}_{T,\lambda_*} = \hat V_s\beta_{T,\lambda_*}$, where 
$$
        \hat \beta_{T,\lambda_*} \doteq \argmin_{\beta \in \R^s} \sum_{j=1}^{n_T} \left(\beta^{\top}\tilde x_{T,j}-y_{T,j}\right)^2 + n_T\lambda_* \|\beta\|_{2}^2 = X_T(K_T + n_T \la_* I_{n_T})^{-1}Y_T, 
$$
and $Y_T \doteq (y_{T,1}, \ldots, y_{T,n_T})^{\top} \in \R^{n_T}$. For all $x \in \cX$, $\hat{f}_{T,\lambda_*}(x) = \beta_{T,\lambda_*}^{\top}\tilde{x}$.
\end{prop}

\section{Main Results}\label{ta}

\subsection{Regularity Assumptions}\label{sec:regularity}
Our first two assumptions are related to the eigensystem of the covariance operator. For $i \in [N] \cup \{T\}$, the covariance operator for task $i$, $\Sigma_i \doteq \E_{\mu_i}[\phi(X) \otimes \phi(X)]$, is positive semi-definite and trace-class, and thereby admits an eigenvalue decomposition with eigenvalues $\la_{i,1} \geq \la_{i,2} \geq \ldots \geq 0$ and eigenvectors $\{\sqrt{\lambda_{i,j}}e_{i,j}\}_{j \geq 1}$ \citep[Lemma 2.12][]{steinwart2012mercer}.

\begin{assumption}\label{asst:evd} For $i \in [N]$, the eigenvalues of the covariance operator $\Sigma_i$ from the $(K, \mu_i)$ pair satisfy a polynomial decay of order $1/p$, i.e., for some constant $c >0$ and $0<p\leq 1$, and for all $j \geq 1$, $\la_{i,j} \leq c j^{-1/p}$. When the covariance operator has finite rank, we have $p=0$.
\end{assumption}
The assumption on the decay rate of the eigenvalues is typical in the risk analysis for kernel ridge regression \citep[see e.g.,][]{fischer2020sobolev,caponnetto2007optimal}.

\begin{assumption}\label{asst:emb}
There exist $\alpha \in [p, 1]$ and $k_{\alpha,\infty} > 0$, such that, for any task $i \in [N]$ and $\mu_{i}-$almost all $x \in \cX$,
\(\sum_{j \geq 1} \la_{i,j}^{\al}e_{i,j}^2(x)  \leq k_{\alpha,\infty}^2.\)
\end{assumption}
This assumption is known as an {\it embedding property} (into $L_\infty$, see \cite{fischer2020sobolev}), and is a regularity condition on the pair $(K, \mu_i)$. In particular, let 
$T_{K,i} \doteq \sum_j \lambda_{i, j} \ e_{i, j} \otimes_{L_{2}(\mu_i)} e_{i, j}$ denote the {\it integral operator} $L_{2}(\mu_i) \mapsto L_{2}(\mu_i)$ induced by $K$, then the assumption characterizes the smallest $\alpha$ such that the range of $T^{{\alpha}/{2}}_{K,i}$ may be continuously embedded into $L_\infty(\mu_i)$. As it is well-known for continuous kernels, $\operatorname{ran}T^{1/2}_{K,i} \equiv \cH$, thus the assumption holds for $\alpha =1$ whenever $K$ is bounded. Note that the {\it interpolation spaces} $\operatorname{ran} {T^{{\alpha}/{2}}_{K,i}}$ only get larger as $\alpha \to 0$, eventually coinciding with the closure of $\operatorname{span}\{e_{i, j}\}_{j \geq 1}$ in $L_{2}(\mu_i)$. Additionally, it can be shown that Assumption~\ref{asst:emb} implies Assumption~\ref{asst:evd} with $p=\alpha$ \citep[Lemma 10][]{fischer2020sobolev}.

As alluded to in the introduction, $\alpha$ has no direct benefit for regression in our {\it well-specified} setting with $f_i\in \cH$, but is beneficial in meta-learning (see Corollary~\ref{cor:learning_rate} and Remark~\ref{rem:beyondregression} thereafter). 

\begin{assumption}\label{asst:src} There exist $r \in [0,1]$ and $R \geq 0$, such that for $i \in [N]$, the regression function $f_i$ associated with $\mu_i$ is an element of $\cH$ and satisfies $\|\Sigma^{-r}_i f_i\|_{\cH} \doteq R < \infty.$ 
\end{assumption}
This assumption, imposing smoothness on each source task regression function, is standard in the statistical analysis of regularized least-squares algorithms \citep{caponnetto2007optimal}.

\begin{rem} 
Assumptions \ref{asst:evd}, \ref{asst:emb}, and \ref{asst:src} only concern the source tasks towards nonlinear meta-learning. We will see in Section \ref{sec:mainresults} that they are complementary in ensuring enough \emph{smoothness} of the source regression functions to allow for sufficient \emph{under-regularization} to take advantage of the aggregation of $N$ source tasks. 
Thus, the main assumption on the target task is simply that it shares the same nonlinear representation as the source tasks. 
\end{rem}
Finally, to control the noise we assume the following.

\begin{assumption}\label{asst:mom} There exists a constant $Y_{\infty}\geq0$ such that for all $Y \sim \mu_i, i \in [N] \cup \{T\}$: $|Y| < Y_{\infty}$.
\end{assumption}

\subsection{Main Theorems}\label{sec:mainresults}

\begin{theorem}\label{th:inference_bound}
Under Assumptions \ref{asst:meta}, \ref{asst:rich} and \ref{asst:mom} with $s \geq 1$, for $\tau \geq 2.6$, $ 0< \la_* \leq 1$ and
\begin{equation*}
n_T \geq 6\kappa^2\la_*^{-1}\left(\tau+ \log(s)\right),
\end{equation*}
with probability not less than $1-3e^{-\tau}$ and conditionally on $\{\mathcal{D}_i\}_{i=1}^N$,
\begin{align*}
    \cE_{\mu_T}(\hat{f}_{T,\lambda_*}) 
    &\leq c_0\left\{
        \sqrt{\frac{\tau s}{n_T}} 
        + \frac{\tau}{n_T \sqrt{\lambda_*}}  
        + \sqrt{\lambda_*} 
        + \left\| \hat{P}_{\perp} P \right\|
    \right\},
\end{align*}
where $\hat P_{\perp} \doteq I_{\cH} - \hat P$ and $c_0$ is a constant that depends only on $Y_{\infty}, \|f_T\|_{\cH}$, and $\kappa$. Hence, treating $\tau$ as a constant, if we take $\la_*=12\kappa^2(\log(s) \vee \tau)n_T^{-1}$, conditionally on $\{\mathcal{D}_i\}_{i=1}^N$, for $n_T \geq 12\kappa^2(\log(s) \vee \tau)$, we get that $\cE_{\mu_T}(\hat{f}_{T,\lambda_*})$ is of the order 
\begin{equation*}
    \sqrt{\frac{s}{n_T}} + \left\|\hat{P}_{\perp} P\right\|.
\end{equation*}
\end{theorem}
Theorem~\ref{th:inference_bound} reveals that the excess risk for the target task consists of two components: $\sqrt{s/n_T}$ due to the inference stage, and $\|\hat{P}_{\perp}P\|$ in the pre-training stage. In the upcoming Theorem~\ref{th:pre_training}, we will see that the pre-training error $\|\hat{P}_{\perp}P\|$ decays with $n$ and $N$. In other words, if either $N$ (number of tasks) or $n$ (number of data within each task) is sufficiently large, we can guarantee that the excess risk decays at the parametric rate $O(\sqrt{s/n_T})$, an optimal rate achieved only by performing linear regression in a space of dimension $s$. $\|\hat{P}_{\perp}P\|$ is the sin-$\Theta$ distance between $\cH_s$ and $\hat \cH_s$ \cite{stewart1990matrix}. {We can relate this distance to the difference between $C_N$ and $\hat C_{N,n,\la}$ using classic perturbation theory for singular vectors. Proposition~\ref{prop:davis_kahan_meta} is a basic generalization of Wedin’s $\sin-\Theta$ Theorem \cite{wedin1972perturbation}.}

\begin{prop}[Wedin’s $\sin-\Theta$ Theorem]
\label{prop:davis_kahan_meta} Given $C_N$ and $\hat C_{N,n,\la}$ defined in Eqs.~\eqref{eq:task_cov} and \eqref{eq:task_cov_approx}, with $\gamma_s$ smallest nonzero eigenvalues of $C_N$. We have, 
\begin{equation}\label{eq:davis_kahan}
{\|\hat P_{\perp} P\| \leq 2\gamma_s^{-1}\|\hat{C}_{N,n,\la}-C_N\|.}
\end{equation}
\end{prop}
We refer to Section~\ref{sec:wedin} in the Appendix for the proof. Note that the operator norm $\|\hat{C}_{N,n,\la}-C_N\|$ is dominated by the Hilbert-Schmidt norm $\|\hat{C}_{N,n,\la}-C_N\|_{HS}$. The following theorem provides high probability bounds on this quantity.

\begin{theorem} \label{th:pre_training}
Let Assumptions \ref{asst:evd}, \ref{asst:emb}, \ref{asst:src} and \ref{asst:mom} hold with parameters $0 < p \leq \alpha \leq 1$ and $r \in [0,1]$. Let $\tau \geq \log(2)$, $N \geq \tau$ and $0< \la \leq 1 \wedge \min_{i \in [N]}\left\|\Sigma_i\right\|$. Define the following terms:
\begin{align*}
    A_{\la} &\doteq c\log(Nn)\left(1 + p\log (\la^{-1})\right) \lambda^{-\alpha}   \\
    B_{\la} &\doteq c\log(Nn) \left(1 + p\log (\la^{-1})\right)\lambda^{-(1+p)},
\end{align*}
where $c$ only depends on $k_{\alpha,\infty}, D, \kappa$. We require $n \geq A_{\la}$ if $r\in (0,1/2]$ or $n \geq B_{\la}$ if $r\in (1/2,1]$. Under both scenarios, with probability greater than $1-2e^{-\tau}-o((nN)^{-10})$ over the randomness in the source tasks we have
\begin{equation} 
\|\hat{C}_{N,n,\la}-C_N\|_{HS}  \leq C_1\left(\frac{\log (nN)\sqrt{\tau}}{\sqrt{nN} \lambda^{\frac{1}{2}+\frac{p}{2}}}\sqrt{1+ \frac{1}{n \lambda^{\alpha-p}}} +  \lambda^{r}\right).\label{eq:sin_risk_sharp}
\end{equation}
where $C_1$ only depends on $Y_{\infty}$, $R$, $\kappa$, $p$ and $k_{\alpha,\infty}$. 
\end{theorem}
We highlight two key aspects of Theorem~\ref{th:pre_training}. First, the bound is comprised of two terms that come from a bias-variance decomposition (refer to Section~\ref{sec:ana_out} for details): \[\|\hat{C}_{N,n,\la}-C_N\|_{HS} \leq \underbrace{\|\hat{C}_{N,n,\la}-\mathbb{E}( \hat{C}_{N,n,\la})\|_{HS}}_{\text{Variance}} + \underbrace{\|\mathbb{E}( \hat{C}_{N,n,\la})-C_N\|_{HS}}_{\text{Bias}}.\]  The first and second terms in Eq.~\eqref{eq:sin_risk_sharp} correspond to bounds on the variance and on the bias respectively. Secondly, while we obtain the same upper bound in Eq.~\eqref{eq:sin_risk_sharp} for the  two distinct scenarios $r \in (0,1/2]$ and $r \in (1/2,1]$, the requirement on the number of training samples per task is different. In particular, $B_{\lambda} \geq A_{\lambda}$, since $\lambda  \leq 1$ and $p+1 \geq  \alpha$. This means that we can benefit from further smoothness $r > 1/2$, but at the cost of a higher number of samples per source task. Our analysis in Theorem~\ref{theo:bias_pf} implies that the difference comes from bounding the bias term. We specifically shows that uniformly bounding the bias from each task when $r \in (1/2, 1]$ (require $n \geq B_{\lambda}$) is strictly harder than when $r \in (0, 1/2]$ (require $n \geq A_{\lambda}$). As such, our results reveal the inherent difficulty of nonlinear meta-learning: analyzing the bias is more involved than analyzing the  variance, a fact which cannot be seen in the linear representation case.

\begin{rem}[Further Smoothness and the Well-specified Regime] \label{rem:rkhs_norm}
While in usual analyses, consistency in $L_2$ norm is assured for $r=0$ (implying that the regression function is in $\cH$), we require further smoothness on source regression functions (i.e., $r>0$) to guarantee consistency in our setting. The requirement for additional smoothness stems from the fact that the result depends on convergence of regression estimates in \emph{the stronger RKHS norm} rather than in $L_2$ norm, as the above $\|\cdot\|_{HS}$ and projections are defined w.r.t. the RKHS itself.

We point out that in kernel learning literature (see e.g., \cite{caponnetto2007optimal,fischer2020sobolev}), one often observes the Tikhonov saturation effect, where the learning rate does not improve for $r >1/2$. However, we remark that this saturation happens only when the $L_2$ norm is used. In particular, Eq.~\eqref{eq:sin_risk_sharp} demonstrates that our learning rate can be improved up to $r =1$. This reflects the fact that, if the RKHS norm is employed, the Tikhonov saturation effect happens for $r > 1$. A similar phenomenon is observed by \cite{blanchard2018optimal}.
\end{rem}
Combining Theorem~\ref{th:inference_bound}, Proposition~\ref{prop:davis_kahan_meta}, and Eq.~\eqref{eq:sin_risk_sharp} from Theorem~\ref{th:pre_training}, we obtain the following results on the meta-learning excess risk.

\begin{corollary}\label{cor:learning_rate}
Under the assumptions of Theorem \ref{th:inference_bound} and Theorem \ref{th:pre_training}, for $\tau \geq 2.6$ and $\la_*=12\kappa^2(\log(s) \vee \tau)n_T^{-1}$, with probability $1-5e^{-\tau}-o((nN)^{-10})$ over the randomness in both the source and target tasks, we have the following regimes of rates for a constant $C_3$ that only depends on $Y_{\infty}$, $R$, $\kappa$, $\gamma_1$, $p$, $c$, $\|f_T\|_{\cH}$ and $k_{\alpha,\infty}$.

{A.} {\bf Small number of tasks.} In this regime, with the number of tasks $N$ being small, the variance is significant compared to the bias. Therefore, we must choose $\lambda$ to balance the order of the bias with that of the variance. If $N \leq n^{\frac{2r+1+p}{\alpha} -1}$ and $ r \in (0,1/2]$ or $N \leq n^{\frac{2r+1+p}{p+1} - 1}$ and $r \in (1/2,1]$, for a choice of $\lambda = \left(\log^2(nN)/(nN)\right)^{\frac{1}{2r+1+p}}$,
\begin{equation} \label{eq:corr_eq1}
\cE_{\mu_T}(\hat{f}_{T,\lambda_*})\leq C_3\tau\left\{\sqrt{\frac{s}{n_T}}+\left(\frac{\log^2(nN)}{nN}\right)^{\frac{r}{2r+1+p}}\right\}.
\end{equation}

{B.} {\bf Large number of tasks.} In this regime, we consider larger $N$ (see $B.1$ and $B.2$ below), so that the variance term becomes negligible compared to the bias. Therefore the rates below correspond to the choices of $\lambda$ that minimize the bias, in Eq.~\eqref{eq:sin_risk_sharp} (under the constraints $n \geq A_{\lambda}, B_{\lambda}$). In what follows, $\omega > 2$ is a free parameter. 

$\bullet$ B.1. For $r \in (0,1/2]$, if $n^{\frac{2r+1+p}{\alpha} -1} \leq N \leq o\left(e^n\right)$, for a choice of $\lambda = \left(\frac{\log^{\omega}(nN)}{n}\right)^{\frac{1}{\alpha}}$,
\begin{equation}
    \cE_{\mu_T}(\hat{f}_{T,\lambda_*})\leq C_3\tau\left\{\sqrt{\frac{s}{n_T}}+\log^{\frac{\omega r}{\alpha}}(nN)\cdot n^{-\frac{r}{\alpha}}\right\}.\nonumber
\end{equation}

$\bullet$ B.2. For $r \in (1/2,1]$, if $n^{\frac{2r+1+p}{p+1} -1} \leq N \leq o\left(e^n\right)$, for a choice of $\lambda = \left(\frac{\log^{\omega}(nN)}{n}\right)^{\frac{1}{p+1}}$,
\begin{equation}
    \cE_{\mu_T}(\hat{f}_{T,\lambda_*})\leq C_3\tau\left\{\sqrt{\frac{s}{n_T}}+\log^{\frac{\omega r}{p+1}}(nN)\cdot n^{-\frac{r}{p+1}}\right\}.\nonumber
\end{equation}
\end{corollary}
\begin{rem}[Saturation effect on large $N$]\label{rem:saturation}
Corollary~\ref{cor:learning_rate} shows no further improvement from larger $N$ once $N \geq n^{\frac{2(r\wedge 1/2)+1+p}{\alpha} -1}$, since the rates then only depend on $n$ (as outlined in case B). This is due to a saturation effect from the bias-variance trade-off, i.e., $N$ only helps decrease the variance term below the best achievable bias; at that point the bias (within each task) can only be further improved by larger per-task sample size $n$. 
\end{rem}

\begin{rem}[Regime $N \gtrsim \exp(n)$]\label{rem:rem_exp} 
The regimes presented in Corollary~\ref{cor:learning_rate} only cover settings where $N \lesssim \exp(n)$, which is in fact the only regime covered by previous works (see, for instance \cite{du2021fewshot,tripuraneni2020theory}). This is due to the constraints $n \geq A_{\lambda}, B_{\lambda}$, that prevents $N \gtrsim \exp(n)$. However, at the cost of a less tight rate we can obtain a bound on the pre-training error that is free of any constraint on $n$ (see Section~\ref{sec:exp_setting}). As a corollary of this theorem, when $N \gtrsim \exp(n)$, choosing $\lambda = n^{-\frac{1}{2}}$, results in the nontrivial rate 
$$
\cE_{\mu_T}(\hat{f}_{T,\lambda_*})\lesssim \sqrt{\frac{s}{n_T}}+n^{-\frac{r}{2}},
$$
Notice that this is a slower rate than shown for smaller $N$ in regime B of Corollary~\ref{cor:learning_rate}. Tightening the rates in the regime of $N \gtrsim \exp(n)$ appears difficult, and is left as an open problem. We emphasize, as stated earlier, that this regime is in fact not addressed by previous works, even under the stronger assumption of linear representations.  
\end{rem}

\paragraph{Regimes of Gain.} We want to contrast our results in the meta-learning setting with the rates obtainable on the target task without the benefits of source tasks. Since no regularity condition is imposed on the target distribution, the best target rate, absent any source tasks, is of the form $O(n_T^{-1/4})$ \citep[see e.g.,][]{caponnetto2007optimal}\footnote{Note that the assumption that $f_T$ is in some subspace $\cH_s$ is irrelevant for usual kernel ridge regression, since it is always true once we know that $f$ belongs to $\cH$.}; thus we gain from the source tasks whenever $\cE_{\mu_T}(\hat{f}_{T,\lambda_*}) = o(n_T^{-1/4})$.

Our interest, however, is in {\it regimes where the gain is greatest}, in that the source tasks permit a final meta-learning rate of $\cE_{\mu_T}(\hat{f}_{T,\lambda_*})\lesssim \sqrt{s/n_T}$; Corollary~\ref{cor:learning_rate} displays such regimes according to the number of source samples $N$ and $n$, and the parameters $r$, $\alpha$ and $p$, denoting the difficulty of the sources tasks. While it is clear that larger $r$ indicates \emph{smoother} source regression functions $f_i$ as viewed from within the RKHS $\cH$, smaller parameters $\alpha$ and $p$ can be understood as a \emph{smoothness level} of the RKHS $\cH$ itself---e.g., consider a Sobolev space $\cH$ of $m$-smooth functions, then we may take $\alpha, p\propto 1/m$ (see Example \ref{ex:sobolev}). Thus the {smoother} the source tasks, viewed under $r$, $\alpha$ and $p$, the faster the rates we can expect, since our approach aims at reducing the bias in each individual task (which is easiest under smoothness see Remark \ref{rem:overfitting_low} below).

Focusing on the situation where the number of samples per task is roughly the same across source and target, i.e., $n \propto n_T$, the conditions for meta-learning to provide the greatest gain, i.e., achieving $O(n^{-1/2})$ rate, under various regimes are listed in Table~\ref{tab:list_con}. 

\begin{table}[tbhp]
    \setstretch{1.5}
    \centering
    \caption{Conditions for meta-learning to reach the parametric rate $O\left(\sqrt{s/n}\right)$, log terms are removed for clarity.} 
    \label{tab:list_con}
\begin{tabular}{c|c|c|c}
\hline
\hline
\textbf{Cases}&\textbf{Range of Source Tasks} & \textbf{Choice of \(\lambda\)} & \textbf{Regimes of Gain} \\
\hline
A & $n^{\frac{2r+1+p}{2r}-1} \leq N \leq n^{\frac{2r+1+p}{\alpha} -1}$ & $(nN)^{-\frac{1}{2r+1+p}}$ & $ \frac{\alpha}{2} \leq r \leq \frac{1}{2}$ \\
\hline
A& $n^{\frac{2r+1+p}{2r}-1} \leq N \leq n^{\frac{2r+1+p}{p + 1} - 1}$ & $(nN)^{-\frac{1}{2r+1+p}}$ & $ \frac{p+1}{2} \leq r \leq 1 $ \\
\hline
B.1 &$n^{\frac{2r+1+p}{\alpha} -1} \leq N \leq o\left(e^n\right)$ & $n^{-\frac{r}{\alpha}}$ & $ \frac{\alpha}{2} \leq r \leq \frac{1}{2}$\\
\hline
B.2 &$n^{\frac{2r+1+p}{p+1} -1} \leq N \leq o\left(e^n\right)$ & $n^{-\frac{r}{p+1}}$ & $ \frac{p+1}{2} \leq r \leq 1$\\
\hline
\hline
\end{tabular}
\end{table}

\begin{rem}[Under-regularization/Overfitting]
\label{rem:overfitting_low} 
In order for meta-learning to provide gain, in particular for $n \propto n_T$, we have to \emph{overfit} the regression estimates in each source task, i.e., set $\lambda$ lower than would have been prescribed for optimal regression (choices of $\lambda$ for the different regimes of gain are summarized in Table~\ref{tab:list_con}). 

Overfitting is essential because, as highlighted in the introduction, the bias inherent in each task during meta-learning cannot be averaged out. Deliberate under-regularization reduces this bias at the expense of increased variance within each task. However, the variance in the target task may subsequently be mitigated by aggregating across multiple tasks.

More specifically, in the regimes of gain discussed earlier, the choices of $\lambda$ in Corollary~\ref{cor:learning_rate} are consistently lower than the optimal regression choice of $\lambda_{KRR} \asymp n^{-\frac{1}{2(r \wedge 1/2)+1+p}}$\citep[see e.g.,][Theorem 1]{fischer2020sobolev} in the well-specified regime. This deviation from the optimal regression setting indicates overfitting, which again reveals that understanding nonlinear meta-learning is fundamentally more difficult than the linear setting due to the bias term. {This effect is similarly observed in nonparametric kernel regression when splitting the dataset and averaging estimators trained on each split of the dataset \cite{zhang2015divide}.}
\end{rem}

\begin{rem}[Regularity beyond regression]\label{rem:beyondregression} Notice that the choice of the regularization parameter in kernel ridge regression $\lambda_{KRR}\asymp n^{-\frac{1}{2(r\wedge1/2)+1+p}}$ has no direct dependence on $\alpha$: lower values of $0<\alpha\leq 1$ yield no further benefit in regression once we assume $f_i \in \cH$, as opposed to the misspecified setting where $f_i$ lies outside $\cH$\footnote{Note, however, that $p\leq \alpha$, and therefore a small $\alpha$ implies that we are in the small $p$ regime (and the rates do depend on $p$).}. By contrast, in meta-learning, we do benefit from considering $\alpha$, as $\alpha$ governs both the threshold level at which the saturation effect on large $N$ kicks in (see Remark~\ref{rem:saturation}) and the level of smoothness required for meta-learning to provide the greatest gain (See Table~\ref{tab:list_con} and associated discussion). Ultimately, if $\alpha \to 0$, there is no saturation effect, and the rates always match the parametric rate $O(n^{-1/2})$. This indicates that subspace learning is a fundamentally different problem to ridge regression.
\end{rem}

\paragraph{Characterizing $\alpha$, $p$ and $r$.} As discussed above, smaller parameters $\alpha $ and $p$ and higher parameter $r$ yield faster meta-learning rates. The next examples yield insights on these situations. Throughout, recall that by Lemma 10 \cite{fischer2020sobolev}, we have $p \leq \alpha$, i.e., $p = \alpha$ is always admissible. 

\begin{ex}[Finite-dimensional kernels]\label{ex:finite}
Suppose $\cH$ is finite dimensional, i.e., the covariance operators $\Sigma_i$ each admit a finite number of eigenfunctions $e_{i, j}, j = 1, 2, \ldots k$ for some $k\geq 1$. Then clearly as the eigenfunctions $\{e_{i, j}\}$ are bounded \footnote{As we employ a bounded kernel, every function in the RKHS is bounded (Lemma 4.23 \cite{steinwart2008support}).} and Assumptions \ref{asst:evd}-\ref{asst:emb} hold for $\alpha, p=0$. Furthermore, Assumption~\ref{asst:src} holds for any value of $r$. In this regime,  
\begin{equation}\label{eq:finite_dim}
    \cE_{\mu_T}(\hat{f}_{T,\lambda_*})  \lesssim\sqrt{\frac{s}{n_T}}+\sqrt{\frac{k}{\gamma_s^2 nN}}\log (nN).
\end{equation}
See Remark~\ref{rem:rem_pf} in the Appendix for the detailed derivations. As an example, for polynomial kernels $K(x, x') \doteq (x^\top x' + b)^m$ on compact domains $\cX\subset \R^d$, we obtain $k=d^m$. Note that, since polynomial regression converges at rate $O(\sqrt{d^{m}/n_T})$ \citep[see e.g.,][]{ghorbani2021linearized,chen2020learning,andoni2014learning,zippel1979probabilistic}, we can gain in meta-learning whenever the representation $\cH_s$ is of dimension $s \ll d^m$. 
\end{ex}

\begin{rem}[Subspace learning guarantees in the linear setting] In the meta learning model with linear representations, with $d$ the dimension of the input points and $s$ the dimension of the subspace, \cite{tripuraneni2021provable} (Theorem 5) provide an information-theoretic lower bound on the $\sin-\Theta$ distance $\|\hat P_{\perp}P\|$ of the order $\sqrt{\frac{ds}{nN}}$ valid for  estimators that are functions of the $nN$ data points. Assuming that the eigenvalues of $C_N$ are well-conditioned ($\gamma_s \asymp s^{-1}$), estimators with matching guarantees on the $\sin-\Theta$ distance are obtained in \cite{du2021fewshot,niu2024collaborative}. By the previous example, if we employ a linear kernel on $\R^d$ and under the assumption $\gamma_s \asymp s^{-1}$, we obtain a subspace learning error (up to a log term) of $\sqrt{\frac{ds^2}{nN}}$, recovering the learning rate obtained in \cite{tripuraneni2021provable}. Generalizing the result of \cite{tripuraneni2021provable} to the nonlinear setting with a lower bound depending on the parameters $(N,n,s,p,r,\alpha)$ represents a significant and valuable direction for future research. 
\end{rem}

\begin{ex}[Gaussian kernel] \label{ex:gaussian}
Let $\cX \subset \Rd$ be a bounded set with
Lipschitz boundary \footnote{For the definition of Lipschitz boundary see \citep[Definition 3][]{kanagawa2020convergence}.}, $\mu$ a distribution supported on $\cX \times \R$, with marginal distribution uniform on $\cX$ and let $K$ be a Gaussian kernel. Then by \citep[Corollary 4.13][]{kanagawa2018gaussian}, Assumption~\ref{asst:emb} is satisfied with any $\alpha \in (0,1]$, implying that Assumption~\ref{asst:evd} is also satisfied with any $p \in (0,1]$.
\end{ex}
 
\begin{ex}[Sobolev spaces and Matérn kernels]\label{ex:sobolev}
Let $\cX \subset \Rd$, be a non-empty, open, connected, and bounded set with a $C_{\infty}-$boundary. Let $\mu$ be a distribution supported on $\cX \times \R$, with marginal equivalent to the Lebesgue measure on $\cX$. Choose a kernel which induces a Sobolev space $H^m$ of smoothness $m \in \N$ with $m>d/2$, such as the Matérn kernel 
(see e.g., \cite{kanagawa2018gaussian} Examples 2.2 and 2.6). Then by \citet[Corollary 5][]{fischer2020sobolev}, Assumption~\ref{asst:evd} is satisfied with $p=\frac{d}{2m}$ and Assumption~\ref{asst:emb} is satisfied for every $\alpha \in (\frac{d}{2m},1]$. Furthermore, it can be shown that Assumption~\ref{asst:src} is satisfied if and only if the $\{f_i\}_{i=1}^N$ belong to a Sobolev space (with fractional smoothness) of smoothness $m(2r+1)$ (see \cite{fischer2020sobolev}).
\end{ex}

\section{Experimental Results} \label{sec:exp}

In this section, we report the results of experiments on simulated data to test the two main theoretical predictions of our paper: 1) with the proper number of tasks it is possible to learn at the parametric rate; 2) overfitting is beneficial for meta learning. Consider the Sobolev space $\cH = \{ f: [0,1] \to \R, f \text{ absolutely continuous}, f' \in L^2([0,1]), f(0)=0 \},$ equipped with the inner product $\langle f, g \rangle_{\cH} = \int_{0}^1 f'(x)g'(x) dx.$ $\cH$ is the RKHS associated to the kernel $K: [0,1] \times [0,1] \to \R, (x,x') \mapsto \min(x,x')$ \cite{gu2013smoothing}. For a fixed parameter $s \in \N$, we consider an orthonormal system (with respect to $\langle \cdot, \cdot \rangle_{\cH}$) of $s$ splines of degree 2 (i.e. piecewise quadratic functions with continuous derivative) $(\psi_1, \ldots, \psi_s)$ as shown in Figure~\ref{fig:task_examples}. We then take $\cH_s = \operatorname{span}\{\psi_1, \ldots, \psi_s\}$ and $P = \sum_{j=1}^s \psi_j \otimes \psi_j$ the projection onto $\cH_s$. Note that $P = VV^*$ with $V = [\psi_1, \ldots, \psi_s]$.  Any $\omega \in \R^s$ leads to an element of $\cH_s$ as,
$$
f = \sum_{\ell=1}^s \omega_{\ell} \psi_\ell(x) = \sum_{\ell=1}^s \omega_{\ell}\langle \psi_\ell, K(x,\cdot) \rangle_{\cH} = \langle g, PK(x,\cdot) \rangle_{\cH}, \qquad g \doteq \sum_{\ell=1}^s \omega_{\ell} \psi_\ell.
$$
To generate each task, we proceed as follows. For $i \in [N] \cup \{T\}$, $\omega_i \sim \cU(\sqrt{s}\mathbb{S}^{s-1})$, $f_i = \sum_{\ell=1}^s \omega_{i,\ell} \psi_\ell$, for $j = 1, \ldots, 2n$ (or $j = 1, \ldots, n_T$ for the target task),
$$
y_{i,j} = f_i(x_{i,j}) + \epsilon_{i,j}, \qquad x_{i,j} \sim \cU(0,1), \quad \epsilon_{i,j} \sim \cN(0,\sigma^2).
$$
Throughout the experiments, $\sigma$ is fixed to 0.1. 
In Figure~\ref{fig:task_examples}, we display an example of generated task for $s=10$. Given an estimator $\hat f$ for the target task, we evaluate its performance by approximating the squared excess risk $\E_{\mu_T}\left[(\hat f(X) - f_T(X))^2\right]$ on independent samples, where $\mu_T$ is the Lebesgue measure on $[0,1]$.
\begin{figure}[ht]
    \centering
    \begin{minipage}[b]{0.3\textwidth}
        \centering
        \includegraphics[width=\textwidth]{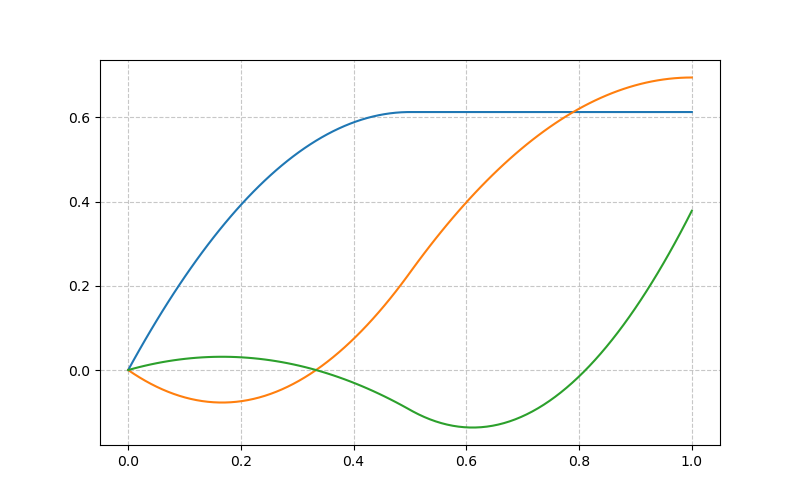}
    \end{minipage}
    \begin{minipage}[b]{0.3\textwidth}
        \centering
        \includegraphics[width=\textwidth]{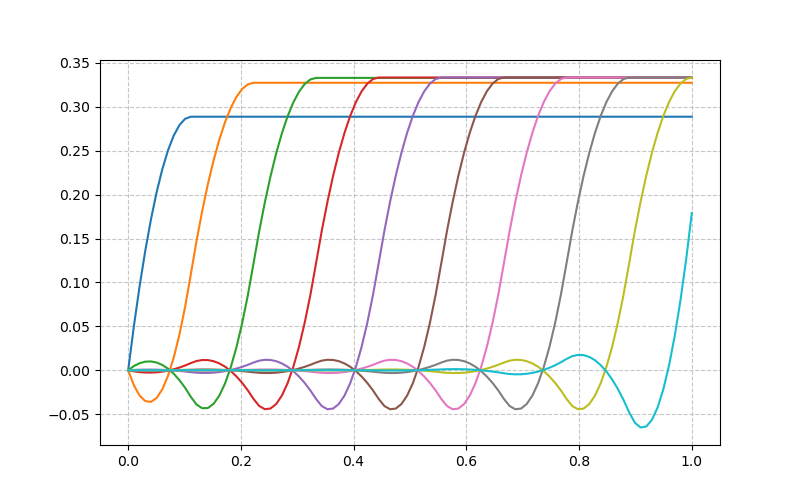}
    \end{minipage}
    \begin{minipage}[b]{0.3\textwidth}
        \centering
     \includegraphics[width=\textwidth]{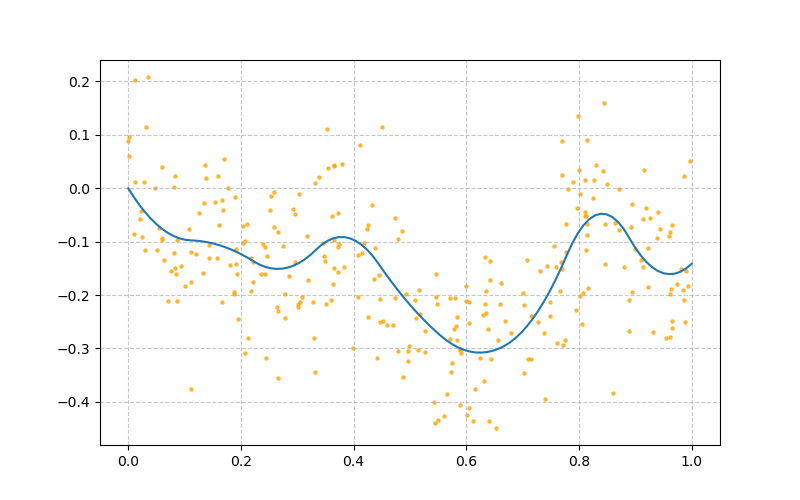}
    \end{minipage}
    \caption{{(Left)-(Center) Orthonormal system in $\cH$ spanning $\cH_s$ for respectively $s=3$ (Left) and $s=10$ (Center). (Right) Example of sampled task for $s=10$ with $300$ datapoints, the blue solid line represents the ground truth.}}
    \label{fig:task_examples}
\end{figure}

\paragraph{Parameter values: $p$, $\alpha$ and $r$.} { As the marginal probability distribution is the uniform measure on $[0,1]$ and $K$ induces a Sobolev space of smoothness $m=1$, by Remark~\ref{ex:sobolev}, Assumption~\ref{asst:evd} is satisfied with $p = \frac{1}{2}$ and Assumption~\ref{asst:emb} is satisfied with every $\alpha \in (\frac{1}{2}, 1]$. Finally, tasks functions are generated as linear combinations of order 2 splines and therefore belong to $H^m(0,1)$ for every $m < \frac{5}{2}$ (and do not belong to $H^m(0,1)$ for any $m \geq \frac{5}{2}$). By Remark~\ref{ex:sobolev}, Assumption~\ref{asst:src} is therefore satisfied for every $r \in [0,\frac{3}{4})$ (and Assumption~\ref{asst:src} is not satisfied for any $r \geq \frac{3}{4}$). In the experiments, we set $r = \frac{1}{2}$.}

\paragraph{Choice of regularization.} {We focus on the \textbf{small number of tasks regime}, Corollary~\ref{cor:learning_rate}-(A),  where $N \leq n^{\frac{2r+1+p}{\alpha}-1} = n^4$. According to Case A, we set $\lambda = (nN)^{-\frac{1}{2r+1+p}} = (nN)^{-\frac{2}{5}}$ and $\lambda_* = n_T^{-1}$. By Corollary~\ref{cor:learning_rate}, the excess risk on the target task is upper bounded (up to constants and log terms) by $\sqrt{s/n_T} +  (nN)^{-\frac{1}{5}}$.}

\paragraph{Learning at the parametric rate.} {We have shown in Table~\ref{tab:list_con} that given enough source tasks and samples per source task it is possible to learn at the parametric rate $\sqrt{s/n_T}$. To illustrate this fact, we compare our meta learning approach to an oracle estimator accessing the true subspace. The oracle estimator has access to $(\psi_1, \ldots, \psi_s)$ and is trained with linear ridge regression. For $x \in [0,1]$, define its transform $\tilde{x}^s \doteq (\psi_1(x), \ldots, \psi_s(x))^{\top} \in \mathbb{R}^s$. Then, $\hat f_{\text{oracle}}(x) \doteq \hat \beta^{\top}\tilde{x}^s$, with
$$
\hat \beta = \argmin_{\beta \in \mathbb{R}^s} \frac{1}{n_T} \sum_{i=1}^{n_T} \left(y_{T,i} - \beta^{\top}\tilde{x}_{T,i}^s\right)^2 + \lambda_{\text{oracle}} \|\beta\|_{2}^2. 
$$
For $\lambda_{\text{oracle}} = n_T^{-1}$, $\cE_{\mu_T}\left(\hat f_{\text{oracle}} \right)$ is of the order $\sqrt{s/n_T}$ \cite{mourtada2022elementary}. In Figure~\ref{fig:experiment}-(Left), for $s = 25$ and $n = 300$ we show the evolution of the squared excess risk as we vary $n_T$ for the oracle estimator and our meta learning estimator trained with different values of $N$. Results are averaged over 100 runs, where for each run we sample new source and target tasks. For $N=250$, the performance of the meta learning is identical to the oracle. It demonstrates that our meta learning strategy successfully leverages the source tasks and that given enough source tasks, it learns at a similar rate of the oracle estimator, leading to a parametric rate of convergence. We refer to Section~\ref{sec:appendix_exp} for additional results.}

\begin{figure}[ht]
    \centering
    \begin{minipage}[b]{0.45\textwidth}
        \centering
        \includegraphics[width=\textwidth]{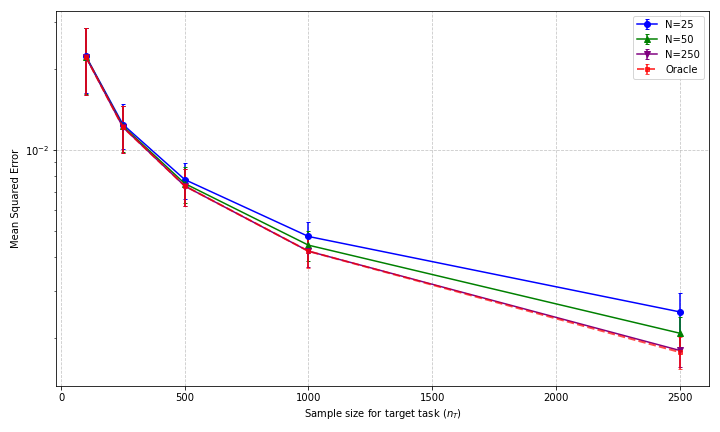}
    \end{minipage}
    \begin{minipage}[b]{0.45\textwidth}
        \centering
     \includegraphics[width=\textwidth]{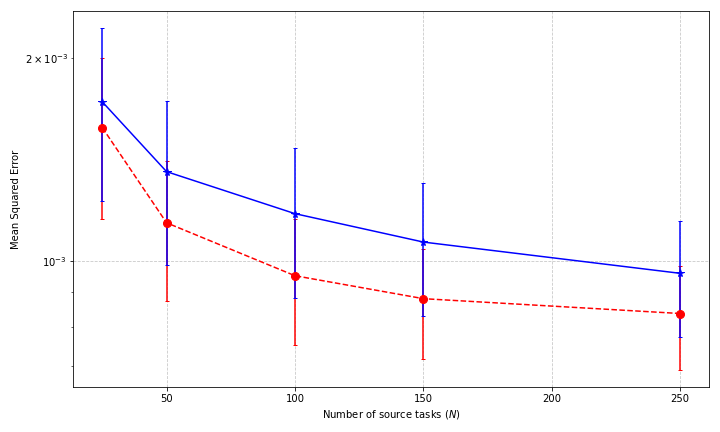}
    \end{minipage}
    \caption{\textbf{(Left)} {Meta Learning versus Oracle: Comparison of the squared excess risk on the target task for the oracle estimator $\hat{f}_{\text{oracle}}$ (dotted red line) and the meta learning estimator $\hat{f}_{T,\lambda_*}$ trained with different number of tasks $N$ (solid lines). $x-$axis represents the size of the dataset for the target task $(n_T)$. \textbf{(Right)} Effect of under-regularization: Comparison of the squared excess risk of the meta learning estimator trained with $\lambda = (nN)^{-\frac{2}{5}}$ (red dotted line) and $\lambda = n^{-\frac{2}{5}}$ (blue solid line). $x-$axis represents the number of source tasks $(N)$. For both figures $n=300$, $s=25$ and results are averaged over $100$ generations of the source and target tasks.}}
    \label{fig:experiment}
\end{figure}

\paragraph{Effect of overfitting.} {To assess the effect of overfitting (see Remark~\ref{rem:overfitting_low}), we compare our meta learning approach trained with $\lambda_1 = (nN)^{-\frac{2}{5}}$ and $\lambda_2 = n^{-\frac{2}{5}}$.  In Figure~\ref{fig:experiment}-(Right), for $s=25$, $n = 300$ and $n_T = 5000$, we plot the evolution of the squared excess risk as we increase $N$ for $\lambda_1$ (red dotted line) and $\lambda_2$ (blue solid line). Results are averaged over 100 runs. It confirms the message of Remark~\ref{rem:overfitting_low} that overfitting (with respect to the usual regularization of kernel ridge regression) on each source task is beneficial for meta learning.  We refer to Section~\ref{sec:appendix_exp} for additional results.}

\section{Analysis Outline}\label{sec:ana_out}
To prove Theorem~\ref{th:pre_training}, we proceed with a  bias-variance decomposition:
\begin{align} \label{eq:bias_variance}
\|\hat{C}_{N,n,\lambda} - C_N\|_{HS} 
&\leq \underbrace{\|\hat{C}_{N,n,\lambda} - \bar C_{N,n,\lambda}\|_{HS}}_{\text{Variance}} 
+ \underbrace{\|\bar C_{N,n,\lambda} - C_N\|_{HS}}_{\text{Bias}},
\end{align}
where $\bar C_{N,n,\lambda}\doteq \frac{1}{N}\sum_i \E (\hat{f}_{i,\lambda}) \otimes \E (\hat{f}_{i,\lambda})$. Next we consider both of these terms separately.

$\bullet$ The variance term can be written as follows
\begin{align*} \|\hat{C}_{N,n,\la}-\bar C_{N,n,\la}\|_{HS} = \left\|\frac{1}{N}\sum_i^N \xi_i\right\|_{HS},
\label{eqn:varterm}
\end{align*}
with $\xi_i \doteq \hat f'_{i,\lambda} \otimes \hat f_{i,\lambda} - \E(\hat f_{i,\lambda}) \otimes \E(\hat f_{i,\lambda}), i \in [N]$.
Thus, the variance term being an average with mean $0$, we would naturally want to bound it via a concentration inequality. However, this requires $\xi_i$ to be well behaved, e.g., bounded or subgaussian. 
A naive upper bound on $\|\xi_i\|_{HS}$ is of the order $\|\hat f_{i,\la}'\|_{\cH}\cdot\|\hat f_{i,\la}\|_{\cH} \leq \lambda^{-1}$ (see Proposition~\ref{prop:xi_as}); however this would lead to a loose concentration bound on the variance term, in particular, such a bound would not go down with the per-task's sample size $n$. 

Therefore, we first establish a high probability bound on $\|\xi_i\|_{HS}$ in terms of $n$ and $\lambda$ as follows. First, recall $f_{i,\lambda}$ from Eq.~\eqref{eq:krr_pop}, and let $\eta_i \doteq \hat f'_{i,\lambda} \otimes \hat{f}_{i,\lambda} - f_{i,\la} \otimes f_{i, \la}$ whereby $\xi_i = \eta_i - \E[\eta_i]$. With some algebra we can get 
$$
\|\eta_i \|_{HS} \leq \|\hat f_{i,\lambda}'- f_{i,\lambda}\|_{\cH}\|\hat f_{i,\lambda}- f_{i,\lambda}\|_{\cH} + \|f_{i}\|_{\cH}(\|\hat f_{i,\lambda}- f_{i,\lambda}\|_{\cH}+\|\hat f_{i,\lambda}'- f_{i,\lambda}\|_{\cH}).
$$
From existing results on kernel ridge regression \citep[see e.g.,][]{fischer2020sobolev}, we can bound $\|\hat{f}_{i,\lambda}-f_{i,\lambda} \|_{\cH}$ in terms of 
both $n$ and $\lambda$, in high-probability. This leads to a high probability bound on $\|\xi_i\|_{HS}$
that takes the form $\P\left(\|\xi_i\|_{HS} \leq V(\delta,n,\la) \right) \geq 1-2e^{-\delta}$ for all $\delta \geq 0$ and $i \in [N]$ (see Theorem~\ref{theo:var_pf} in Section~\ref{sec:prf_sin} for details). Define the event $E_{N,\delta,n,\la}= \cap_{i \in [N]} E_{i,\delta,n,\la}$ where $E_{i,\delta,n,\la}\doteq\{\|\xi_i\|_{HS} \leq V(\delta,n,\la)\}$. We then have
\begin{equation} \label{eq:var_hp_pf}
\P\left( \left\| \frac{1}{N} \sum_{i=1}^N \xi_i \right\|_{HS} \geq \epsilon \right) 
\leq \P\left( \left\| \frac{1}{N} \sum_{i=1}^N \xi_i \right\|_{HS} \geq \epsilon \,\middle|\, E_{N,\delta,n,\lambda} \right) 
+ 2N e^{-\delta}.
\end{equation}
For the first term on the r.h.s, we can now apply Hoeffding inequality (Theorem~\ref{th:hoeffding_hilbert}) since $\xi_i$ conditionally on $E_{N,\delta,n,\la}$ is bounded. However, conditioning on $E_{N,\delta,n,\la}$, the variable $\xi_i$ may no longer have zero mean, a requirement for usual concentration arguments. We therefore proceed by first centering $\xi_i$ around $\E(\xi_i \mid E_{N,\delta,n,\la}) =\E\left(\xi_i \mid E_{i,\delta,n,\la}\right)$ (by independence of the source tasks), and upper-bounding this expectation as 
\[\left\|\E\left[\xi_i \mid E_{i,\delta,n,\la}\right]\right\| =  \left\| \E(\xi_i\mid E_{i,\delta,n,\la}) - \E(\xi_i)\right\| \leq 2\E\left[\|\xi_i\| \mid E_{i,\delta,n,\la}^c \right]\P\left(E_{i  ,\delta,n,\la}^c\right) \leq 4e^{-\delta}\la^{-1},\] 
where we used the upper bound $\la^{-1}$ on $\|\xi_i\|_{HS}$. Then, applying Hoeffding inequality to the first term, we obtain with probability greater than $1-2e^{-\tau}-2Ne^{-\delta}$,
$$
\begin{aligned}
    \left\| \frac{1}{N}\sum_{i=1}^N \xi_i\right\|_{HS} &\leq V(\delta,n,\la)\sqrt{\frac{\tau}{N}} + \frac{4e^{-\delta}}{\lambda} \leq V(\delta,n,\la)\sqrt{\frac{\tau}{N}} + \frac{4}{\lambda N^{12} n^{12}},
\end{aligned}
$$
by choosing $\delta$ (a free parameter) as $12\log (nN)$. In that way, for our choices of $\lambda$ (see Corollary~\ref{cor:learning_rate}), $(\la N^{12}n^{12})^{-1}$ is always of lower order and $2Ne^{-\delta} = o((nN)^{-10})$. Our choice of $V(\delta,n,\la)$ is given in Theorem~\ref{prop:variance_regression} (leading to Eq.~\eqref{eq:sin_risk_sharp}), with the constraint that $n \geq A_{\la}$ (see Theorem~\ref{th:pre_training} for the definition of $A_{\la}$).
For the detailed proof of the variance bound, please refer to Theorem~\ref{theo:var_pf} in Section~\ref{sec:prf_sin}. 

$\bullet$ To bound the bias, we first notice that it can be decomposed in the following way
\[\|\bar C_{N,n,\la}-C_N\|_{HS} \lesssim \frac{1}{N}\sum_{i=1}^N\left\| f_i - \E(\hat{f}_{i,\lambda})\right\|_{\cH}.\] The key is therefore to obtain a good control on $\| f_i - \E(\hat{f}_{i,\lambda})\|_{\cH}$.  We consider two different ways of bounding this term, commensurate with regimes of $r$.

--- When $r \in (0, 1/2]$, we proceed as follows, 
\begin{align*} 
\left\| f_i - \E(\hat{f}_{i,\lambda})\right\|_{\cH} &=\lambda\left\| \E \left( \hat \Sigma_{i,\lambda}^{-1}\right)f_i \right\|_{\cH} = \lambda\left\| \Sigma_{i,\lambda}^{-1/2}\E \left(I+  \Sigma_{i,\lambda}^{-1/2}\left(\hat \Sigma_{i}-\Sigma_i\right)\Sigma_{i,\lambda}^{-1/2} \right)^{-1} \Sigma_{i,\lambda}^{-1/2}f_i \right\|_{\cH}\\
& \leq \lambda\left\| \Sigma_{i,\lambda}^{-1/2}\right\| \left\|\E \left(I+  \Sigma_{i,\lambda}^{-1/2}\left(\hat \Sigma_{i}-\Sigma_i\right)\Sigma_{i,\lambda}^{-1/2} \right)^{-1} \right\| \left\|\Sigma_{i,\lambda}^{-1/2}f_i \right\|_{\cH}. 
\end{align*}
For $r \leq 1/2$, we have $\left\| \Sigma_{i,\lambda}^{-1/2}f_i \right\|_{\cH} = \left\| \Sigma_{i,\lambda}^{r-1/2}\Sigma_{i,\lambda}^{r}f_i \right\|_{\cH} \leq \lambda^{r-1/2}$, 
while $\left\| \Sigma_{i,\lambda}^{-1/2}\right\| \leq \lambda^{-1/2}$. We then have,\[\left\| f_i - \E(\hat{f}_{i,\lambda})\right\|_{\cH} \leq \lambda^{r}\left\|\E \left(I+  \Sigma_{i,\lambda}^{-1/2}\left(\hat \Sigma_{i}-\Sigma_i\right)\Sigma_{i,\lambda}^{-1/2} \right)^{-1}\right\|. \]
For $n \geq A_{\lambda}$, with probability over $1-2e^{-\delta}$---where $\delta$ is chosen as discussed for the variance bound---we can show that $\|(I+  \Sigma_{i,\lambda}^{-1/2}(\hat \Sigma_{i}-\Sigma_i\Sigma_{i,\lambda}^{-1/2} )^{-1}\| \leq 3$, whereby we get with the same probability $\| f_i - \E(\hat{f}_{i,\lambda})\|_{\cH} \leq 3\lambda^r$. Thus, conditioning on this event, we get a final bound 
\[\| f_i - \E(\hat{f}_{i,\lambda})\|_{\cH} \leq 3\lambda^r + 2e^{-\delta}\|f_i\|_{\cH},\]
using the fact that $\| f_i - \E(\hat{f}_{i,\lambda})\|_{\cH} = \lambda\| \E ( \hat \Sigma_{i,\lambda}^{-1})f_i \|_{\cH}$ is always at most 
$\|f_i\|_{\cH}$.

--- When $r \in (1/2,1]$, we proceed as follows, 
$$
\begin{aligned}
    \left\| f_i - \E(\hat{f}_{i,\lambda})\right\|_{\cH} &=\lambda\left\| \E \left( \hat \Sigma_{i,\lambda}^{-1}\right)f_i \right\|_{\cH}= \lambda\left\| \E \left( \hat \Sigma_{i,\lambda}^{-1}\Sigma_{i,\lambda}\right) \Sigma_{i,\lambda}^{-1}f_i \right\|_{\cH} \\
    &\leq \lambda \left\| \E \left( \hat \Sigma_{i,\lambda}^{-1}\Sigma_{i,\lambda}\right) \right\| \left\|\Sigma_{i,\lambda}^{r-1}\Sigma_{i,\lambda}^{-r}\Sigma_{i}^r \Sigma_{i}^{-r}f_i \right\|_{\cH}\\
    &\leq \lambda^r\left\| \E \left( \hat \Sigma_{i,\lambda}^{-1}\Sigma_{i,\lambda}\right)\right\| =\lambda^r\left\| \Sigma_{i,\lambda}\E \left( \hat \Sigma_{i,\lambda}^{-1}\right)\right\|.
\end{aligned}
$$
We then use the following derivation \[\hat \Sigma_{i,\lambda}^{-1} = \left(\hat \Sigma_i + \lambda\right)^{-1} =  \left(\Sigma_i + \lambda -(\Sigma_i-\hat \Sigma_i) \right )^{-1} = \Sigma_{i,\lambda}^{-1}\left(I -(\Sigma_i-\hat \Sigma_i)\Sigma_{i,\lambda}^{-1} \right )^{-1}.\] We are left with bounding the term $\E\|(I -(\Sigma_i-\hat \Sigma_i)\Sigma_{i,\lambda}^{-1} )^{-1}\|$ which can be obtained by using a Neumann series. For a detailed analysis of the bias, see Theorem~\ref{theo:bias_pf} of  Section~\ref{sec:prf_sin}

\section{Conclusion}

We address the problem of meta-learning with nonlinear representations, providing theoretical guarantees for its effectiveness. Our study focuses on the scenario where the shared representation maps inputs nonlinearly into an infinite dimensional RKHS. By leveraging the smoothness of task-specific regression functions and employing careful regularization techniques, the paper demonstrates that biases introduced in the nonlinear representation can be mitigated. Importantly, the derived guarantees show that the convergence rates in learning the common representation can scale with the number of tasks, in addition to the number of samples per task. The analysis extends previous results obtained in the linear setting, and highlights the challenges and subtleties specific to the nonlinear case. The findings presented in this work open up several avenues for future research, which include: exploration of different types of nonlinear representations beyond RKHS,  alternative subspace estimation techniques, and further refinement of trade-offs between bias and variance.

\section*{Acknowledgements}

Dimitri Meunier, Zhu Li, and Arthur Gretton were supported by the Gatsby Charitable Foundation. Zhu Li is also funded by Imperial College London through the Chapman Fellowship. Samory Kpotufe is thankful for support from NSF CNS-2334997, and a Sloan 2021 Fellowship over the bulk of this study. 

\bibliography{file}
\newpage
\appendix

\section*{Appendix}

In Section~\ref{sec:appendix_proofs} we present the proofs of the main results:
\begin{itemize}
    \item \ref{sec:proof_th_1}: proof of Theorem~\ref{th:inference_bound};
    \item \ref{sec:wedin}: proof of Proposition~\ref{prop:davis_kahan_meta};
    \item \ref{sec:prf_sin}: proofs of Theorem~\ref{th:pre_training} and Example~\ref{ex:finite};
    \item \ref{sec:conv_rate}: proof of Corollary~\ref{cor:learning_rate};
    \item \ref{sec:algo_proof}: proofs Section~\ref{sec:algo};
    \item \ref{sec:exp_setting}: proof of Remark~\ref{rem:rem_exp}.
\end{itemize}
In Section~\ref{ref:aux_results}, we present auxiliary results used for the main proofs. In Section~\ref{sec:concentration}, we list concentration inequalities used in the different proofs. Finally, in Section~\ref{sec:appendix_exp}, we present additional results from our experimental results.

\section{Proofs of the Main Results} \label{sec:appendix_proofs}

\subsection{Proof of Theorem~\ref{th:inference_bound}} \label{sec:proof_th_1}
Before embarking on the proof of Theorem~\ref{th:inference_bound} we need a few preliminary results and definitions. We first introduce the empirical counterpart of the covariance operator, for $i \in [N] \cup \{T\}$ and $m = n$ if $i \in [N]$, $m = n_T$ if $i = T$,
\begin{equation} \label{eq:cov_empirical}
    \hat \Sigma_i = \frac{1}{m}\sum_{j=1}^m \phi(x_{i,j}) \otimes \phi(x_{i,j}) = \frac{1}{m} \Phi_i\Phi_i^*
\end{equation}
where $\Phi_i = [\phi(x_{i,1}), \ldots, \phi(x_{i,m})]$ is defined as
\begin{align*}
  \Phi_i \colon  & \mathbb{R}^m \longrightarrow \mathcal{H}\\[-1ex]
  v & \longmapsto \sum_{j=1}^m v_j\phi(x_{i,j})
\end{align*}
and admits as adjoint the sampling operator for task $i$, 
\begin{align*}
  \Phi_i^* \colon \mathcal{H} & \longrightarrow \mathbb{R}^m \\[-1ex]
  f & \longmapsto (\langle f, \phi(x_{i,j}) \rangle)_{j=1}^m
\end{align*}
The Gram matrix for each task is $K_i \doteq \Phi_i^*\Phi_i \in \R^{m \times m}$. For any linear operator $F: \cH \to \cH$ and scalar $\gamma>0$, we define $F_{\gamma}\doteq F + \gamma I_{\cH}$. With those notations and taking derivatives with respect to $f$ in Eq.~\eqref{eq:krr}, we can derive a closed-form expression for $\hat f_{i,\la}$, for $i \in [N]$,
\begin{equation}\label{eq:charac_f_la}
\hat{f}_{i,\lambda} = \hat \Sigma_{i,\la}^{-1} \frac{1}{n}\Phi_iY_i, \quad Y_i \doteq (y_{i,1},\dots,y_{i,n})^{\top} \in \R^n. 
\end{equation}
Recall that $\hat \cH_s$ is a RKHS with canonical feature map $\hat P \phi(\cdot)$ equipped with the same inner product as $\cH$. Hence, the covariance operator in that space equipped with the marginal distribution $\mu_T$ on $\cX$ is defined as 
\begin{equation} \label{eq:sigma_P}
    \Sigma_{\hat P} : = \E_{X \sim \mu_T} \left[\hat P \phi(X) \otimes \hat P \phi(X) \right] = \hat P \Sigma_T \hat P,
\end{equation}
which is a positive semi-definite self-adjoint operator. The counterpart of Eq.~\eqref{eq:cov_empirical} in $\hat \cH_s$ is 
\begin{align} 
\hat \Sigma_{\hat P} &\doteq \hat P \hat{\Sigma}_T \hat P, \nonumber \\
\Phi_{\hat P} &\doteq \hat P\Phi_T. \nonumber
\end{align}
Then $\hat \Sigma_{\hat P} = \frac{1}{n_T}\Phi_{\hat P}\Phi_{\hat P}^*$ is the empirical covariance in $\hat \cH_s$ for the target task. Therefore, since $\hat f_{T,\la_*}$ is the ridge estimator in $\hat \cH_s$ (see Eq.~\eqref{eq:krr_inference}), in light of Eq.~\eqref{eq:charac_f_la}, we have 
\begin{equation}\label{eq:charac_f_T_la}
\hat{f}_{T,\lambda_*} = \hat \Sigma_{\hat P,\la_*}^{-1} \frac{1}{n_T}\Phi_{\hat P}Y_T, \quad Y_T \doteq (y_{T,1},\dots,y_{T,n_T})^{\top} \in \R^{n_T}. 
\end{equation}
The next technical lemmata are useful for the proof of Theorem~\ref{th:inference_bound}. Recall that $\hat P$ is the projection onto $\hat \cH_s = \operatorname{span}\{\hat v_1, \ldots, \hat v_s\}$, hence $\hat P = \hat V \hat V^*$ where $\hat V = [\hat v_1, \ldots, \hat v_s]$ and $\hat V^*\hat V = I_s$ with $I_s$ the identity in $\R^{s\times s}$. 
\begin{lemma} \label{lma:b1}
    $\hat P\hat \Sigma_{T,\la_*}\hat P = \hat V(\hat V^*\hat \Sigma_T \hat V + \la_* I_s)\hat V^*$ and $\hat \Sigma_{\hat P, \la_*}^{-1}\hat P\hat \Sigma_{T,\la_*}\hat P = \hat P$.
\end{lemma}
\begin{proof}
    The first identity is obtained by plugging $\hat P = \hat V \hat V^* $ and using $\hat V^* \hat V = I_s$. For the second identity, we have  
\begin{align*}
    \hat \Sigma_{\hat P, \la_*}^{-1}\hat P\hat \Sigma_{T,\la_*}\hat P 
    &= (\hat V \hat V^* \hat \Sigma_{T}\hat V\hat V^* + \la_* I_{\cH})^{-1}\hat V(\hat V^*\hat \Sigma_T \hat V + \la_* I_s)\hat V^*\\ &= \hat V( \hat V^* \hat \Sigma_{T}\hat V\hat V^*\hat V + \la_* I_{s})^{-1}(\hat V^*\hat \Sigma_T \hat V + \la_* I_s)\hat V^* \\ &= \hat V( \hat V^* \hat \Sigma_{T}\hat V + \la_* I_{s})^{-1}(\hat V^*\hat \Sigma_T \hat V+ \la_* I_s)\hat V^* \\ &= \hat V\hat V^* = \hat P,
\end{align*}
where in the second equality, we used the matrix inversion lemma.
\end{proof}
\begin{lemma} \label{lma:b1_bis}
    $\hat P-\hat \Sigma_{\hat P, \la_*}^{-1} \hat P\hat \Sigma _T = \la_* \hat \Sigma_{\hat P, \la_*}^{-1}\hat P - \hat \Sigma_{\hat P, \la_*}^{-1} \hat P\hat \Sigma_{T,\la_*}\hat P_{\perp}$, where $\hat P_{\perp} \doteq I_{\cH} - \hat P$.
\end{lemma}
\begin{proof}
\begin{align*}
        \hat P-\hat \Sigma_{\hat P, \la_*}^{-1} \hat P\hat \Sigma _T 
        &= \hat P - \hat \Sigma_{\hat P, \la_*}^{-1} \hat P\hat \Sigma_{T,\la_*} + \la_*\hat \Sigma_{\hat P, \la_*}^{-1} \hat P\\ &= \hat P - \hat \Sigma_{\hat P, \la_*}^{-1} \hat P\hat \Sigma_{T,\la_*}(\hat P + \hat P_{\perp}) + \la_*\hat \Sigma_{\hat P, \la_*}^{-1} \hat P\\ 
        &= \la_* \hat \Sigma_{\hat P, \la_*}^{-1}\hat P - \hat \Sigma_{\hat P, \la_*}^{-1} \hat P\hat \Sigma_{T,\la_*}\hat P_{\perp},
\end{align*}
where the last equality follows from Lemma~\ref{lma:b1}.
\end{proof}
\begin{lemma} \label{ref:lma_b2}
    $\left\|\hat{\Sigma}_{\hat P,\la_*}^{-1/2} \hat P\hat \Sigma_{T,\la_*}^{1/2}\right\| \in \{0,1\}$. 
\end{lemma}
\begin{proof}
    First, note that $h_{\la_*} \doteq \hat{\Sigma}_{\hat P,\la_*}^{-1} \hat P  = \hat V( \hat V^* \hat \Sigma_{T,\la_*}\hat V)^{-1}\hat V^*$. Secondly, note that 
    \begin{align*}
    h_{\la_*}\hat{\Sigma}_{\hat P,\la_*} h_{\la_*} &= \hat{\Sigma}_{\hat P,\la_*}^{-1} \hat P\hat{\Sigma}_{\hat P,\la_*}\hat{\Sigma}_{\hat P,\la_*}^{-1} \hat P \\ 
    &= \hat{\Sigma}_{\hat P,\la_*}^{-1} \hat P \\
    &= h_{\la_*}.
    \end{align*}
    Thirdly, 
    \begin{align*}
    h_{\la_*}\hat{\Sigma}_{T,\la_*} h_{\la_*}  &= \hat V(\hat V^* \hat \Sigma_{T,\la_*}\hat V)^{-1}\hat V^*(\hat \Sigma_{T} + \la_* I_{\cH})\hat V( \hat V^* \hat \Sigma_{T,\la_*}\hat V)^{-1}\hat V^*\\
    &= \hat V( \hat V^* \hat \Sigma_{T,\la_*}\hat V)^{-1}(\hat V^* \hat \Sigma_{T}\hat V + \la_* I_{s})( \hat V^* \hat \Sigma_{T,\la_*}\hat V)^{-1}\hat V^*\\
    &= \hat V( \hat V^* \hat \Sigma_{T,\la_*}\hat V)^{-1} \hat V^* \\ &= h_{\la_*}.
    \end{align*}
    Hence, using that for any bounded linear operator $F:\cH \to \cH$, $\|F\|^2 = \|F^*F\|$,
    \begin{align*}
        \left\|\hat{\Sigma}_{\hat P,\la_*}^{-1/2} \hat P\hat \Sigma_{T,\la_*}^{1/2}\right\|^4  &= \left\|\hat{\Sigma}_{\hat P,\la_*}^{1/2} h_{\la_*}\hat \Sigma_{T,\la_*}^{1/2}\right\|^4 \\ &= \left\|\hat \Sigma_{T,\la_*}^{1/2} h_{\la_*}\hat{\Sigma}_{\hat P,\la_*} h_{\la_*}\hat \Sigma_{T,\la_*}^{1/2}\right\|^2 \\ &=  \left\|\hat \Sigma_{T,\la_*}^{1/2}  h_{\la_*}\hat \Sigma_{T,\la_*}^{1/2}\right\|^2 \\ &=  \left\|\hat \Sigma_{T,\la_*}^{1/2}  h_{\la_*}\hat \Sigma_{T,\la_*}h_{\la_*}\hat \Sigma_{T,\la_*}^{1/2}\right\| \\ &= \left\|\hat \Sigma_{T,\la_*}^{1/2}  h_{\la_*}\hat \Sigma_{T,\la_*}^{1/2}\right\|.
    \end{align*}
Since $\left\|\hat \Sigma_{T,\la_*}^{1/2}  h_{\la_*}\hat \Sigma_{T,\la_*}^{1/2}\right\|^2 = \left\|\hat \Sigma_{T,\la_*}^{1/2}  h_{\la_*}\hat \Sigma_{T,\la_*}^{1/2}\right\|$, it belongs to $\{0,1\}$, and therefore,
$$\left\|\hat{\Sigma}_{\hat P,\la_*}^{-1/2} \hat P\hat \Sigma_{T,\la_*}^{1/2}\right\|  = \left\|\hat \Sigma_{T,\la_*}^{1/2}  h_{\la_*}\hat \Sigma_{T,\la_*}^{1/2}\right\|^{1/2} \in \{0,1\}.
$$
\end{proof}
\begin{proof}[Proof of Theorem~\ref{th:inference_bound}]
Under Assumptions \ref{asst:meta} and \ref{asst:rich} with $s \geq 1$, we have the following excess risk decomposition
\begin{equation}\label{eq:risk_decomp}
    \cE_{\mu_T}(\hat{f}_{T,\lambda_*}) = \|\hat{f}_{T,\lambda_*} - f_T\|_{L_2(\mu_T)} \leq \|\hat{f}_{T,\lambda_*} - \hat P f_{T}\|_{L_2(\mu_T)} + \|\hat P_{\perp}Pf_T\|_{L_2(\mu_T)},
\end{equation} 
where we used $\hat P_{\perp} + \hat P = I_{\cH}$ and $f_T = Pf_T$ since $f_T \in \cH_s$. Instead of working with the $L_2-$norm we can work with the $\cH$ norm as for any $f \in \hat \cH_s$,

\begin{align}\label{eq:norm_conv_0}
\|f\|_{L_2(\mu_T)}^2 &= \E_{X \sim \mu_T}[f(X)^2] \nonumber \\ &= \E_{X \sim \mu_T}[\langle f,\hat P\phi(X) \rangle_{\cH}^2] \nonumber \\ &= \E_{X \sim \mu_T}[\langle (\hat P\phi(X) \otimes \hat P\phi(X))f, f\rangle_{\cH}] \nonumber \\ &= \langle \Sigma_{\hat P}f, f \rangle_{\cH} \nonumber \\ &= \|\Sigma_{\hat P}^{1/2}f\|_{\cH}^2,
\end{align}
where in the second equality we used the reproducing property in $\hat \cH_s$  and in the fourth equality we used the definition of $\Sigma_{\hat P}$ in Eq.~\eqref{eq:sigma_P}. Similarly, for any $f \in \cH$,
\begin{equation} \label{eq:norm_conv}
    \|f\|_{L_2(\mu_T)} = \|\Sigma_{T}^{1/2}f\|_{\cH}.
\end{equation}
Therefore, we have  
\begin{equation} \label{eq:risk_dec_1}
    \|\hat P_{\perp}Pf_T\|_{L_2(\mu_T)} = \|\Sigma_{T}^{1/2}\hat P_{\perp}Pf_T\|_{\cH} \leq \|\Sigma_{T}\|^{1/2}\|\hat P_{\perp}P\|\|f_T\|_{\cH}  \leq \kappa\|\hat P_{\perp}P\|\|f_T\|_{\cH},
\end{equation}
where we used that for a bounded kernel (here $\sup_{x,x' \in \cX}K(x,x') \doteq \kappa^2 < \infty$), for any marginal distribution, the trace norm (and hence the operator norm) of the associated covariance operator is bounded by $\kappa$ \citep[see][Theorem 4.27]{steinwart2008support}. On the other hand, by Eq.~\eqref{eq:norm_conv_0} and Eq.~\eqref{eq:charac_f_T_la}, we have
\begin{align} \label{eq:risk_dec_2}
    \|\hat{f}_{T,\la_*} &- \hat Pf_T\|_{L_2(\mu_T)} = \left\|\Sigma_{\hat P}^{1/2}  \left( \hat{f}_{T,\la_*} - \hat Pf_T\right)\right\|_{\cH}  \nonumber \\
    &=\left\|\Sigma_{\hat P}^{1/2} \left(\hat \Sigma_{\hat P, \la_*}^{-1}\frac{1}{n_T}\Phi_{\hat P}Y_T-\hat P f_T  \right)\right\|_{\cH} \nonumber \\
    &= \left\|\Sigma_{\hat P}^{1/2} \left(\hat \Sigma_{\hat P, \la_*}^{-1} \frac{1}{n_T}\Phi_{\hat P}(Y_T-\Phi_{T}^*f_T +\Phi_{T}^*f_T)-\hat P f_T  \right)\right\|_{\cH}  \nonumber\\
    &\leq \left\| \Sigma_{\hat P}^{1/2} \hat \Sigma_{\hat P, \la_*}^{-1} \frac{1}{n_T}\Phi_{\hat P}(Y_T-\Phi_{T}^*f_T)\right\|_{\cH} +  \left\|\Sigma_{\hat P}^{1/2} \left(\hat \Sigma_{\hat P, \la_*}^{-1} \frac{1}{n_T}\Phi_{\hat P}\Phi_{T}^*f_T-\hat P f_T  \right)\right\|_{\cH} \nonumber \\
    &= \underbrace{\left\| \Sigma_{\hat P}^{1/2} \hat \Sigma_{\hat P, \la_*}^{-1} \frac{1}{n_T}\Phi_{\hat P}(Y_T-\Phi_{T}^*f_T)  \right\|_{\cH}}_{\doteq A} + \underbrace{\left\| \Sigma_{\hat P}^{1/2} \left(\hat \Sigma_{\hat P, \la_*}^{-1} \hat P\hat \Sigma_{T}-\hat P   \right)f_T\right\|_{\cH}}_{\doteq B}
\end{align}
where in the last equality we used $\Phi_{\hat P}\Phi_{T}^* = \hat P \Phi_{T}\Phi_{T}^* = n_T \hat P\hat \Sigma_T.$\\
\textbf{Term A.} For term A, we have 
\begin{align*}
    A &= \left\| \Sigma_{\hat P}^{1/2} \hat \Sigma_{\hat P, \la_*}^{-1} \frac{1}{n_T}\Phi_{\hat P}(Y_T-\Phi_{T}^*f_T) \right\|_{\cH} \\ &\leq \left\|\Sigma_{\hat P}^{1/2} \hat \Sigma_{\hat P, \la_*}^{-1/2} \right\|\left\|\hat \Sigma_{\hat P, \la_*}^{-1/2}\Sigma_{\hat P, \la_*}^{1/2}\right\|\left\|\Sigma_{\hat P, \la_*}^{-1/2}\frac{1}{n_T}\Phi_{\hat P}(Y_T-\Phi_{T}^*f_T)\right\|_{\cH} \\ &\leq \left\|\Sigma_{\hat P, \la_*}^{1/2} \hat \Sigma_{\hat P, \la_*}^{-1/2} \right\|^2\left\|\Sigma_{\hat P, \la_*}^{-1/2}\frac{1}{n_T}\Phi_{\hat P}(Y_T-\Phi_{T}^*f_T)\right\|_{\cH},
\end{align*}
where in the last inequality, we used
$$
\left\|\Sigma_{\hat P}^{1/2} \hat \Sigma_{\hat P, \la_*}^{-1/2} \right\| = \left\|\Sigma_{\hat P}^{1/2}\Sigma_{\hat P, \la_*}^{-1/2}\Sigma_{\hat P, \la_*}^{1/2} \hat \Sigma_{\hat P, \la_*}^{-1/2} \right\| \leq \left\|\Sigma_{\hat P}^{1/2}\Sigma_{\hat P, \la_*}^{-1/2}\right\|\left\|\Sigma_{\hat P, \la_*}^{1/2} \hat \Sigma_{\hat P, \la_*}^{-1/2} \right\| \leq \left\|\Sigma_{\hat P, \la_*}^{1/2} \hat \Sigma_{\hat P, \la_*}^{-1/2} \right\|.
$$
To deal with $\left\|\Sigma_{\hat P, \la_*}^{1/2}\hat \Sigma_{\hat P, \la_*}^{-1/2}\right\|$ we apply the first part of Proposition~\ref{prop:op_trick} to $C=\hat \Sigma_{\hat P}$, $D= \Sigma_{\hat P}$ and $\la_* > 0$, we get
\begin{equation} \label{eq:transform_a1}
    \left\|\Sigma_{\hat P, \la_*}^{1/2} \hat \Sigma_{\hat P, \la_*}^{-1/2} \right\| = \left\|\left(I - B_{T,\la_*} \right)^{-1}\right\|^{1/2},
\end{equation}
where $B_{T,\la_*} \doteq \Sigma_{\hat P, \la_*}^{-1/2}\hat P(\Sigma_T - \hat{\Sigma}_T)\hat P\Sigma_{\hat P, \la_*}^{-1/2}$. We control $B_{T,\la_*}$ in operator norm with a Bernstein-type concentration inequality for Hilbert-Schmidt operator valued random variables. By Proposition~\ref{prop:hsu_bernstein_finite_dim}, for $\la_*>0$, $\tau \geq 2.6$ and $n_T \geq 1$, the following operator norm bound is satisfied with $\mu_T^{n_T}$-probability not less than $1-e^{-\tau}$
$$
\left\|B_{T,\la_*}\right\| \leq \frac{2 \kappa^2(\tau+ \log(s))}{3 \la_*n_T}+\sqrt{\frac{4 \kappa^2(\tau+ \log(s))}{\la_* n_T}}
$$
conditionally on $\mathcal{D}_i = \{(x_{i,j}, y_{i,j})_{j=1}^{2n}\}, i \in [N]$. Therefore, if  $n_T \geq 6(\tau+\log(s))\kappa^2 \la_*^{-1}$, Proposition~\ref{prop:hsu_bernstein_finite_dim} yields
$$
\left\|B_{T,\la_*}\right\| \leq \frac{2}{3} \cdot \frac{(\tau+\log(s))\kappa^2 \la_*^{-1}}{n_T}+\sqrt{4 \cdot \frac{(\tau+\log(s))\kappa^2 \la_*^{-1}}{n_T }} \leq \frac{2}{3} \cdot \frac{1}{6}+\sqrt{4 \cdot \frac{1}{6}}<0.93
$$
with $\mu_T^{n_T}$-probability not less than $1-e^{-\tau}$. Consequently, the inverse of
$I-B_{T,\la_*}$ can be represented by the Neumann series. In particular, the Neumann series gives us the following bound
\begin{equation} \label{eq:neumann}
\left\|\Sigma_{\hat P, \la_*}^{1/2} \hat \Sigma_{\hat P, \la_*}^{-1/2} \right\|^2 = \left\|\left(I - B_{T,\la_*} \right)^{-1}\right\|
\leq \sum_{k=0}^{\infty}\left\|B_{T,\la_*}\right\|^k
\leq \sum_{k=0}^{\infty}\left(0.93\right)^k\leq 15
\end{equation}
with $\mu_T^{n_T}$-probability not less than $1-e^{-\tau}$. Hence, for $\la_*>0$, $\tau \geq 2.6$ and $n_T \geq 6(\tau+\log(s))\kappa^2 \la_*^{-1}$, conditionally on $\mathcal{D}_i = \{(x_{i,j}, y_{i,j})_{j=1}^{2n}\}, i \in [N]$ with $\mu_T^{n_T}$-probability not less than $1-e^{-\tau}$, 
\begin{equation}\label{eq:term_A_first_bound}
    A \leq 15\left\|\Sigma_{\hat P, \la_*}^{-1/2}\frac{1}{n_T}\Phi_{\hat P}(Y_T-\Phi_{T}^*f_T)\right\|_{\cH}.
\end{equation}
To deal with the remaining term in term A, note that 
$$
\Sigma_{\hat P, \la_*}^{-1/2}\frac{1}{n_T}\Phi_{\hat P}(Y_T-\Phi_{T}^*f_T)=\frac{1}{n_T} \sum_{i=1}^{n_T}\Sigma_{\hat P, \la_*}^{-1/2}\hat P \phi(x_{T,i})(y_{T,i} - f_T(x_{T,i})) = \frac{1}{n_T} \sum_{i=1}^{n_T}\xi(x_{T,i},y_{T,i})
$$
where 
\begin{align*}
  \xi \colon  & \cX \times \R \longrightarrow \mathcal{H}\\[-1ex]
  &(x,y) \longmapsto (y - f_T(x))\Sigma_{\hat P, \la_*}^{-1/2}\hat P \phi(x).
\end{align*}
We can bound this quantity in probability using a Bernstein concentration inequality for Hilbert space valued random variables (Theorem~\ref{th:pinelis}). First note that \[\E_{(X,Y) \sim \mu_{T}}\left[\xi(X,Y) \right] = \E_{(X,Y) \sim \mu_T}\left[\Sigma_{\hat P, \la_*}^{-1/2}\hat P \phi(X)(\E\left[Y \mid X\right] - f_T(X)) \right]=0.\]
Consequently, to apply Theorem~\ref{th:pinelis}, it remains to bound the $m$-th moment of $\xi$, for $m \geq 2$,
$$
\mathbb{E}_{(X,Y) \sim \mu_{T}}\left\|\xi(X,Y)\right\|_{\cH}^{m}=\int_{\cX}\left\|\Sigma_{\hat P, \la_*}^{-1/2}\hat P \phi(x)\right\|_{\cH}^{m} \int_{\R}|y-f_{T}(x)|^{m} \mu_{T}(x, \mathrm{~d} y) \mathrm{d} \mu_T(x) .
$$
The inner integral can be bounded by Assumption $\ref{asst:mom}$, for $\mu_T$-almost all $x \in \cX$, for $m \geq 2$, 
$$
\int_{\R}|y-f_{T}(x)|^{m} \mu_{T}(x, \mathrm{~d} y) \leq 2^m Y_{\infty}^m \leq \frac{1}{2} m ! (2Y_{\infty})^{m}.
$$
Then, by Lemma~\ref{lma:eff_dim_exp}, and since $\operatorname{dim}(\hat \cH_s)=s$,
$$
\int_{\cX}\left\|\Sigma_{\hat P, \la_*}^{-1/2}\hat P \phi(x)\right\|_{\cH}^{2} \mathrm{d} \mu_T(x) = \operatorname{Tr}(\Sigma_{\hat P, \la_*}^{-1}\Sigma_{\hat P}) \leq s.
$$
Since $\|\hat P\|\leq 1$ and $\sup_{x,x' \in \cX}K(x,x') \doteq \kappa^2 < \infty$, we have for all $x \in \cX$, 
\begin{align*}
\left\|\Sigma_{\hat P, \la_*}^{-1/2}\hat{P}\phi(x)\right\|_{\mathcal{H}} \leq \left\|\Sigma_{\hat P, \la_*}^{-1/2}\right\|\|\hat{P}\|\|\phi(x)\|_{\mathcal{H}} \leq \frac{\kappa}{\sqrt{\la_*}}.  
\end{align*}
Therefore,
$$
\begin{aligned}
\mathbb{E}_{(X,Y) \sim \mu_{T}}\left\|\xi(X,Y)\right\|_{\cH}^{m} &\leq \frac{1}{2} m ! (2Y_{\infty})^{m} \left(\frac{\kappa}{\sqrt{\la_*}}\right)^{m-2} \int_{\cX}\left\|\Sigma_{\hat P, \la_*}^{-1/2}\hat P\phi(x)\right\|_{\cH}^{2} \mathrm{d} \mu_T(x) \\ &\leq \frac{1}{2} m ! (2Y_{\infty})^2 \left(2Y_{\infty}\frac{\kappa}{\sqrt{\la_*}}.\right)^{m-2}s.
\end{aligned}
$$
Applying Theorem~\ref{th:pinelis} and Proposition~\ref{prop:bound_conversion} with $v^2 = (2Y_{\infty})^2s$ and $b = 2Y_{\infty} \kappa \la_*^{-1/2},$ we get that for $\tau \geq 1$ and $n_T \geq 1$, with probability at least $1-2e^{-\tau}$,
\begin{align}
\left\|\frac{1}{n_T} \sum_{i=1}^{n_T}\Sigma_{\hat P, \la_*}^{-1/2}\hat P\phi(x_{T,i})(y_{T,i} - f_T(x_{T,i})) \right\| \leq \sqrt{\frac{2\tau (2Y_{\infty})^2s}{n_T}} + \frac{4\tau Y_{\infty}\kappa}{n_T \sqrt{\la_*}}, \nonumber
\end{align}
conditionally on $\mathcal{D}_i = \{(x_{i,j}, y_{i,j})_{j=1}^{2n}\}, i \in [N]$. Therefore, merging with Eq.~\eqref{eq:term_A_first_bound} and using a union bound, for $\la_*>0$, $\tau \geq 2.6$ and $n_T \geq 6(\tau+\log(s))\kappa^2 \la_*^{-1}$, conditionally on $\mathcal{D}_i = \{(x_{i,j}, y_{i,j})_{j=1}^{2n}\}, i \in [N]$, with $\mu_T^{n_T}$-probability not less than $1-3e^{-\tau}$
\begin{equation} \label{eq:bound_A}
A \leq 15\left(\sqrt{\frac{8\tau Y_{\infty}^2s}{n_T}} + \frac{4\tau Y_{\infty}\kappa}{n_T \sqrt{\la_*}}\right).
\end{equation}
\textbf{Term B.} By Lemma~\ref{lma:b1_bis}, we have 
$$
B = \left\| \Sigma_{\hat P}^{1/2} \left(\hat \Sigma_{\hat P, \la_*}^{-1} \hat P\hat \Sigma _T-\hat P   \right)f_T\right\|_{\cH} \leq \underbrace{\la_*\left\|\Sigma_{\hat P}^{1/2} \hat \Sigma_{\hat P, \la_*}^{-1}\hat Pf_T\right\|_{\cH}}_{\doteq B.1} + \underbrace{\left\| \Sigma_{\hat P}^{1/2} \hat \Sigma_{\hat P, \la_*}^{-1} \hat P\hat \Sigma_{T,\la_*}\hat P_{\perp}f_T\right\|_{\cH}}_{\doteq B.2}.
$$
For B.1,
\begin{align}
B.1 &\leq \la_* \left\|\Sigma_{\hat P}^{1/2}\hat \Sigma_{\hat P, \la_*}^{-1/2}\right\| \left\|\hat \Sigma_{\hat P, \la_*}^{-1/2}\right\|\|\hat P\|\left\|f_T\right\|_{\cH} \nonumber \\ &\leq \sqrt{\la_*} \left\|\Sigma_{\hat P, \la_*}^{1/2}\hat \Sigma_{\hat P, \la_*}^{-1/2}\right\|\left\|f_T\right\|_{\cH} \nonumber.
\end{align}
We encountered the first term when we bounded Term A (see Eqs.~\eqref{eq:transform_a1} and \eqref{eq:neumann})). For $\la_* > 0$, $\tau \geq 2.6$ with $n_T \geq 6(\tau+\log(s))\kappa^2 \la_*^{-1}$, with probability at least $1-e^{-\tau}$,
$$
\left\|\Sigma_{\hat P, \la_*}^{1/2}\hat \Sigma_{\hat P, \la_*}^{-1/2}\right\| \leq \sqrt{15},
$$
conditionally on $\mathcal{D}_i = \{(x_{i,j}, y_{i,j})_{j=1}^{2n}\}, i \in [N]$. Hence,
\begin{equation} \label{eq:bound_b1}
    B.1 \leq \sqrt{15\la_*} \left\|f_T\right\|_{\cH}.
\end{equation}
For term B.2., for $\la_* > 0$, $\tau \geq 2.6$ with $n_T \geq 6(\tau+\log(s))\kappa^2 \la_*^{-1}$, with probability at least $1-e^{-\tau}$
\begin{align*}
B.2 &\leq \left\|\Sigma_{\hat P}^{1/2}\hat{\Sigma}_{\hat P,\la_*}^{-1/2}\right\|\left\|\hat{\Sigma}_{\hat P,\la_*}^{-1/2} \hat P\hat \Sigma_{T,\la_*}^{1/2}\right\|\left\|\hat\Sigma_{T,\la_*}^{1/2}
\right\| \left\|\hat{P}_{\perp}Pf_T\right\|_{\cH}\\
&\leq \sqrt{15}\left(\kappa + \la_*\right)^{1/2} \left\|\hat{P}_{\perp}P\right\|\left\|f_T\right\|_{\cH},
\end{align*}
where we used Lemma~\ref{ref:lma_b2} and Eqs.~\eqref{eq:transform_a1} and \eqref{eq:neumann} again. 
Putting together Eq.~\eqref{eq:risk_decomp}, Eq.~\eqref{eq:risk_dec_1}, Eq.~\eqref{eq:risk_dec_2}, Eq.~\eqref{eq:bound_A} and Eq.~\eqref{eq:bound_b1}, for $\la_*>0$, $\tau \geq 2.6$ and 
$$
n_T \geq 6\kappa^2 \la_*^{-1}\left(\tau+\log\left(s\right)\right),
$$
conditionally on $\mathcal{D}_i = \{(x_{i,j}, y_{i,j})_{j=1}^{2n}\}, i \in [N]$ with $\mu_T^{n_T}$-probability not less than $1-3e^{-\tau}$,
\begin{align} 
    \cE_{\mu_T}(\hat{f}_{T,\la_*}) &\leq c \left\{ \left(\sqrt{\frac{\tau Y_{\infty}^2s}{n_T}} + \frac{\tau Y_{\infty}\kappa}{n_T \sqrt{\la_*}}\right) + \sqrt{\la_*} \left\|f_T\right\|_{\cH} + \sqrt{\kappa+ \la_*}\left\|\hat{P}_{\perp}P\right\| \left\|f_T\right\|_{\cH}\right\}, \nonumber
\end{align}
where $c$ is a universal constant. 
\end{proof}

\subsection{Proof of Proposition~\ref{prop:davis_kahan_meta}}\label{sec:wedin}
We prove the following infinite dimensional version of Wedin's $\sin-\Theta$ Theorem. 

\begin{theorem} \label{th:wedin_generalized}
 Let $A: H \rightarrow H$ and $\widehat{A}: H \rightarrow H$ be compact operators on a separable Hilbert space $H$ with nonincreasingly ordered singular values $\left(\gamma_i\right)_{i\geq 1}$ and $\left(\hat \gamma_i\right)_{i \geq 1}$ respectively. Let $s \leq \min \{\operatorname{rank}(A), \operatorname{rank}(\widehat{A})\}$ and assume $\gamma_s > \gamma_{s+1}$. Let furthermore $P$ and $\widehat{P}$ be the projections on the span of the top-$s$ left singular vectors for $A$ and $\widehat{A}$ respectively. Then we have,
$$
\|(I-\widehat{P}) P\| \leq \frac{2\|A - \hat A\|}{\gamma_s - \gamma_{s+1}},
$$ 
where the result also holds in Hilbert-Schmidt norm. Both bounds also hold when we replace the top-$s$ left singular vectors with the sets of top-$s$ right singular vectors.
\end{theorem}
\begin{proof}
In this proof, $\|\cdot\|$ denotes either the operator norm or the Hilbert-Schmidt norm. First note that $\|(I-\widehat{P}) P\| \leq 1$, therefore if $2\|A - \widehat{A}\| \geq \gamma_s - \gamma_{s+1}$, the bound is trivially obtained. Let us now consider $2\|A - \widehat{A}\| \leq \gamma_s - \gamma_{s+1}$. We start by assuming that $A$ and $\widehat{A}$ are rectangular $n \times m$ matrices. By Wedin's $\sin-\Theta$ Theorem, if $\gamma_s - \hat \gamma_{s+1} > 0$,
\begin{equation} \label{eq:wedin_original}
    \|(I-\widehat{P}) P\| \leq \frac{\|A - \widehat{A}\|}{\gamma_s - \hat \gamma_{s+1}}.
\end{equation}
By Weyl's inequality for singular values,
$$
\hat\gamma_{s+1} - \gamma_{s+1} \leq \|A - \widehat{A}\| \leq \frac{\gamma_s - \gamma_{s+1}}{2}.
$$
This implies, by the assumption $\gamma_s > \gamma_{s+1}$, that
$$
\gamma_s - \hat \gamma_{s+1} \geq \frac{\gamma_s - \gamma_{s+1}}{2} > 0.
$$
Therefore, combining Eq.~(\ref{eq:wedin_original}) and $2\|A - \widehat{A}\| \leq \gamma_s - \gamma_{s+1}$, we obtain 
\begin{equation}\label{eq:wedin_new}
\|(I-\widehat{P}) P\| \leq \frac{\|A - \widehat{A}\|}{\gamma_s - \hat \gamma_{s+1}} \leq \frac{2\|A - \hat A\|}{\gamma_s - \gamma_{s+1}}
\end{equation}
Let us now assume that $A$ and $\widehat{A}$ are compact operators.  Let $U$ and $\widehat{U}$ be the sets of first left $s+1$ eigenvectors of $A$ and $\widehat{A}$, respectively and let $\Pi_{U \cup \widehat{U}}$ be the projection on the union of their spans. Let $V$ and $\widehat{V}$ be the sets of first right $s+1$ eigenvectors of $A$ and $\widehat{A}$, respectively and let $\Pi_{V \cup \widehat{V}}$ be the projection on the union of their spans. We define the operators $A_0\doteq\Pi_{U \cup \widehat{U}} A \Pi_{V \cup \widehat{V}}$ and $\widehat{A}_0\doteq\Pi_{U \cup \widehat{U}} \widehat{A} \Pi_{V \cup \widehat{V}}$. By construction, the first $s+1$ singular values and left-right eigenvectors of $A_0$ and $\widehat{A}_0$ coincide with the first $s+1$ singular values and left-right eigenvectors of $A$ and $\widehat{A}$, respectively. By choosing some orthonormal basis of the finite-dimensional spaces $\operatorname{span}(U \cup \widehat{U})$ and $\operatorname{span}(V \cup \widehat{V})$ and expressing $A_0$ and $\widehat{A}_0$ in terms of matrices, we can apply the previous Eq.~(\ref{eq:wedin_new}) to conclude the proof.
\end{proof}

The extension of the original Wedin's $\sin-\Theta$ Theorem \cite{wedin1972perturbation} to Hilbert spaces is taken from the proof technique used in Theorem A.4.4 \cite{mollenhauer2021statistical}.

\begin{proof}[Proof of Proposition~\ref{prop:davis_kahan_meta}]
    We apply Theorem~\ref{th:wedin_generalized} to $C_N$ and $\hat{C}_{N,n,\la}$. As $C_N$ has rank $s$, $\gamma_{s+1}=0$.   
\end{proof}

\subsection{Proof of Theorem~\ref{th:pre_training}}\label{sec:prf_sin}
Before proving Theorem~\ref{th:pre_training}, we provide some intermediate results. 

\begin{lemma}\label{lma:lemma_1} For all $i \in [N]$, we have 
\begin{equation*} 
    \mathbb{E}[\hat f_{i,\lambda}] = \left(I - \lambda \mathbb{E}\left[\hat \Sigma_{i,\lambda}^{-1}\right]\right)f_i \qquad \text{ and } \qquad \|\mathbb{E}[\hat f_{i,\lambda}]\|_{\mathcal{H}} \leq \|f_i\|_{\mathcal{H}} 
\end{equation*}
\end{lemma}

\begin{proof}
    For all $i\in [N]$, we define $Y_i \doteq (y_{i,1},\dots,y_{i,n})^{\top} \in \R^n$ and $\epsilon_i \doteq [\epsilon_{i,1},\dots,\epsilon_{i,n}]^{\top} \in \R^n$ where $\epsilon_{i,j} \doteq y_{i,j} - f_i(x_{i,j})$, $j \in [n]$. For all $i\in [N]$, using $\epsilon_i = Y_i - \Phi_i^*(f_i)$, we get from Eq.~\eqref{eq:charac_f_la} that $\hat f_{i,\lambda}$ can be decomposed as 
    \begin{equation*} 
    \hat f_{i,\lambda} = \hat \Sigma_{i,\lambda}^{-1}\frac{\Phi_iY_i}{n} = \hat \Sigma_{i,\lambda}^{-1}\hat \Sigma_{i}f_i+  \hat \Sigma_{i,\lambda}^{-1}\frac{\Phi_i\epsilon_i}{n} ,
    \end{equation*}
    Since $\E[\epsilon_i \mid x_{i,1}, \ldots, x_{i,n}] = 0$, it yields 
    \begin{equation*} 
    \mathbb{E}[\hat f_{i,\lambda}] = \mathbb{E}[ \hat \Sigma_{i,\lambda}^{-1}\hat \Sigma_{i}f_i + \hat \Sigma_{i,\lambda}^{-1}\frac{\Phi_i}{n}\mathbb{E}[\epsilon_i \mid x_{i,1}, \ldots, x_{i,n}]] = \mathbb{E}[\hat \Sigma_{i,\lambda}^{-1}\hat \Sigma_{i}f_i] = \left(I - \lambda \mathbb{E} \hat \Sigma_{i,\lambda}^{-1}\right)f_i.
    \end{equation*}
    It gives us the following bound on $\|\mathbb{E}[\hat f_{i,\lambda}]\|_{\mathcal{H}}$
    \begin{equation*}  
        \|\mathbb{E}[\hat f_{i,\lambda}]\|_{\mathcal{H}} \leq \left\|I - \lambda \mathbb{E}\left[\hat \Sigma_{i,\lambda}^{-1} \right]\right\|\|f_i\|_{\mathcal{H}} \leq \mathbb{E}\left[\left\|I - \lambda \hat \Sigma_{i,\lambda}^{-1}\right\|\right]\|f_i\|_{\mathcal{H}} \leq \|f_i\|_{\mathcal{H}},
    \end{equation*}
    where in the last inequality we used the fact that the eigenvalues of $I - \lambda \hat \Sigma_{i,\lambda}^{-1}$ are in the interval $[0,1]$, hence its operator norm is bounded by $1$.
\end{proof}
For each source task $i \in [N]$, we introduce the regularized population regression function
\begin{equation*}
    f_{i,\lambda} \doteq \argmin_{f \in \mathcal{H}} \mathbb{E}_{\mu_i}\left[\left(Y - f(X) \right)^2\right] + \lambda \|f\|^2_{\mathcal{H}}.
\end{equation*}
It admits the closed-form expression
\begin{equation}\label{eq:pop_estim}
    f_{i,\lambda} = \Sigma_{i,\la}^{-1}\Sigma_{i} f_i = \left(I - \lambda \Sigma_{i,\lambda}^{-1}\right)f_i.
\end{equation}
Therefore, we have the following bound for its $\cH-$norm 
\begin{equation} \label{eq:ridge_estimator_bound_1_r}
    \|f_{i,\lambda}\|_{\cH} = \left\|\left(I - \lambda \Sigma_{i,\lambda}^{-1}\right)f_i\right\|_{\cH} \leq \|I - \lambda \Sigma_{i,\lambda}^{-1}\|\|f_i\|_{\cH} \leq  \|f_i\|_{\cH}.
\end{equation}
Furthermore, we have
$$
    f_{i,\lambda} - \mathbb{E}[\hat f_{i,\lambda}] =  \lambda\left(\mathbb{E} \hat \Sigma_{i,\lambda}^{-1} - \Sigma_{i,\lambda}^{-1}\right)f_i.
$$
This quantity is the statistical bias of the estimator $\hat f_{i,\lambda}$. To prove Theorem~\ref{th:pre_training}, we use the following decomposition,
\begin{equation}
    \|\hat C_{N,n,\lambda} - C_N\|_{HS} \leq \underbrace{\|\hat C_{N,n,\lambda} -  \bar C_{N,n,\lambda}\|_{HS}}_{\doteq\text{Variance}} + \underbrace{\| \bar C_{N,n,\lambda} - C_N\|_{HS}}_{\doteq\text{Bias}},\nonumber
\end{equation}
where 
$$
\bar C_{N,n,\la} \doteq \frac{1}{N} \sum_{i=1}^N \E (\hat f_{i,\la}) \otimes \E (\hat f_{i,\la}),
$$
and $C_N$, $\hat C_{N,n,\lambda}$ are defined in Eqs.~\eqref{eq:task_cov} and \eqref{eq:task_cov_approx} respectively.

\begin{theorem}[Bounds on the variance term]\label{theo:var_pf} Suppose Assumption~\ref{asst:src} and Assumption~\ref{asst:mom} hold. Define $\operatorname{\mathbf{Var}}_{\lambda}^{(N,n)}\doteq \|\hat C_{N,n,\lambda} -  \bar C_{N,n,\lambda}\|_{HS} $. For $\la \in (0,1],\tau,\delta \geq \log(2)$ and $N,n \geq 1$, with probability greater than $1-2e^{-\tau}-4Ne^{-\delta}$,
\begin{equation}\label{eq:bound_variance_2}
   \operatorname{\mathbf{Var}}_{\lambda}^{(N,n)}  \leq c_1\left\{\left(\frac{\delta^2}{n \lambda^2} + \frac{\delta}{\sqrt{n} \lambda}  + \frac{e^{-\delta}}{\lambda}\right)\sqrt{\frac{\tau}{N}} + \frac{e^{-\delta}}{\lambda} \right\},
\end{equation}
with $c_1$ a constant depending on $Y_{\infty}$, $\max_{i \in [N]}\left\| f_i \right\|_{\cH}$, $\kappa$ and $R$.

Alternatively, suppose Assumptions~\ref{asst:evd}, \ref{asst:emb}, \ref{asst:src} and \ref{asst:mom} hold. For $0< \la < 1 \wedge \min_{i \in [N]}\left\|\Sigma_i\right\|$, $\delta \geq 1$, $\tau \geq \log(2)$, $N \geq 1$ and  $n \geq c_0\delta\left(1 + p\log (\la^{-1})\right) \lambda^{-\alpha}$, with probability greater than $1-2e^{-\tau}-8Ne^{-\delta}$,
\begin{equation}\label{eq:bound_variance_1}
   \operatorname{\mathbf{Var}}_{\lambda}^{(N,n)}  \leq c\left\{\left(\left(\frac{\delta}{\sqrt{n} \lambda^{\frac{1+p}{2}}}\sqrt{1+ \frac{1}{n \lambda^{\alpha-p}}}\right)^2 +  \frac{\delta}{\sqrt{n} \lambda^{\frac{1+p}{2}}}\sqrt{1+ \frac{1}{n \lambda^{\alpha-p}}}  + \frac{e^{-\delta}}{\lambda}\right)\sqrt{\frac{\tau}{N}} + \frac{e^{-\delta}}{\lambda}\right\}
\end{equation}
with $c_0$ a constant depending on $k_{\alpha,\infty}, D$ and $c$ a constant depending on $Y_{\infty}, k_{\alpha,\infty}, D$ and $R$. 
\end{theorem}
We use the second variance bound to prove Theorem~\ref{th:pre_training}. The first bound is used to prove Remark~\ref{rem:rem_exp}.
\begin{proof}
For $i \in [N]$, we let $\xi_i \doteq \hat f_{i,\lambda}' \otimes \hat{f}_{i,\lambda} - \E (\hat f_{i,\lambda}) \otimes \E (\hat{f}_{i,\lambda})$ and $\eta_i \doteq \hat f'_{i,\lambda} \otimes \hat{f}_{i,\lambda} - f_{i,\la} \otimes f_{i, \la}$ such that $\xi_i = \eta_i - \E[\eta_i]$. We start with the following decomposition, for $i \in [N]$
\begin{align*}
\eta_i &= \hat f_{i,\lambda}' \otimes \hat{f}_{i,\lambda} - f_{i,\la} \otimes f_{i, \la}\\
&= \hat f_{i,\lambda}' \otimes (\hat f_{i,\lambda} - f_{i,\lambda}) + (\hat f_{i,\lambda}' - f_{i,\lambda}) \otimes f_{i,\lambda}\\
&= (\hat f_{i,\lambda}' + f_{i,\la} - f_{i,\la}) \otimes (\hat f_{i,\lambda} - f_{i,\lambda}) + (\hat f_{i,\lambda}' - f_{i,\lambda}) \otimes f_{i,\lambda}\\
&= (\hat f_{i,\lambda}'- f_{i,\lambda}) \otimes (\hat f_{i,\lambda} - f_{i,\lambda}) + f_{i,\lambda} \otimes (\hat f_{i,\lambda} - f_{i,\lambda}) + (\hat f_{i,\lambda}' - f_{i,\lambda}) \otimes f_{i,\lambda}.
\end{align*}
We now use Eq.~\eqref{eq:ridge_estimator_bound_1_r},
\begin{align*}
\left\|\eta_i \right\|_{HS} &\leq \left\|\hat f_{i,\lambda}'- f_{i,\lambda}\right\|_{\cH}\left\|\hat f_{i,\lambda}- f_{i,\lambda}\right\|_{\cH} + \left\|f_{i,\lambda}\right\|_{\cH}\left(\left\|\hat f_{i,\lambda}- f_{i,\lambda}\right\|_{\cH}+\left\|\hat f_{i,\lambda}'- f_{i,\lambda}\right\|_{\cH}\right) \\
&\leq \left\|\hat f_{i,\lambda}'- f_{i,\lambda}\right\|_{\cH}\left\|\hat f_{i,\lambda}- f_{i,\lambda}\right\|_{\cH} + \left\|f_{i}\right\|_{\cH}\left(\left\|\hat f_{i,\lambda}- f_{i,\lambda}\right\|_{\cH}+\left\|\hat f_{i,\lambda}'- f_{i,\lambda}\right\|_{\cH}\right).
\end{align*}
In the following we assume that we have access to a function $g(n,\la,\delta)$ such that for all $\delta \geq 0$ and $i \in [N]$, with probability at least $1-Je^{-\delta}$
\begin{equation} \label{eq:general_g}
    \left\|\hat f_{i,\lambda} - f_{i,\lambda}\right\|_{\cH} \leq g(n,\lambda,\delta),
\end{equation}
for some constant $J \geq 1$. We will use either Theorem~\ref{prop:variance_regression_smale}: for $\la>0,\delta\geq \log(2), n \geq 1$, with probability at least $1-2e^{-\delta}$
$$
    g(n,\la,\delta) = \frac{6 \kappa Y_{\infty}\delta}{\sqrt{n} \lambda},
$$
or Theorem~\ref{prop:variance_regression}: for $ \delta \geq 1$, $\lambda < 1 \wedge \min_{i \in [N]}\left\|\Sigma_i\right\|$, and $n \geq c_0\delta\left(1 + p\log (\la^{-1})\right) \lambda^{-\alpha}$, with probability not less than $1-4e^{-\delta}$
\begin{equation*}
    g(n,\la,\delta) = \frac{c\delta}{\sqrt{n} \lambda^{\frac{1+p}{2}}}\sqrt{1+ \frac{1}{n \lambda^{\alpha-p}}}, 
\end{equation*}
with $c_0$ a constant depending on $k_{\alpha,\infty}, D$ and $c$ a constant depending on $Y_{\infty}, k_{\alpha,\infty}, D$ and $R$. We fix $g$ a function satisfying Eq.~\eqref{eq:general_g} and define the events
$$
E_{i,n,\lambda,\delta} \doteq \left\{\left\|\hat f_{i,\lambda} - f_{i,\lambda}\right\|_{\cH} \vee \left\|\hat f_{i,\lambda}' - f_{i,\lambda}\right\|_{\cH} \leq g(n,\lambda,\delta) \right\}, i \in [N], \qquad E_{N,n,\lambda,\delta} \doteq \bigcap_{i=1}^N E_{i,n,\lambda,\delta}.
$$
By independence of the $\hat f_{i,\lambda}$ and $\hat f_{i,\lambda}'$, we have for all $i \in [N]$
$$
\begin{aligned}
    \P\left(E_{i,n,\lambda,\delta}\right) &\geq (1-Je^{-\delta})^{2} \geq 1-2Je^{-\delta} \\
    \P\left(E_{N,n,\lambda,\delta}\right) &\geq (1-Je^{-\delta})^{2N} \geq 1-2JNe^{-\delta}, 
\end{aligned}
$$
where we used Bernoulli's inequality. We then have 
\begin{align}
\mathbbm{1}_{E_{N,n,\lambda,\delta}}\|\eta_i\|_{HS} &\leq  g(n,\lambda,\delta)^2 + 2\|f_i\|_{\cH}g(n,\lambda,\delta) \nonumber
\end{align}
For any $\epsilon >0$,
\begin{align*}
\P \left( \left\| \frac{1}{N}\sum_{i=1}^N \xi_i\right\|_{HS} \geq \epsilon\right)
&= \P \left(\left\{ \left\| \frac{1}{N}\sum_{i=1}^N \xi_i\right\|_{HS} \geq \epsilon \right\} \cap  E_{N,n,\lambda,\delta}\right) \\ &+ \underbrace{\P \left( \left\| \frac{1}{N}\sum_{i=1}^N \xi_i\right\|_{HS} \geq \epsilon \mid E^c_{N,n,\lambda,\delta}\right)}_{\leq 1}\underbrace{\P(E^c_{N,n,\lambda,\delta})}_{\leq 2JNe^{-\delta}}\\
&\leq \P \left(\left\{ \left\| \frac{1}{N}\sum_{i=1}^N \xi_i\right\|_{HS} \geq \epsilon \right\} \cap  E_{N,n,\lambda,\delta}\right) + 2JNe^{-\delta}.
\end{align*}
For each $i \in [N]$, $\E[\xi_i]=0$, therefore by Proposition~\ref{prop:xi_as},
$$
\left\|\E\left[\xi_i\mathbbm{1}_{E_{i,n,\lambda,\delta}}\right]\right\| = \left\|\E\left[\xi_i\mathbbm{1}_{E_{i,n,\lambda,\delta}}\right] - \E[\xi_i]\right\| \leq \E\left[\|\xi_i\|\mathbbm{1}_{i_{N,n,\lambda,\delta}^c}\right] \leq \frac{c_1}{\la}\P(E_{i,n,\lambda,\delta}^c) \leq \frac{2c_1J}{\la}e^{-\delta},
$$
with $c_1 \doteq Y_{\infty}^2+\max_{i \in [N]}\left\| f_i \right\|_{\cH}^2$. For all $i \in [N]$, we define $\zeta_i \doteq \xi_i\mathbbm{1}_{E_{i,n,\lambda,\delta}} - \E\left[\xi_i\mathbbm{1}_{E_{i,n,\lambda,\delta}}\right]$, $i \in [N]$. We have
\begin{align}
\P \left(\left\{ \left\| \frac{1}{N}\sum_{i=1}^N \xi_i\right\|_{HS} \geq \epsilon \right\} \cap  E_{N,n,\lambda,\delta}\right) &\leq \P \left(\left\| \frac{1}{N}\sum_{i=1}^N \xi_i\mathbbm{1}_{E_{i,n,\lambda,\delta}}\right\|_{HS} \geq \epsilon \right) \nonumber \\ 
&= \P \left(\left\| \frac{1}{N}\sum_{i=1}^N \zeta_i + \E\left[\xi_i\mathbbm{1}_{E_{i,n,\lambda,\delta}}\right]\right\|_{HS} \geq \epsilon \right) \nonumber \\ & \leq \P \left(\left\| \frac{1}{N}\sum_{i=1}^N \zeta_i \right\|_{HS}+\frac{1}{N}\sum_{i=1}^N\left\| \E\left[\xi_i\mathbbm{1}_{E_{i,n,\lambda,\delta}}\right]\right\|_{HS} \geq \epsilon \right)\nonumber\\
&\leq \P \left(\left\| \frac{1}{N}\sum_{i=1}^N \zeta_i\right\|_{HS} + 2c_1 \la^{-1} Je^{-\delta} \geq \epsilon \right). \nonumber
\end{align}
By Proposition~\ref{prop:xi_as} again,
$$
\begin{aligned}
    \|\xi_i \|_{HS}\mathbbm{1}_{E_{i,n,\lambda,\delta}} &\leq \mathbbm{1}_{E_{i,n,\lambda,\delta}}\|\eta_i\| +  \|\E[\eta_i]\|_{HS} \\ &\leq  g(n,\lambda,\delta)^2 + 2\|f_i\|_{\cH}g(n,\lambda,\delta) + \E\left[\|\eta_i \|_{HS}\right] \\ &=  g(n,\lambda,\delta)^2 + 2\|f_i\|_{\cH}g(n,\lambda,\delta) + \E\left[\|\eta_i  \|_{HS}\mathbbm{1}_{E_{i,n,\lambda,\delta}}\right] + \E\left[\|\eta_i \|_{HS}\mathbbm{1}_{E_{i,n,\lambda,\delta}^c}\right] \\ &\leq 2\left( g(n,\lambda,\delta)^2 + 2\|f_i\|_{\cH}g(n,\lambda,\delta) \right) + \frac{c_1}{\lambda}\P\left(E_{i,n,\lambda,\delta}^c\right) \\ &\leq 2\left( g(n,\lambda,\delta)^2 + 2\kappa Rg(n,\lambda,\delta) \right) + \frac{2c_1Je^{-\delta}}{\lambda} \doteq V(n, \la, \delta),
\end{aligned}
$$
where we used that by Assumption~\ref{asst:src}: for $i \in [N]$
$$
\|f_i\|_{\cH} = \|\Sigma_i^{r}\Sigma_i^{-r}f_i\|_{\cH} \leq  \|\Sigma_i\|^{r}\|\Sigma_i^{-r}f_i\|_{\cH} \leq R\kappa.
$$
Hence $\|\zeta_i\|_{HS} \leq 2V(n, \la, \delta)$. We now use Hoeffding's inequality (Theorem~\ref{th:hoeffding_hilbert}) on $H=HS$ for the centered and bounded random variables $\{\zeta_i\}_{i \in [N]}$. As long as $\epsilon \geq 2c_1\la^{-1}Je^{-\delta}$,
\begin{align}
\P \left(\left\| \frac{1}{N}\sum_{i=1}^N \zeta_i\right\|_{HS} + 2c_1 \la^{-1} Je^{-\delta} \geq \epsilon \right)  
&\leq& 2\exp\left(-\frac{N(\epsilon- 2c_1 \la^{-1} Je^{-\delta})^2}{8V^2(n,\lambda,\delta)}\right). \nonumber
\end{align}
Therefore, combining the results, we obtain
\begin{align*}
\P \left( \left\| \frac{1}{N}\sum_{i=1}^N \xi_i\right\|_{HS} \geq \epsilon \right) &\leq 2\exp\left(-\frac{N(\epsilon-2c_1 \la^{-1} Je^{-\delta})^2}{8V^2(n,\lambda,\delta)}\right) + 2JNe^{-\delta} \doteq \tau.
\end{align*}
Solving for $\epsilon$, we have for all $\tau  \in (2JNe^{-\delta},1)$,
$$
\epsilon = V(n,\lambda,\delta)\sqrt{\frac{8}{N}\log\left(\frac{2}{\tau - 2JNe^{-\delta}}\right)} +2c_1 \la^{-1} Je^{-\delta}. 
$$ 
Finally, we obtain that for all $\tau \in (2JNe^{-\delta},1)$ with probability greater than $1-\tau$,
\[\left\| \frac{1}{N}\sum_{i=1}^N \xi_i\right\|_{HS}  \leq V(n,\lambda,\delta) \sqrt{\frac{8}{N}\log\left(\frac{2}{\tau - 2JNe^{-\delta}}\right)} + 2c_1 \la^{-1} Je^{-\delta}.\]
Alternatively we can write, for all $\tau \geq \log(2)$ with probability greater than $1-2e^{-\tau}-2JNe^{-\delta}$,
$$
\begin{aligned}
    \left\| \frac{1}{N}\sum_{i=1}^N \xi_i\right\|_{HS}  &\leq \left( 2g(n,\lambda,\delta)^2 + 4\kappa Rg(n,\lambda,\delta)  + \frac{2c_1Je^{-\delta}}{\lambda}\right)\sqrt{\frac{8\tau}{N}} + \frac{2c_1Je^{-\delta}}{\lambda} ,
\end{aligned}
$$
$\delta$ is a free parameter that we will adjust as a function of $N,n$ such that $Ne^{-\delta}$ and $e^{-\delta}\la^{-1}$ converge to $0$ with $N \to +\infty$ or $n \to +\infty$.
\end{proof}

\begin{theorem}[Bounds on the bias term]\label{theo:bias_pf}
Suppose Assumptions~\ref{asst:evd}, \ref{asst:emb} and \ref{asst:src} hold with $0 < p \leq \alpha \leq 1$ and $r \in [0,1]$. For any $\delta \geq 1$ and $\la \in (0,1]$ if $n \geq c_1\delta \lambda^{-p-1}$ then
\begin{equation}
    \|\bar C_{N,n,\lambda} - C_N\|_{HS} \leq c_2\lambda^{r}\left(1 + \frac{e^{-\delta}}{\lambda} \right). \label{eq:bias_1}
\end{equation}
where $c_1$ is a constant depending only on $D$, $\kappa$, and $k_{\alpha,\infty}$ and $c_2$ is a constant depending only on $R$, $\kappa$. 

Furthermore, for any $\delta \geq 1$ and $0<\lambda < 1 \wedge \min_{i \in [N]}\left\|\Sigma_i\right\|$, if $n \geq c_3\delta\left(1 + p\log (\la^{-1})\right) \lambda^{-\alpha}$ then  
\begin{equation} \label{eq:bias_bound_2}
    \|\bar C_{N,n,\lambda} - C_N\|_{HS} \leq c_4\left(e^{-\delta} + \la^{r \wedge 1/2} \right).
\end{equation}
where $c_3$ only depends on $k_{\al,\infty},D$ and $c_4$ only depends on $R, \kappa$.
\end{theorem}

\begin{proof}
\begin{equation*}
\begin{aligned}
    \|\bar C_{N,n,\lambda} - C_N\|_{HS} &=  \left\|\frac{1}{N}\sum_{i=1}^N\left(f_i \otimes f_i - \mathbb{E}[\hat f_{i,\lambda}] \otimes \mathbb{E}[\hat f_{i,\lambda}]\right) \right\|_{HS} \\
    &\leq \frac{1}{N}\sum_{i=1}^N\left\|f_i  \otimes (f_i - \mathbb{E}[\hat f_{i,\lambda}]) -  (\mathbb{E}[\hat f_{i,\lambda}]-f_i) \otimes \mathbb{E}[\hat f_{i,\lambda}]) \right\|_{HS} \\
    &\leq \frac{2 \max\{\max_{i=1}^N \|f_{i}\|_{\mathcal{H}}, \max_{i=1}^N \|\mathbb{E}[\hat f_{i,\lambda}]\|_{\mathcal{H}} \}}{N}\sum_{i=1}^N \left\|f_i - \mathbb{E}[\hat f_{i,{\lambda}}] \right\|_{\mathcal{H}},\\
    & \leq \frac{2 R\kappa}{N}\sum_{i=1}^N \left\|f_i - \mathbb{E}[\hat f_{i,{\lambda}}] \right\|_{\mathcal{H}},\\
\end{aligned}
\end{equation*}
where we used Lemma~\ref{lma:lemma_1} and Assumption~\ref{asst:src}: for $i \in [N]$
$$
\|f_i\|_{\cH} = \|\Sigma_i^{r}\Sigma_i^{-r}f_i\|_{\cH} \leq  \|\Sigma_i\|^{r}\|\Sigma_i^{-r}f_i\|_{\cH} \leq R\kappa.
$$
For the first bound, by Proposition~\ref{prop:bias_exp}, we have for $\la \in (0,1]$, $\delta \geq 1$ and $n \geq c_1\delta \lambda^{-1-p}$, 
\[
\|f_{i} -\E(\hat{f}_{i,\la})\|_{\cH} \leq c_5\lambda^{r}\left(1 + \frac{e^{-\delta}}{\lambda} \right),
\]
where $c_1$ is a constant depending only on $D$, $\kappa$, and $k_{\alpha,\infty}$ and $c_5$ is a constant depending only on $R$, $\kappa$, which proves the second bound. For the second bound, by Proposition~\ref{prop:bias_exp_r_0_5}, we have for $0<\lambda < 1 \wedge \min_{i \in [N]}\left\|\Sigma_i\right\|$, $\delta \geq 1$ and $n \geq c_6\delta\left(1 + p\log (\la^{-1})\right) \lambda^{-\alpha}$, 
\begin{equation*} 
    \|f_{i} -\E(\hat{f}_{i,\la})\|_{\cH} \leq c_7 \left(e^{-\delta} + \la^{r \wedge 1/2}\right),
\end{equation*}
where $c_6$ only depends on $k_{\al,\infty},D$ and $c_7$ only depends on $R, \kappa$.
\end{proof}

\begin{proof}[Proof of Theorem~\ref{th:pre_training}] \textbf{Bound in Eq.~(\ref{eq:sin_risk_sharp}).} We first notice that the bias bounds in Eq.~\eqref{eq:bias_1} and Eq.~\eqref{eq:bias_bound_2}  can be combined as follows: for any $\delta \geq 1$, $0<\lambda < 1 \wedge \min_{i \in [N]}\left\|\Sigma_i\right\|$ and $n \geq c_1\delta\left(1 + p\log (\la^{-1})\right) \lambda^{-\alpha}$ if $r \leq 1/2$ or $n \geq c_2\delta \lambda^{-p-1}$ if $r \in (1/2,1]$,
\begin{equation*}
    \|\bar C_{N,n,\lambda} - C_N\|_{HS} \leq C_0\left(\lambda^{r} + \frac{e^{-\delta}}{\lambda} \right),
\end{equation*}
where we used $\la \leq 1$. $c_1$ is a constant depending only on $D$, $k_{\alpha,\infty}$, $c_2$ is a constant depending only on $D$, $\kappa$, and $k_{\alpha,\infty}$ and $C_0$ is a constant depending only on $R$, $\kappa$. We now combine this bias bound with Eq.~\eqref{eq:bound_variance_1} for the variance. Note that both bounds have a free parameter $\delta$ that we take as the same value for each bound. Since $0<\la \leq 1$ by assuming $n$ large enough so that $\frac{\delta}{\sqrt{n} \lambda^{\frac{1+p}{2}}}\sqrt{1+ \frac{1}{n \lambda^{\alpha-p}}} \leq 1$, we obtain that for all $0< \la < 1 \wedge \min_{i \in [N]}\left\|\Sigma_i\right\|$, $\tau \geq \log(2)$, $\delta \geq 1$, $N \geq \tau$ and $n \geq c_1\delta\left(1 + p\log (\la^{-1})\right) \lambda^{-\alpha}$ if $r \leq 1/2$ or $n \geq \delta\max\{c_2\lambda^{-p-1},c_1\left(1 + p\log (\la^{-1})\right) \lambda^{-\alpha}\}$ if $r \in (1/2,1]$,  with probability greater than $1-2e^{-\tau}-8Ne^{-\delta}$,
$$
\|\hat C_{N,n,\lambda} - C_N\|_{HS} \leq C_1\left(\frac{\delta\sqrt{\tau}}{\sqrt{nN}\la^{\frac{1+p}{2}}}\sqrt{1 + \frac{1}{n\la^{\al-p}}} + \frac{e^{-\delta}}{\lambda} +  \lambda^{r} \right),
$$
with $C_1$ a constant depending on $Y_{\infty}, k_{\alpha,\infty}, D, \kappa$ and $R$. As $\delta$ is a free parameter, we pick $\delta = 12\log(Nn)$ with $N,n$ large enough such that $\delta \geq 1$. 
\\
We obtain that for $n \geq 12c_0\log(Nn)\left(1 + p\log (\la^{-1})\right) \lambda^{-\alpha}$ if $r \leq 1/2$ or $$
n \geq 12\log(Nn)\max\{c_2\lambda^{-p-1},c_1\left(1 + p\log (\la^{-1})\right) \lambda^{-\alpha}\},
$$ 
if $r \in (1/2,1]$, with probability greater than $1-2e^{-\tau}-o((nN)^{-10})$:
    \begin{equation} \label{eq:bound_theorem2_2}
    \|\hat{C}_{N,n,\la}-C_N\|_{HS}  \leq C_1\left(\frac{12\log (nN)\sqrt{\tau}}{\sqrt{nN} \lambda^{\frac{1}{2}+\frac{p}{2}}}\sqrt{1+ \frac{1}{n \lambda^{\alpha-p}}} + \frac{1}{\la (nN)^{12}} +   \lambda^{r}\right).
\end{equation}
When we optimise for $\lambda$ in Corollary~\ref{cor:learning_rate}, we notice that the term $\frac{1}{\la (nN)^{12}}$ is always of lower order, therefore we do not include it in the presentation of Theorem~\ref{th:pre_training}. Finally, since $\la \leq 1$ and we always have $p+1 \geq \alpha$, when $r \in (1/2,1]$, we can simplify the constraint on $n$ as $n \geq c_3\log(Nn)\left(1 + p\log (\la^{-1})\right) \lambda^{-p-1}$, with $c_3$ a constant depending on $D$, $\kappa$, and $k_{\alpha,\infty}$.
\end{proof}

\begin{rem}[Proof of Example~\ref{ex:finite}]\label{rem:rem_pf} 
In the finite dimensional case, we use the same steps as the previous proof with $g(n,\lambda, \delta) = c\delta\sqrt{\frac{k}{n}}$ from Eq.~\eqref{eq:fin_rank} in Theorem~\ref{prop:variance_regression}. Furthermore, for the bias, we let $r=1$ since Assumption~\ref{asst:src} is satisfied for any value of $r$ when the RKHS is finite dimensional. This leads to the following bound. For $n \gtrsim k\log(nN)$, with probability greater than $1-2e^{-\tau}-o((nN)^{-10})$:
\begin{equation} 
    \|\hat{C}_{N,n,\la}-C_N\|_{HS}  \lesssim \frac{\log (nN)\sqrt{k\tau}}{\sqrt{nN}} + \frac{1}{\la (nN)^{12}} +\lambda^{\frac{1}{2}} + \frac{1}{(nN)^{12}},
\end{equation}
We obtain Eq.~\eqref{eq:finite_dim}, by plugging $\lambda = \frac{\log^2(nN)}{nN}$.
\end{rem}

\subsection{Proof of Corollary~\ref{cor:learning_rate}}\label{sec:conv_rate}
In the following, we ignore constants as only the orders of $n$ and $N$ matter for the proof of the corollary. \\
By Theorem~\ref{th:pre_training}, under the assumption that $n \geq \log(Nn)\log (\la^{-p}) \lambda^{-\alpha}$ for $r \leq 1/2$ or $n \geq \log(Nn)\log (\la^{-p}) \lambda^{-p-1}$ for $r \in (1/2,1]$, with high probability, $\|\hat{C}_{N,n,\la}-C_N\|_{HS}$ is bounded by a term of the order (see Eq.~\eqref{eq:sin_risk_sharp} and Eq.~\eqref{eq:bound_theorem2_2}),
\begin{equation*}
    \frac{\log (nN)}{\sqrt{nN} \lambda^{\frac{1+p}{2}}} + \frac{\log (nN)}{n\sqrt{N} \lambda^{\frac{1+\al}{2}}} + \frac{1}{ \la N^{12}n^{12}} +   \lambda^{r}.
\end{equation*}
Let $a>0$ be defined as $a\doteq\log(N)/\log(n)$, then $N = n^a$. Plugging $N=n^a$, we get the following upper bounds on $\|\hat{C}_{N,n,\la}-C_N\|_{HS}$:
\begin{equation}\label{eq:corr_1_bound_2}
    r_1(\la, n) \doteq \lambda^{r} +  \underbrace{\frac{\log (n^{1+a})}{n^{\frac{1+a}{2}} \lambda^{\frac{1+p}{2}}}}_{\boxed{1}} + \underbrace{\frac{\log (n^{1+a})}{n^{1+\frac{a}{2}}\lambda^{\frac{1+\al}{2}}}}_{\boxed{2}} + \underbrace{  \frac{1}{\la n^{12(1+a)}}}_{\boxed{3}},
\end{equation}

\textbf{Case A:} To optimize this bound, we start by matching $\lambda^r$ with \boxed{1}. 
$$
\la^{r} =\frac{\log (n^{1+a})}{\lambda^{\frac{1+p}{2}} n^{(a+1)/2}} \iff \la = n^{-\frac{a+1}{2r+1+p}}\log (n^{1+a})^{\frac{2}{2r+1+p}}.
$$ 
We need to make sure that the matched term $\boxed{n^{-\frac{r(a+1)}{2r+1+p}}\log (n^{1+a})^{\frac{2r}{2r+1+p}}}$ is the slowest term.
    \begin{itemize}
        \item[a)] $\boxed{1} \geq \boxed{2} \iff \boxed{n^{-\frac{r(a+1)}{2r+1+p}}\log (n^{1+a})^{\frac{2r}{2r+1+p}}} \geq \log (n^{1+a})^{1-\frac{1+\al}{2r+1+p}}n^{-1-a/2+\frac{(a+1)(1+\al)}{2(2r+1+p)}}$ which is satisfied if $a \leq \frac{2r+p+1}{\al-p}-1$. 
        \item[b)] $\boxed{1} \geq \boxed{3}$ is always satisfied.
    \end{itemize}
    We then need to check the constraint on $n$ for both $r \leq 1/2$ ($n \geq A_{\la}$) and $r \in (1/2,1]$ ($n \geq B_{\la}$). Plugging back $a=\log(N)/\log(n)$, we get $\la = (nN)^{-\frac{1}{2r+1+p}}\log (nN)^{\frac{2}{2r+1+p}}$. Let us start with $r \leq 1/2$. Recall that $A_{\la} = \log(nN)\log (\la^{-p})\lambda^{-\alpha}$ (where we ignore constant as only the orders of $n$ and $N$ matter here). Hence, plugging the value of $\la$,
    $$
    n \geq A_{\la} \iff n \geq \log(nN)^{1-\frac{2\al}{2r+1+p}}\log \left((nN)^{\frac{p}{2r+1+p}}\log (nN)^{-\frac{2p}{2r+1+p}}\right)(nN)^{\frac{\al}{2r+1+p}}.
    $$
    To satisfy this condition, it is sufficient that $N \leq n^{\frac{2r+p+1}{\al}-1}$, i.e. $a \leq \frac{2r+p+1}{\al}-1$. Notice that $a \leq \frac{2r+p+1}{\al}-1 \leq \frac{2r+p+1}{\al-p}-1$, therefore if $a \leq \frac{2r+p+1}{\al}-1$ we have $\boxed{1} \geq \boxed{2}$ and under this condition, the obtained upper bound is of the order
    \begin{equation} \label{eq:bound_A1}
        (nN)^{-\frac{r}{2r+1+p}}\log (nN)^{\frac{2r}{2r+1+p}}
    \end{equation}
    Let us now move to the constraint $n \geq B_{\lambda}$ for $r \in (1/2,1]$. Recall that $$B_{\la} = \log(Nn)\log (\la^{-p}) \lambda^{-p-1}.$$ To satisfy $n \geq B_{\lambda}$, it is sufficient that $a \leq \frac{2r+p+1}{p+1}-1,$ i.e. $N \leq n^{\frac{2r+p+1}{p+1}-1}$. Notice that $\frac{2r+p+1}{p + 1} - 1 \leq \frac{2r+p+1}{\al-p}-1$, therefore if $a \leq \frac{2r+p+1}{p + 1} - 1$ we have $\boxed{1} \geq \boxed{2}$ and under this condition, the obtained upper bound is the same as in Eq.~(\ref{eq:bound_A1}) with $r \in (1/2,1]$. It concludes the proof of Eq.~\eqref{eq:corr_eq1}.

    \textbf{Case B:} In that regime, we further increase $N$ beyond the constraints in case A. As a result, the variance become negligible and we only need to minimize the bias. 
    
    $\bullet$ B.1. $r \in (0,1/2]$. We focus on bounding the risk with Eq.~\eqref{eq:sin_risk_sharp} under the constraint $n \geq A_{\lambda}$. We choose the minimum $\lambda$ such that $n \geq A_{\lambda}$ is satisfied. This gives us $\lambda = (\log^{\omega}(nN)/n)^{1/\alpha}$ for $\omega >2$. This choice of $\lambda$ leads to the final rate of 
    \begin{equation}
    \log(nN)^{\frac{\omega r}{\alpha}}\cdot n^{-\frac{r}{\alpha}}\nonumber
    \end{equation}
    
    $\bullet$ B.2. $r\in (1/2,1]$. We now choose the minimum $\lambda$ such that $n \geq B_{\lambda}$ is satisfied. This gives us $\lambda = \left(\frac{\log^{\omega}(nN)}{n}\right)^{\frac{1}{p+1}}$ for $\omega >2$. This choice of $\lambda$ leads to the final rate of 
    \begin{equation}
    \log(nN)^{\frac{\omega r}{p+1}}\cdot n^{-\frac{r}{p+1}}\nonumber
    \end{equation}

\subsection{Proofs of Section~\ref{sec:algo}}\label{sec:algo_proof}

\begin{definition}
Let $A,B \in \R^{n \times n}$ be two real symmetric matrices, the generalized eigenvalue problem solves for $(v,\gamma) \in \R^n \times \R$,
$$
(A - \gamma B)v = 0.
$$
A solution $(v,\gamma)$ is called a generalized eigenpair where $v$ is called a generalized eigenvector and $\gamma$ a generalized eigenvalue. Note that if $B=I_n$ we retrieve  the standard eigenvalue problem.
\end{definition}

For a comprehensive treatment of the generalized eigenvalue problem see \cite{parlett1998symmetric}. Before proving Proposition~\ref{prop:gen_eigen}, we need the following lemma.

\begin{lemma}
Let $(\hat U,\hat V)$ be the top$-s$ left and right singular vectors of $\hat C_{N,n,\la}$. $(\hat U,\hat V)$ is solution of 
\begin{equation} \label{eq:svd_prob}
\begin{aligned}
\max_{U,V:\R^s \to \cH} \quad & \operatorname{Tr}(U^*\hat C_{N,n,\la} V)\\
\textrm{s.t.} \quad & U^*U=I_s\\
  & V^*V=I_s    \\
\end{aligned}
\end{equation}
\end{lemma}

\begin{proof}
    For all $U,V:\R^s \to \cH$ such that $U^*U=V^*V=I_s$, let us write $U = [u_1, \ldots, u_s]$, $V = [v_1, \ldots, v_s]$. Plugging the SVD $\hat C_{N,n,\la}=\sum_{i=1}^{N} \hat \gamma_i \hat u_i \otimes \hat v_i$ in the objective, we have
    \begin{equation*} 
    \operatorname{Tr}(U^*\hat C_{N,n,\la} V) = \sum_{i=1}^s \sum_{l=1}^N \hat \gamma_l \langle u_i,\hat u_l \rangle_{\cH} \langle v_i,\hat v_l \rangle_{\cH}.
    \end{equation*}
    In that form, the objective is separable in the $s$ variables $\{(u_i,v_i)\}_{i=1}^s$. For $i=1$, we have 
    $$
    \sum_{l=1}^N \hat \gamma_l \langle u_1,\hat u_l \rangle_{\cH} \langle v_1,\hat v_l \rangle_{\cH}\leq \hat \gamma_1 \left(\sum_{l=1}^N \langle u_1,\hat u_l \rangle_{\cH}^2\right)^{1/2}\left(\sum_{l=1}^N \langle v_1,\hat v_l \rangle_{\cH}^2\right)^{1/2} \leq \hat \gamma_1,
    $$
    and the upper bound is achieved for $u_1 = \hat u_1$ and $v_1 = \hat v_1$. For $i=2$, incorporating the constraint $\langle u_2, u_1 \rangle_{\cH}= \langle v_2, v_1 \rangle_{\cH}=0$ and plugging $u_1 = \hat u_1$ and $v_1 = \hat v_1$, we have 
    $$
    \sum_{l=1}^N \hat \gamma_l \langle u_2,\hat u_l \rangle_{\cH} \langle v_2,\hat v_l \rangle_{\cH} = \sum_{l=2}^N \hat \gamma_l \langle u_2,\hat u_l \rangle_{\cH} \langle v_2,\hat v_l \rangle_{\cH}\leq \hat \gamma_2, 
    $$
    and the upper bound is again achieved for $u_2 = \hat u_2$ and $v_2 = \hat v_2$. Iterating up to $i=s$, we obtain that the solution of Eq.~\eqref{eq:svd_prob} is $(\hat U,\hat V)$.
\end{proof}
From this formulation of $(\hat U,\hat V)$ we can further relate it to a generalized eigenvalue problem and prove Proposition~\ref{prop:gen_eigen}.

\begin{proof}[Proof of Proposition~\ref{prop:gen_eigen}] We omit the subscript $\la$ for clarity. We have
    $$
    \operatorname{Tr}(U^*\hat C_{N,n,\la} V) = \frac{1}{N}\sum_{i=1}^s \sum_{l=1}^N  \langle u_i,\hat f_l' \rangle_{\cH} \langle v_i,\hat f_l \rangle_{\cH},
    $$
    hence the columns of $U$ and $V$ can be restricted to $\operatorname{span}\{\hat f_1', \ldots, \hat f_N' \}$ and $\operatorname{span}\{\hat f_1, \ldots, \hat f_N \}$ respectively. Therefore every solution of Problem (\ref{eq:svd_prob}) can be written $U = AR, V = BS$ where $R, S \in R^{N \times s}$. The objective can then be re-written
    $$
    \operatorname{Tr}(U^*\hat C_{N,n,\la} V) = \operatorname{Tr}(R^{\top}A^*AB^*BS) = \operatorname{Tr}(R^{\top}QJS),
    $$
    and the constraints
    $$
    I_s=U^*U=R^{\top}A^*AR = R^{\top}QR, \qquad I_s=V^*V=S^{\top}B^*BS = S^{\top}JS
    $$
    Therefore Problem (\ref{eq:svd_prob}) is equivalent to 
    \begin{equation*} 
    \begin{aligned}
    \max_{R,S\in \R^{N \times s}} \quad & \operatorname{Tr}(R^{\top}QJS)\\
    \textrm{s.t.} \quad & R^{\top}QR=I_s\\
      & S^{\top}JS=I_s    \\
    \end{aligned}
    \end{equation*}
    with solution $(\hat R, \hat S)$ linked to the solution $(\hat U, \hat V)$ of Eq.~\eqref{eq:svd_prob} through $\hat U = A\hat R, \hat V = B\hat S$. The Lagrangian for this problem is 
    $$
    \cL(R,S,\Lambda, \Gamma) = \operatorname{Tr}(R^{\top}QJS) + \operatorname{Tr}(\Lambda(R^{\top}QR - I_s)) + \operatorname{Tr}(\Gamma(S^{\top}JS - I_s)),
    $$
    where $\Lambda, \Gamma \in \R^{s \times s}$ are diagonal matrices whose entries are the Lagrange multipliers (for a proof that $\Lambda, \Gamma$ can be taken as diagonal matrices see Appendix B in \cite{ghojogh2019eigenvalue}). Equating the derivative of $\cL$ with respect to the primal variables gives us 
    $$
    QJS +QR\Lambda = 0, \qquad  JQR + JS\Gamma=0.
    $$
    Multiplying on the left respectively by $R^{\top}$ and $S^{\top}$ and subtracting gives us 
    $$
    R^{\top}QR\Lambda = S^{\top}JS\Gamma.
    $$
    Plugging the constraints $R^{\top}QR=S^{\top}JS=1$ implies $\Lambda = \Gamma$. Therefore we are looking for the largest diagonal real matrix $\Gamma$ that solves the following expression with respect to $R,S$:
    \begin{gather} \label{eq;QJ_prob}
     \begin{bmatrix} 0 & QJ \\ JQ & 0 \end{bmatrix}
     \begin{bmatrix} R \\ 
     S\end{bmatrix}
     =
     \Gamma
      \begin{bmatrix} Q & 0 \\ 0 & J \end{bmatrix}
       \begin{bmatrix} R \\ 
     S\end{bmatrix}
    \end{gather}
    with constraints $R^{\top}QR=S^{\top}JS=1$. This is the generalized eigenvalue problem stated in Proposition~\ref{prop:gen_eigen}. Denoting by $\{(\hat \alpha^{\top}_i, \hat \beta^{\top}_i)^{\top}\}_{i=1}^s$ the generalized eigenvectors associated to the $s-$largest generalized eigenvalues, the solution $(\hat R, \hat S)$ of Eq.~\eqref{eq;QJ_prob} is $\hat R = [\hat \al_1, \ldots \hat \al_s], \hat S = [\hat \be_1, \ldots \hat \be_s] \in \R^{N \times s}$.
\end{proof}

\begin{proof}[Proof of Proposition~\ref{prop:compute_inference}]
Recall that $\hat{f}_{T,\lambda_*}$ is defined as the solution of
$$
\argmin_{f \in \hat \cH_s}\sum_{j=1}^{n_T} \left(f(x_{T,j})-y_{T,j}\right)^2 + n_T\lambda_* \|f\|_{\mathcal{H}}^2 
$$
For $f \in \hat \cH_s$, we define $\beta \doteq \hat V^*f = (\langle f, \hat v_1 \rangle_{\cH}, \ldots, \langle f, \hat v_s \rangle_{\cH})^{\top} \in \R^{s}$ (those are the coordinates of $f$ in the basis $\{\hat v_1, \ldots, \hat v_s\}$ of $\hat \cH_s$. For $x \in \cX$ we define 
$$
\tilde{x} \doteq \hat V^*\phi(x) = (\langle \phi(x), \hat v_1 \rangle_{\cH}, \ldots, \langle \phi(x), \hat v_s \rangle_{\cH})^{\top} = (\hat v_1(x), \ldots, \hat v_s(x))^{\top} \in \R^{s},$$ 
those are the coordinates of $\phi(x)$ in the basis $\{\hat v_1, \ldots, \hat v_s\}$ of $\hat \cH_s$. We then have
$$
f(x) = \hat P f(x) = \langle \hat P f, \phi(x) \rangle_{\cH} = (\hat V^*f)^{\top}(\hat V^*\phi(x)) = \beta^{\top}\tilde{x}. 
$$
Furthermore,
$$
\|f\|_{\cH}^2 = \langle \hat P f, f \rangle_{\cH} = (\hat V^*f)^{\top}(\hat V^*f) =   \beta^{\top}\be.
$$
hence we can re-frame Eq.~\eqref{eq:krr_inference} as 
\begin{equation}\label{eq:krr_inference_reduced}
         \beta_{T,\lambda_*} \doteq \argmin_{\beta \in \R^s} \sum_{j=1}^{n_T} \left(\beta^{\top}\tilde x_{T,j}-y_{T,j}\right)^2 + n_T\lambda_* \|\beta\|_{2}^2 
\end{equation}
with $\hat{f}_{T,\lambda_*} = \hat P \hat{f}_{T,\lambda_*} = \hat V \hat V^* \hat{f}_{T,\lambda_*} = \hat V\beta_{T,\lambda_*}$. Solving for Eq.~\eqref{eq:krr_inference_reduced} gives
\begin{align}
        \beta_{T,\lambda_*} &= (X_TX_T^{\top} + n_T \la_* I_s)^{-1}X_TY_T. \nonumber
\end{align}
Using $(X_TX_T^{\top} + n_T \la_* I_s)^{-1}X_T = X_T(X_T^{\top}X_T + n_T \la_* I_s)^{-1}$, we can also write 
\begin{align}
        \beta_{T,\lambda_*} &= X_T(K_T + n_T \la_* I_s)^{-1}Y_T. \nonumber
\end{align}
We can choose one form or the other depending if $n_T \leq s$ or $n_T > s$.
\end{proof}

\subsection{Proof of Remark~\ref{rem:rem_exp}} \label{sec:exp_setting}

To get a bound that can handle the case where $N$ is exponential in $n$ we need bounds on the bias and variance that are free of constraints of the type $n \geq A_{\lambda}$. We start with the bias.

\begin{prop}
Suppose Assumption~\ref{asst:src} holds with $r \in [0,1]$.
For $\lambda \in (0,1]$ and $n \geq 1$, 
\begin{equation*}
    \|\bar C_{N,n,\lambda} - C_N\|_{HS} \leq J_1\lambda^{r}\left(1+\frac{1}{\lambda\sqrt{n}}\right),
\end{equation*}
where $J_1$ depends on $\kappa$ and $R$.
\end{prop}
\begin{proof} We proved in the proof of Theorem~\ref{theo:bias_pf} that
$$
\|\bar C_{N,n,\lambda} - C_N\|_{HS} \leq \frac{2 R\kappa}{N}\sum_{i=1}^N \left\|f_i - \mathbb{E}[\hat f_{i,{\lambda}}] \right\|_{\mathcal{H}}.
$$
We then use the following decomposition
$$
\left\| f_i - \E(\hat{f}_{i,\lambda})\right\|_{\cH} \leq \left\| f_i - f_{i,\lambda}\right\|_{\cH} + \left\| f_{i,\lambda} - \E(\hat{f}_{i,\lambda})\right\|_{\cH}.
$$
By Proposition~\ref{prop:bias_gamma_norm}, the first term is bounded by $R\lambda^r$. By Lemma~\ref{lma:lemma_1} and Eq.~\eqref{eq:pop_estim},
\begin{equation*}
    \left\|f_{i,\lambda} - \mathbb{E}[\hat f_{i,\lambda}]\right\|_{\mathcal{H}} = \lambda  \left\|\left(\mathbb{E}\left[\hat \Sigma_{i,\lambda} ^{-1}\right] - \Sigma_{i,\lambda} ^{-1}\right)f_i\right\|_{\mathcal{H}} \leq \lambda  \mathbb{E}\left[\left\|\left(\Sigma_{i,\lambda}^{-1} - \hat \Sigma_{i,\lambda} ^{-1}\right)f_i\right\|_{\mathcal{H}}\right],
\end{equation*}
where we used Jensen's inequality. Using the first order decomposition 
\begin{equation*}
    A^{-1} - B^{-1}=B^{-1}(B-A)A^{-1},
\end{equation*}
we have
\begin{equation*}
\begin{aligned}
    \lambda  \mathbb{E}\left[\left\|\left(\Sigma_{i,\lambda}^{-1} - \hat \Sigma_{i,\lambda} ^{-1}\right)f_i\right\|_{\mathcal{H}}\right] &= \lambda  \mathbb{E}\left[\left\|\hat \Sigma_{i,\lambda} ^{-1}\left(\Sigma_i - \hat \Sigma_{i}\right)\Sigma_{i,\lambda}^{-1}\Sigma_{i,\lambda}^{r}\Sigma_{i,\lambda}^{-r}f_i\right\|_{\mathcal{H}}\right] \\ &\leq R\mathbb{E}\left[\left\|(\Sigma_i - \hat \Sigma_{i})\Sigma_{i,\lambda}^{r-1}\right\|\right] 
\end{aligned}
\end{equation*}
To bound this term, we use Proposition~\ref{prop:general_cov_concentration} with $A = I_{\cH}$, $B = \Sigma_{i,\lambda}^{r-1}$, $C_A = \kappa$, $C_B = \kappa\la^{r-1}$ and $\sigma^2 = \kappa^2\la^{r-1}$, and convert the high probability bound to a bound in expectation with Proposition~\ref{prop:bound_conversion}. There is a universal constant $C > 0$ such that
\begin{equation*}
    \E\left\|(\Sigma_i - \hat \Sigma_{i})\Sigma_{i,\lambda}^{r-1}\right\| \leq 
    \E\left\|(\Sigma_i - \hat \Sigma_{i})\Sigma_{i,\lambda}^{r-1}\right\|_{HS} \leq C\left(\sqrt{\frac{(\kappa^2 \la^{r-1})^2}{n}} + \frac{\kappa^2 \la^{r-1}}{n}\right) \leq C\frac{\kappa^2 \la^{r-1}}{\sqrt{n}}
\end{equation*}
hence,  \[\|f_{i,\lambda} -\E(\hat{f}_{i,\lambda})\| \leq J\frac{\lambda^r}{\lambda\sqrt{n}},\] where $J$ depends on $\kappa$ and $R$. Putting it together, we obtain 
$$
\|\bar C_{N,n,\lambda} - C_N\|_{HS} \leq J_1\lambda^r \left(1 + \frac{1}{\lambda\sqrt{n}}\right),
$$
where $J_1$ depends on $\kappa$ and $R$, which concludes the proof.
\end{proof}

For the variance part, we use the bound obtained in Theorem~\ref{theo:var_pf}. For $\la \in (0,1],\tau,\delta \geq \log(2)$ and $N,n \geq 1$, with probability greater than $1-2e^{-\tau}-4Ne^{-\delta}$,
\begin{equation*}
   \left\|\hat C_{N,n,\lambda} -  \bar C_{N,n,\lambda}\right\|_{HS}  \leq c_1\left\{\left(\frac{\delta^2}{n \lambda^2} + \frac{\delta}{\sqrt{n} \lambda}  + \frac{e^{-\delta}}{\lambda}\right)\sqrt{\frac{\tau}{N}} + \frac{e^{-\delta}}{\lambda} \right\},
\end{equation*}
with $c_1$ a constant depending on $Y_{\infty}$, $\max_{i \in [N]}\left\| f_i \right\|_{\cH}$, $\kappa$ and $R$. Combining the bias and variance bound, we obtain that for all $\la \in (0,1]$, $\delta,\tau \geq \log(2)$, $N \geq \tau$ and $n \geq 1$ large enough so that $\frac{\delta}{\sqrt{n} \lambda} \leq 1$, with probability greater than $1-2e^{-\tau}-4Ne^{-\delta}$
$$
\|\hat C_{N,n,\lambda} - C_N\|_{HS} \leq C_2 \left( \frac{\delta\sqrt{\tau}}{\sqrt{nN}\la} + \frac{e^{-\delta}}{\lambda} +  \lambda^{r}\left(1+\frac{1}{\lambda\sqrt{n}}\right) \right),
$$
with $C_2$ a constant that depends on $Y_{\infty}, \kappa, R$ and $\max_{i \in [N]}\left\| f_i \right\|_{\cH}$. Plugging $\delta = 12\log(Nn)$, leads to 
\begin{equation*} 
\|\hat C_{N,n,\lambda} - C_N\|_{HS} \leq C_2 \left(\frac{\log (nN)\sqrt{\tau}}{\sqrt{nN}\la} + \frac{1}{\la (nN)^{12}} +  \lambda^{r}\left(1+\frac{1}{\lambda\sqrt{n}}\right)\right),
\end{equation*}
with probability greater than $1-2e^{-\tau}-o((nN)^{-10})$. When $N$ is exponential in $n$, the bias term dominates and it is minimized with $\lambda = n^{-1/2},$ leading to the bound in Remark~\ref{rem:rem_exp}.

\section{Auxiliary Results} \label{ref:aux_results}
\begin{prop} \label{prop:op_C}
    Under Assumption~\ref{asst:rich}, 
    $$
    C_N \doteq \frac{1}{N} \sum_{i=1}^N f_i \otimes f_i, 
    $$
    is such that  $\operatorname{ran} C_N = \mathcal{H}_s$.
\end{prop}
\begin{proof}
    $C_N = SS^*$ where $S:\R^N \to \cH, \al \mapsto \sum_{i=1}^N \al_i f_i $, hence $\operatorname{ran} C_N = \operatorname{ran} SS^* = \operatorname{ran} S = \operatorname{span}\{f_1, \ldots, f_N\} = \cH_s$, where the last equality follows from Assumption~\ref{asst:rich}.
\end{proof} 
\begin{prop}\label{prop:op_trick}
Let $\cH$ be a Hilbert space, let $C$ and $D$ be two bounded self-adjoint positive semidefinite linear operators and $\lambda >0$, then \[\left\|C_{\lambda}^{-1/2}D_{\lambda}^{1/2} \right\| = \left\|\left(I - D_{\lambda}^{-1/2}(D-C)D_{\lambda}^{-1/2}\right)^{-1}\right\|^{1/2},\] and
\[\left\|C_{\lambda}^{1/2}D_{\lambda}^{-1/2} \right\| = \left\|I - D_{\lambda}^{-1/2}(D-C)D_{\lambda}^{-1/2}\right\|^{1/2}.\]
\end{prop}
\begin{proof}
    First note that 
    $$
    \begin{aligned}
        C_{\lambda} &= D_{\lambda} + C-D = D_{\lambda}^{1/2}\left(I + D_{\lambda}^{-1/2}(C-D)D_{\lambda}^{-1/2}\right)D_{\lambda}^{1/2} \\
        \implies C_{\lambda}^{-1} &= D_{\lambda}^{-1/2}\left(I + D_{\lambda}^{-1/2}(C-D)D_{\lambda}^{-1/2}\right)^{-1}D_{\lambda}^{-1/2}.
    \end{aligned}
    $$
    Hence,
    $$
    \begin{aligned}
        \left\|C_{\lambda}^{1/2}D_{\lambda}^{-1/2} \right\|^2 &= \left\|D_{\lambda}^{-1/2}C_{\lambda}D_{\lambda}^{-1/2} \right\| = \left\|I + D_{\lambda}^{-1/2}(C-D)D_{\lambda}^{-1/2} \right\|  \\
        \left\|C_{\lambda}^{-1/2}D_{\lambda}^{1/2} \right\|^2 &= \left\|D_{\lambda}^{1/2}C_{\lambda}^{-1}D_{\lambda}^{1/2} \right\| = \left\|\left(I + D_{\lambda}^{-1/2}(C-D)D_{\lambda}^{-1/2}\right)^{-1} \right\|.
    \end{aligned}
    $$
\end{proof}
\begin{lemma}\label{lma:eff_dim_exp}
    Let $\cH$ be a separable RKHS on $\cX$ w.r.t. a bounded and measurable kernel $K$, and $\mu$ be a probability distribution on $\cX$. $\Sigma_{\mu} \doteq \E_{X \sim \mu}[K(X,\cdot) \otimes K(X,\cdot)]$ is the covariance operator associated to $(\cH,\mu)$. For all $\la > 0$:
    $$
    \int_{\cX} \|\left(\Sigma_{\mu} + \la I_{\cH}\right)^{-1/2}K(x,\cdot)\|^2_{\cH}d\mu(x) = \operatorname{Tr}\left(\left(\Sigma_{\mu} + \la I_{\cH}\right)^{-1}\Sigma_{\mu}\right).
    $$
    $\mathcal{N}(\lambda) = \operatorname{Tr}\left(\left(\Sigma_{\mu} + \la I_{\cH}\right)^{-1}\Sigma_{\mu}\right)$ is the effective dimension associated to $(\cH,\mu)$. Under Assumption~\ref{asst:evd} with $c >0$ and $0 < p \leq 1$,
    \[\mathcal{N}(\lambda) \leq D\lambda^{-p},\]
    with $D \doteq c/(1-p)1_{p < 1} + \operatorname{Tr}(\Sigma_{\mu})1_{p = 1}$. Furthermore, under Assumption~\ref{asst:emb} with $k_{\alpha,\infty} > 0$ and $0 < \alpha \leq 1$,
    $$
    \left\|\left(\Sigma_{\mu}+\lambda I_{\cH} \right)^{-\frac{1}{2}} k(X,\cdot)\right\|_{\mathcal{H}_X}^2 \leq k_{\alpha,\infty}^2\lambda^{-\alpha}.
    $$
\end{lemma}
\begin{proof}
The first inequality is proven in \cite{caponnetto2007optimal} for $p < 1$ and in Lemma~$11$ \cite{fischer2020sobolev} for $p=1$. The second inequality is proven in Lemma 13 \cite{fischer2020sobolev}.
\end{proof}
The forthcoming result provides a bound on the approximation error $f_{i,\lambda} - f_i$. While similar bounds can be found in Theorem 4 \cite{smale2007learning} (with the correspondence $r' = r+1/2$) and Lemma~$14$ \cite{fischer2020sobolev} (with the correspondence $\beta = 2r+1$), it is worth noting that these references assert that the bound saturates at $r=1/2$ when measuring the approximation error in both the $L_2$-norm and the $\cH$-norm. In contrast, our result reveals that the saturation point occurs at $r=1$ when utilizing the $\cH$-norm. Notably, we are aware of only one reference, \cite{blanchard2018optimal}, that acknowledges the saturation point for $r$ exceeding $1/2$ when the approximation is measured in norms stronger than the $L_2-$norm. As they consider a setting with abstract spectral regularization techniques, we offer our own proof. 
\begin{prop}\label{prop:bias_gamma_norm} 
For $i \in [N]$ and $\omega \in \{0,1/2\}$, let $\|\cdot\|_{\omega}$ be $\|\cdot\|_{\cH}$ if $\omega = 0$ and $\|\cdot\|_{L_2(\mu_i)}$ if $\omega = 1/2$. Then if Assumption \ref{asst:src} is satisfied with $0 \leq r \leq 1-\omega$ we have
\[\|f_{i,\lambda} - f_i\|_{\omega} \leq R \lambda^{r+\omega}.\]
\end{prop}
\begin{proof}
    Fix $i \in [N]$ and $\omega \in \{0,1/2\}$. With the same argument as in Eq.~\eqref{eq:norm_conv} we have that if $f \in \cH$, then $\|f\|_{L_2(\mu_i)} = \|\Sigma_i^{1/2}f\|_{\cH}$. Therefore, by Eq.~\eqref{eq:pop_estim},
    $$
    \begin{aligned}
        \|f_{i,\lambda} - f_i\|_{\omega} &= \la \|\Sigma_{i,\la}^{-1}f_i\|_{\omega} \\
        &= \la \|\Sigma^{\omega}_i\Sigma_{i,\la}^{-1}f_i\|_{\cH} \\
        &= \la \|\Sigma^{\omega}_i\Sigma_{i,\la}^{-1}\Sigma_{i}^{r}\Sigma_{i}^{-r}f_i\|_{\cH} \\ 
        &\leq \la R\|\Sigma_{i,\la}^{r-1+\omega}\| \\
        & \leq R\la\la^{r-1+\omega}  \\ 
        &= R\la^{r+\omega},
    \end{aligned}
    $$
    where in the last inequality we used $r+\omega \leq 1$.
\end{proof}
Analysis of the estimation error $\hat f_{i,\lambda} - f_{i,\lambda}$ in kernel ridge regression is an important part of our theory. Below we provide two different results on this analysis.
\begin{theorem}[Theorem 1 \cite{smale2007learning}]\label{prop:variance_regression_smale}
Let $\hat f_{i,\lambda}$ be the solution from Eq.~\eqref{eq:def_source_estim} and $f_{i,\lambda}$ its population version as defined in Eq.~\eqref{eq:pop_estim}. Suppose Assumption~\ref{asst:mom} holds with $Y_{\infty} \geq 0$. For $i \in [N]$, $ \delta \geq \log(2)$ and $ \lambda>0$, the following bound is satisfied with probability not less than $1-2e^{-\delta}$
$$
\left\|\hat f_{i,\lambda} - f_{i,\lambda}\right\|_{\cH} \leq \frac{6 \kappa Y_{\infty}\delta}{\sqrt{n} \lambda}.
$$
\end{theorem}
Theorem 16 \cite{fischer2020sobolev} provides a refined bound on the estimation error at the cost of an additional constraint on the relationship between $\lambda$ and $n$. The following result is a simple extension of this result.
\begin{theorem} \label{prop:variance_regression}
Let $\hat f_{i,\lambda}$ be the solution from Eq.~\eqref{eq:def_source_estim} and $f_{i,\lambda}$ its population version as defined in Eq.~\eqref{eq:pop_estim}. Suppose Assumptions~\ref{asst:evd}, \ref{asst:emb} and \ref{asst:mom} hold with $c> 0$, $0 < p \leq 1$, $\alpha \in [p,1]$, $k_{\alpha,\infty} \geq 0$ and $Y_{\infty} \geq 0$. 
For $ \delta \geq 1$, $\lambda < 1 \wedge \min_{i \in [N]}\left\|\Sigma_i\right\|$, and $n \geq c_0\delta\left(1 + p\log (\la^{-1})\right) \lambda^{-\alpha}$, the following bound is satisfied with probability not less than $1-4e^{-\delta}$
$$
\left\|\hat f_{i,\lambda} - f_{i,\lambda}\right\|_{\cH} \leq \frac{c_1\delta}{\sqrt{n} \lambda^{\frac{1}{2}+\frac{p}{2}}}\sqrt{1+ \frac{1}{n \lambda^{\alpha-p}}},
$$
where $c_0$ depends on $k_{\al,\infty},D$ and $c_1$ depends on $Y_{\infty}, k_{\alpha,\infty}, D$ and $R$.
Furthermore, if $\Sigma_i$ has finite rank $k$ with eigensystem $\{(e_{i,j}, \lambda_{i,j})\}_{j=1}^k$ such that $\sqrt{\la_{i,j}}e_{i,j}$ is an orthonormal basis of $\operatorname{range}(\Sigma_i)$. Then for $ \delta \geq 1$, $0 < \lambda < 1 \wedge \min_{i \in [N]}\left\|\Sigma_i\right\|$, and $n \geq c_2\delta k \left(1 + p\log (\la^{-1})\right)$, the following bound is satisfied with probability not less than $1-4e^{-\delta}$
\begin{equation}
\left\|\hat f_{i,\lambda} - f_{i,\lambda}\right\|_{\cH} \leq c_3\delta\sqrt{\frac{k}{n}} \label{eq:fin_rank}
\end{equation}
where $c_2$ depends $D,\kappa,\min_{i \in [N]}\{\lambda_{i,d}\}$ and $c_3$ depends on $Y_{\infty}, \kappa, \max_{i \in [N]}\{\lambda_{i,1}\}, \min_{i \in [N]}\{\lambda_{i,d}\}$ and $R$.
\end{theorem}
\begin{proof} Define the following two terms for $i \in [N]$:
\begin{align*}
q_{i,\la} &\doteq \log\left(2e\cN_i(\la)\frac{\|\Sigma_i\| + \la}{\|\Sigma_i\|}\right) \\
A_{i,\la,\delta} &\doteq 8k_{\al,\infty}^2 \delta q_{i,\la} \la^{-\alpha}
\end{align*}
Apply Theorem~$16$ in \cite{fischer2020sobolev} (by letting $\gamma = 1$ in their result), for $\tau \geq 1$, with probability over $1-4e^{-\delta}$ for $n \geq A_{i,\lambda,\delta}$, 
\[\left\|\hat f_{i,\lambda} - f_{i,\lambda}\right\|_{\cH} \leq \frac{24 \delta}{\sqrt{n} \lambda^{1/2}}\sqrt{Y_{\infty}^2\mathcal{N}_{i}(\lambda)+k_{\al,\infty}^2 \frac{\left\|f_i-f_{i, \lambda}\right\|_{L_2}^2}{\lambda^{\alpha}}+2k_{\al,\infty}^2 \frac{\max\{Y_{\infty},\|f_i-f_{i,\lambda}\|_{\infty}\}^2}{n \lambda^\alpha}}.\] 
By Lemma~\ref{lma:eff_dim_exp}, under Assumption~\ref{asst:evd}, we have $\mathcal{N}_{i}(\lambda) \leq D\lambda^{-p}$; and by Proposition~\ref{prop:bias_gamma_norm}, under Assumption~\ref{asst:src}, we have $\left\|f_i-f_{i, \lambda}\right\|_{L_2}^2 \leq R^2\lambda^{2r+1}$ (plug $\omega=1/2$ in Proposition~\ref{prop:bias_gamma_norm}). In addition, by Assumption~\ref{asst:emb}, we have
\begin{align*}
\left\|f_i - f_{i,\lambda}\right\|_{\infty} &\leq k_{\alpha,\infty} \left\|f_i - f_{i,\lambda}\right\|_{\cH} \\
&\leq k_{\alpha,\infty}R\lambda^{r} \leq  k_{\alpha,\infty}R,
\end{align*}
where the second last step follows from Proposition~\ref{prop:bias_gamma_norm} with $\omega=0$ and the last step uses $\la \leq 1$. As such, denote $Y_R \doteq \max\{Y_{\infty},  k_{\alpha,\infty}R\}$ we have 
\begin{align*}
\left\|\hat f_{i,\lambda} - f_{i,\lambda}\right\|_{\cH} &\leq \frac{24 \delta}{\sqrt{n} \lambda^{1/2}}\sqrt{Y_{\infty}^2D \lambda^{-p}+k_{\al,\infty}^2 R^2\lambda^{2r+1-\alpha}+\frac{2k_{\al,\infty}^2Y_R^2}{n \lambda^\alpha}} \\ &\leq \frac{b_0\delta}{\sqrt{n} \lambda^{\frac{1}{2}+\frac{p}{2}}}\sqrt{1+\lambda^{2r+p+1-\alpha}+\frac{1}{n \lambda^{\alpha-p}}},
\end{align*}
where $b_0$ is a constant that depends on $Y_{\infty}, k_{\alpha,\infty}, D$ and $R$. We obtain the desired bound by noticing that $\alpha \leq 1 \leq 2r+p+1$ and $\lambda \leq 1$.\\
Next, let us simplify the constraint $n \geq A_{i,\lambda,\delta}$. Let us fix some lower bound $0< c \leq 1$ with $c \leq \min_{i \in [N]}\left\|\Sigma_i\right\|$. $\la$ will be chosen as a function of $(n,N)$ or $n$ only with the property that $\la \to 0$ when $n \to \infty$. We choose an index bound $n_0 \geq 1$ such that $\lambda \leq c \leq \min \left\{1,\min_{i \in [N]}\left\|\Sigma_i\right\|\right\}$ for all $n \geq n_0$. Using the definition $q_{i,\lambda}$, $\lambda \leq c \leq \min_{i \in N}\left\|\Sigma_i\right\|$, $\mathcal{N}_i\left(\lambda\right) \leq D \lambda^{-p}$ from Lemma~\ref{lma:eff_dim_exp}, we get, for $n \geq n_0$,
$$
\begin{aligned}
8k_{\al,\infty}^2 \delta \frac{q_{i,\lambda}}{\lambda^\alpha} 
& =8k_{\al,\infty}^2 \delta \frac{\log \left(2 e \mathcal{N}_i\left(\lambda\right)\left(1+\lambda /\left\|\Sigma_i\right\|\right)\right)}{\lambda^\alpha} \\
& \leq 8k_{\al,\infty}^2 \delta \frac{\log \left(4 e D \lambda^{-p}\right)}{\lambda^\alpha} \\
& =8k_{\al,\infty}^2 \delta\left(\log (4 e D) + p\log (\la^{-1})\right) \lambda^{-\alpha}.
\end{aligned}
$$
For the second bound, since $\Sigma_i$ has finite rank $k$, let us show that Assumption~\ref{asst:emb} holds with $\alpha=0$. We have, 
\begin{align*}
\sum_{j=1}^d \lambda_{i,j}^{\alpha} e^2_{i,j}(x) = \sum_{j=1}^d e^2_{i,j}(x) = \sum_{j=1}^d \langle e_{i,j}, \phi(x) \rangle_{\cH}^2 \leq  \sum_{j=1}^d \lambda_{i,j}^{-1}\|\sqrt{\lambda_{i,j}}e_{i,j}\|_{\cH}^2 \|\phi(x) \|_{\cH}^2 \leq \frac{\kappa^2}{\lambda_{i,k}} k,
\end{align*}
where we used $\|\sqrt{\lambda_{i,j}}e_{i,j}\|_{\cH}=1$, and $\|\phi(x) \|_{\cH}^2 = K(x,x) \leq \kappa^2$. As such, the constant for Assumption~\ref{asst:emb} can be taken as $k_{0,\infty}^2 = k\kappa^2/\lambda_{i,k}$. 
Apply Theorem~$16$ in \cite{fischer2020sobolev} (by letting $\gamma = 0$ this time) and let $\alpha =0$, with probability over $1-4e^{-\delta}$ for $n \geq A_{i,\lambda,\delta}$,
\begin{align*}
\left\|\hat f_{i,\lambda} - f_{i,\lambda}\right\|_{L_2}  \leq \frac{24 \delta}{\sqrt{n}}\sqrt{Y_{\infty}^2\mathcal{N}_{i}(\lambda)+k_{0,\infty}^2 \left\|f_i-f_{i, \lambda}\right\|_{L_2}^2 +2k_{0,\infty}^2 \frac{\max\{Y_{\infty},\|f_i-f_{i,\lambda}\|_{\infty}\}^2}{n}}.
\end{align*}
We have $\cN_i(\lambda) =\text{Tr}\left( \Sigma_i \Sigma_{i,\lambda}^{-1}\right) = \sum_{j=1}^k \frac{\lambda_{i,j}}{\lambda_{i,j}+\lambda} \leq k$. Furthermore, since $K$ is bounded by $\kappa^2$, we have 
\begin{align*}
\left\|f_i - f_{i,\lambda}\right\|_{\infty} &\leq& \kappa \left\|f_i - f_{i,\lambda}\right\|_{\cH} = \kappa\left\| \Sigma_i^{1/2} (f_{i} - f_{i,\lambda})\right\|_{L_2} \leq \kappa\sqrt{\lambda_{i,1}}\left\|f_{i} - f_{i,\lambda}\right\|_{L_2}.
\end{align*}
Therefore,
\begin{align*}
\left\|\hat f_{i,\lambda} - f_{i,\lambda}\right\|_{L_2}  \leq \frac{24 \delta\sqrt{k}}{\sqrt{n}}\sqrt{Y_{\infty}^2+\frac{\kappa^2}{\lambda_{i,k}}\left\|f_i-f_{i, \lambda}\right\|_{L_2}^2 + 2 \frac{\kappa^2}{\lambda_{i,k}} \frac{\max\{Y_{\infty},\kappa\sqrt{\lambda_{i,1}}\left\|f_{i} - f_{i,\lambda}\right\|_{L_2}\}^2}{n}}.
\end{align*}
Finally, we obtain the result using $\lambda_{i,k}^{-1}\leq (\min_{i \in [N]}\{\lambda_{i,k}\})^{-1}$, $\lambda_{i,1} \leq\max_{i \in [N]}\{\lambda_{i,1}\}$ and (since $\la \leq 1$) $\left\|f_i-f_{i, \lambda}\right\|_{L_2} \leq R\lambda^{r+1/2} \leq R$. \\
Note that plugging $\alpha=0$ and $k_{0,\infty}^2 = k\kappa^2/\lambda_{i,k}$ in $A_{i,\la,\delta}$ gives the constraint $n \geq 8\delta k\kappa^2q_{i,\la}/\lambda_{i,k}$. We can similarly simplify it to $n \geq 8\delta k\kappa^2 \left(\log (4 e D) + p\log (\la^{-1})\right)(\min_{i \in [N]}\{\lambda_{i,k}\})^{-1}$ for $0 < \lambda <\min \left\{1,\min_{i \in [N]}\left\|\Sigma_i\right\|\right\}$.
\end{proof}
The next results provide bounds on the bias $f_{i} -\E(\hat{f}_{i,\la})$ in kernel ridge regression.
\begin{prop} \label{prop:bias_exp}
Let $\hat f_{i,\lambda}$ be the solution from Eq.~\eqref{eq:def_source_estim}. Suppose Assumptions~\ref{asst:evd} and \ref{asst:src} hold with $0 < p \leq 1$ and $r \in [0,1]$. We have, for any $\la \in (0,1]$, $\delta \geq 1$, and $n \geq c_1\delta \lambda^{-p-1},$ 
\[
\|f_{i} -\E(\hat{f}_{i,\la})\|_{\cH} \leq c_2\lambda^{r}\left(1 + \frac{e^{-\delta}}{\lambda} \right)
\]
where $c_1$ is a constant depending only on $D$, $\kappa$, and $c_2$ is a constant depending only on $R$, $\kappa$.
\end{prop}
\begin{proof}
By Lemma~\ref{lma:lemma_1} we have 
\begin{align*}
\left\|f_i -\E(\hat{f}_{i,\lambda})\right\|_{\cH} &= \lambda \|\E(\hat \Sigma_{i,\lambda}^{-1})f_i\|_{\cH}\\
& \leq \lambda \mathbb{E}\left\|\hat{\Sigma}_{i,\lambda}^{-1} \Sigma_{i,\lambda} \Sigma_{i,\lambda}^{-1}f_i \right\|_{\mathcal{H}}\\
& \leq  \lambda \mathbb{E}\left\|\hat{\Sigma}_{i,\lambda}^{-1} \Sigma_{i,\lambda} \right\| \left\|\Sigma_{i,\lambda}^{-1}f_i \right\|_{\mathcal{H}}\\
& = \lambda \mathbb{E}\left\|\hat{\Sigma}_{i,\lambda}^{-1} \Sigma_{i,\lambda} \right\| \left\|\Sigma_{i,\lambda}^{r-1} \Sigma_{i,\lambda}^{-r} \Sigma_i^r \Sigma_i^{-r}f_i \right\|_{\mathcal{H}},\\
& \leq  R\lambda^r \mathbb{E}\left\|\Sigma_{i,\lambda}\hat{\Sigma}_{i,\lambda}^{-1} \right\|, 
\end{align*}
To bound $\mathbb{E}\left\|\Sigma_{i,\lambda}\hat{\Sigma}_{i,\lambda}^{-1} \right\|$, we notice that 
\[
\hat \Sigma_{i,\lambda} = (\Sigma_i + \lambda I - (\Sigma_i -\hat \Sigma_i)) = \left( I - (\Sigma_i -\hat \Sigma_i)\Sigma_{i,\lambda}^{-1}\right)\Sigma_{i,\lambda}.
\]
As a result, we have, 
$$
\mathbb{E}\left\|\Sigma_{i,\lambda}\hat{\Sigma}_{i,\lambda}^{-1} \right\| = \E \left\|\left( I - (\Sigma_i -\hat \Sigma_i)\Sigma_{i,\lambda}^{-1}\right)^{-1} \right\| \leq \E \sum_{k \geq 0} \|(\Sigma_i -\hat \Sigma_i)\Sigma_{i,\lambda}^{-1}\|^k,
$$
where in the second step we used Neumann series, provided $\|(\Sigma_i -\hat \Sigma_i)\Sigma_{i,\lambda}^{-1}\|  < 1$. We are left to bound $\|(\Sigma_i -\hat \Sigma_i)\Sigma_{i,\lambda}^{-1}\|$. By Proposition 5.3 \cite{blanchard2018optimal}, for any $n \geq 1$, $\lambda \in (0,1]$ and $\delta \geq \log(2)$, it holds with probability at least $1 - 2e^{-\delta}$,
$$ 
\left\| (\Sigma_i -\hat \Sigma_i)\Sigma_{i,\lambda}^{-1}\right\| \leq \left\| (\Sigma_i -\hat \Sigma_i)\Sigma_{i,\lambda}^{-1}\right\|_{HS} \leq \sqrt{\frac{8\delta \kappa^2 \cN(\lambda)}{n\lambda}} + \frac{4\delta \kappa}{n\lambda} \leq c_0\left(\sqrt{\frac{\delta }{n\lambda^{1+p}}} + \frac{\delta}{n\lambda}\right), 
$$
where we used Lemma~\ref{lma:eff_dim_exp} and $c_0$ depends on $\kappa, D$.
Note that on one hand, 
$$
c_0\sqrt{\frac{\delta }{n\lambda^{1+p}}} \leq 1/3 \iff n \geq  \frac{9c_0^2 \delta}{\lambda^{p+1}}
$$
and on the other hand,
$$
\frac{\delta c_0}{n\lambda} \leq 1/3 \iff n \geq \frac{3\delta c_0}{\lambda}.
$$
Note that $p+1 \geq 1$, therefore, since  $\lambda \leq 1$, for $
n \geq \delta \max(9c_0^2, 3c_0)\lambda^{-p-1},$ we have with probability over $1-2e^{-\delta}$ 
$$
\left\| (\Sigma_i -\hat \Sigma_i)\Sigma_{i,\lambda}^{-1}\right\| \leq 1/3 + 1/3 = 2/3.
$$
This further implies that with probability over $1-2e^{-\delta}$, 
$$
\left\|\Sigma_{i,\lambda} \hat \Sigma_{i,\lambda}^{-1} \right\| \leq \sum_k \|(\Sigma_i -\hat \Sigma_i)\Sigma_{i,\lambda}^{-1}\|^k \leq \sum_k \left(\frac{2}{3}\right)^k = 3. 
$$
Moreover, we have
\[
    \left\|\Sigma_{i,\lambda} \hat \Sigma_{i,\lambda}^{-1} \right\| \leq \left\|\Sigma_{i,\lambda}\right\| \left\| \hat \Sigma_{i,\lambda}^{-1} \right\|\leq \frac{\kappa^2+1}{\lambda}.
\]
As such, we have 
\[
\begin{aligned}
    \E\left\|\Sigma_{i,\lambda} \hat \Sigma_{i,\lambda}^{-1} \right\| &= \E \left(\left\|\Sigma_{i,\lambda} \hat \Sigma_{i,\lambda}^{-1} \right\| \mid \left\|\Sigma_{i,\lambda} \hat \Sigma_{i,\lambda}^{-1} \right\| \leq 3\right) \P(\left\|\Sigma_{i,\lambda} \hat \Sigma_{i,\lambda}^{-1} \right\| \leq 3) \\ &+  \E \left(\left\|\Sigma_{i,\lambda} \hat \Sigma_{i,\lambda}^{-1} \right\| \mid \left\|\Sigma_{i,\lambda} \hat \Sigma_{i,\lambda}^{-1} \right\| \geq  3\right)\P(\left\|\Sigma_{i,\lambda} \hat \Sigma_{i,\lambda}^{-1} \right\| \geq 3)\\
    &\leq  3 + 2\frac{\kappa^2+1}{\lambda} e^{-\delta}.
\end{aligned}
\]
\end{proof}
\begin{prop} \label{prop:bias_exp_r_0_5}
Let $\hat f_{i,\lambda}$ be the solution from Eq.~\eqref{eq:def_source_estim}. Suppose Assumptions~\ref{asst:evd}, \ref{asst:emb}, \ref{asst:src} hold with $0 < p \leq \alpha \leq 1$ and $r \in (0,1]$. We have, for any 
$ \delta \geq 1$, $\lambda < 1 \wedge \min_{i \in [N]}\left\|\Sigma_i\right\|$, and $n \geq c_0\delta\left(1 + p\log (\la^{-1})\right) \lambda^{-\alpha}$,
\[
\|f_{i} -\E(\hat{f}_{i,\la})\|_{\cH} \leq c_1 \left(e^{-\delta} + \la^{r \wedge 1/2}\right),
\]
where $c_0$ depends on $k_{\al,\infty},D$ and $c_1$ depends on $R, \kappa$.
\end{prop}
\begin{proof}
By Lemma~\ref{lma:lemma_1}, for all $i \in [N]$,
$$
f_i - \mathbb{E}[\hat f_{i,\lambda}] = \la \E\left[\hat \Sigma_{i,\la}^{-1}\right]f_i.
$$
Notice that we have 
$$
\|\E\left[\hat \Sigma_{i,\la}^{-1}\right]\| \leq \la^{-1}.
$$
Furthermore by Proposition~\ref{prop:op_trick}, we have 
$$
\hat \Sigma_{i,\la}^{-1} = \Sigma_{i,\la}^{-1/2}\hat \beta_{i,\la} \Sigma_{i,\la}^{-1/2},
$$
where $\hat \beta_{i,\la} \doteq \left(I - \Sigma_{i,\la}^{-1/2}(\Sigma_i -\hat \Sigma_i)\Sigma_{i,\la}^{-1/2}  \right)^{-1}$. By Lemma 17 \cite{fischer2020sobolev}, for all $i \in [N]$, $\delta \geq 1$, $\la > 0$ and $n\geq 1$ with probability at least $1-2e^{-\delta}$
$$
\left\|\Sigma_{i,\la}^{-1/2}(\Sigma_i -\hat \Sigma_i)\Sigma_{i,\la}^{-1/2}\right\| \leq \frac{4}{3} \cdot \frac{\delta k^2_{\al, \infty} q_{i,\la}}{n\la^{\al}}+\sqrt{2 \cdot \frac{\delta k^2_{\al, \infty} q_{i,,\la}}{n\la^{\al}}},
$$
with 
$$
q_{i,\la} \doteq \log\left(2e\cN_i(\la)\frac{\|\Sigma_i\| + \la}{\|\Sigma_i\|}\right).
$$
Therefore, if 
$n \geq \max_{i \in [N]}8k_{\al,\infty}^2 \delta q_{i,\la} \la^{-\alpha}$, with probability at least $1-2e^{-\delta}$
$$
\left\|\Sigma_{i,\la}^{-1/2}(\Sigma_i -\hat \Sigma_i)\Sigma_{i,\la}^{-1/2}\right\| \leq \frac{4}{3} \cdot \frac{1}{8}+\sqrt{2 \cdot \frac{1}{8}}=\frac{2}{3}.
$$
Consequently, $\hat \beta_{i,\la}$ can be represented by the Neumann series. In particular, the Neumann series gives us the following bound
\begin{equation*} 
\left\|\hat \beta_{i,\la}\right\|
\leq \sum_{k=0}^{\infty}\left\|\Sigma_{i,\la}^{-1/2}(\Sigma_i -\hat \Sigma_i)\Sigma_{i,\la}^{-1/2} \right\|^k
\leq \sum_{k=0}^{\infty}\left(2/3\right)^k = 3,
\end{equation*}
with probability not less than $1 - 2e^{-\delta}$. Let us define the event $E_{i,\la} \doteq \{ \|\hat \beta_{i,\la}\| \leq 3 \}$ and assume $n \geq 8k_{\al,\infty}^2 \delta q_{i,\la} \la^{-\alpha}$, we have 
$$
\| f_i - \mathbb{E}[\hat f_{i,\lambda}] \|_{\cH} \leq \la \left(\left\|\E\left[\hat \Sigma_{i,\la}^{-1}\mathbbm{1}_{E_{i,\la}}\right]f_i\right\|_{\cH} + \left\|\E\left[\hat \Sigma_{i,\la}^{-1}\mathbbm{1}_{E_{i,\la}^c}\right]f_i\right\|_{\cH} \right).
$$
On one hand, 
$$
\la\left\|\E\left[\hat \Sigma_{i,\la}^{-1}\mathbbm{1}_{E_{i,\la}^c}\right]f_i\right\|_{\cH} \leq \la\E\left[\left\|\hat \Sigma_{i,\la}^{-1}\right\|\mathbbm{1}_{E_{i,\la}^c}\right]\left\|f_i\right\|_{\cH} \leq  \kappa R \P(E_{i,\la}^c) \leq 2\kappa Re^{-\delta}.
$$
On the other hand,
$$
\begin{aligned}
\la\left\|\E\left[\hat \Sigma_{i,\la}^{-1}\mathbbm{1}_{E_{i,\la}}\right]f_i\right\|_{\cH} &=  \la\left\|\Sigma_{i,\la}^{-1/2}\E\left[\hat \beta_{i,\la}\mathbbm{1}_{E_{i,\la}}\right] \Sigma_{i,\la}^{-1/2}f_i\right\|_{\cH} \\ &\leq \sqrt{\la}\left\|\E\left[\hat \beta_{i,\la}\mathbbm{1}_{E_{i,\la}}\right] \Sigma_{i,\la}^{-1/2}\Sigma_{i,\la}^{r}\Sigma_{i,\la}^{-r}f_i\right\|_{\cH} \\ &\leq \sqrt{\la}R\left\|\E\left[\hat \beta_{i,\la}\mathbbm{1}_{E_{i,\la}}\right] \Sigma_{i,\la}^{r-1/2}\right\| \\ &\leq R\sqrt{\la}\left( (\kappa+\la)^{1/2}\mathbbm{1}_{r \geq 1/2} + \la^{r - 1/2}\mathbbm{1}_{r < 1/2}\right)\E[\|\hat \beta_{i,\la}\|\mathbbm{1}_{E_{i,\la}}] \\ &\leq 3R(\kappa + 1)^{1/2}\la^{r \wedge 1/2},
\end{aligned}
$$
where we used $\la \leq 1$. Finally the constraint $n \geq 8k_{\al,\infty}^2 \delta q_{i,\la} \la^{-\alpha}$, can be simplified to $n \geq c_0\delta\left(1 + p\log (\la^{-1})\right) \lambda^{-\alpha}$, where $c_0$ depends on $k_{\al,\infty},D$ as in the proof of Theorem~\ref{prop:variance_regression}.
\end{proof}
\begin{prop}\label{prop:xi_as} For all $i \in [N]$ and $0<\la \leq 1$ it holds almost surely that
$$
\begin{aligned}
    \left\| \hat f'_{i,\lambda} \otimes \hat f_{i,\lambda} - \E (\hat f_{i,\lambda}) \otimes \E (\hat f_{i,\lambda}) \right\|_{HS}  &\leq c/\la  \\ 
    \left\| \hat f'_{i,\lambda} \otimes \hat f_{i,\lambda} - f_{i,\lambda} \otimes f_{i,\lambda} \right\|_{HS} &\leq c /\la,  
\end{aligned}
$$
with $c \doteq Y_{\infty}^2+\max_{i \in [N]}\left\| f_i \right\|_{\cH}^2$.
\end{prop}
\begin{proof}
By definition of $\hat{f}_{i,\la}$ in Eq.~\eqref{eq:def_source_estim}, we have
$$
\begin{aligned}
    \la \|\hat{f}_{i,\la}\|_{\cH}^2 &\leq  \frac{1}{n}\sum_{j=1}^n \left(y_{i,j} - \hat{f}_{i,\la}(x_{i,j}) \right)^2 + \lambda\|\hat{f}_{i,\la}\|^2_{\mathcal{H}} \\ &=\min_{f \in \cH} \frac{1}{n}\sum_{j=1}^n \left(y_{i,j} - f(x_{i,j}) \right)^2 + \lambda\|f\|^2_{\mathcal{H}} \\ &\leq \frac{1}{n}\sum_{j=1}^n y_{i,j}^2 \\ &\leq Y_{\infty}^2,
\end{aligned}
$$
where in the second inequality we plugged $f=0$ and in the last inequality we used the bound $|y_{i,j}| \leq Y_{\infty}$ for all $i \in [N]$ and $j \in [n]$. We therefore have $\|\hat{f}_{i,\la}\|_{\cH} \leq Y_{\infty}/\sqrt{\la}$ and the same bound holds for $\|\hat{f}_{i,\la}'\|_{\cH}$. Using Lemma~\ref{lma:lemma_1}, we then have 
\begin{align*}
\left\|\hat f'_{i,\lambda} \otimes \hat f_{i,\lambda} - \E (\hat f_{i,\lambda}) \otimes \E (\hat f_{i,\lambda}) \right\|_{HS} &\leq \left\|\hat f'_{i,\lambda} \right\|_{\cH}\left\|\hat f_{i,\lambda} \right\|_{\cH} + \left\| \E (\hat f_{i,\lambda}) \right\|_{\cH}^2  \\
& \leq \frac{Y_{\infty}^2}{\la} +  \E \left\|\hat f_{i,\lambda}\right\|_{\cH}^2  \\
&\leq \frac{Y_{\infty}^2}{\la} + \left\| f_i \right\|_{\cH}^2 \\
&\leq \frac{Y_{\infty}^2 + \left\| f_i \right\|_{\cH}^2}{\la},
\end{align*}
where the last inequality follows from $\lambda \in (0,1]$. Similarly, using Eq.~\eqref{eq:ridge_estimator_bound_1_r} we have
\[
\begin{aligned}
\left\|\hat f'_{i,\lambda} \otimes \hat f_{i,\lambda} - f_{i,\lambda} \otimes f_{i,\lambda} \right\|_{HS} &\leq  \left\|\hat f'_{i,\lambda} \right\|_{\cH}\left\|\hat f_{i,\lambda} \right\|_{\cH} + \left\|f_{i,\lambda} \right\|_{\cH}^2 \\
&\leq \frac{Y_{\infty}^2}{\la} + \left\| f_i \right\|_{\cH}^2 \\
&\leq \frac{Y_{\infty}^2 + \left\| f_i \right\|_{\cH}^2}{\la}.
\end{aligned}
\]
\end{proof}

\section{Concentration inequalities} \label{sec:concentration}

\begin{prop}\label{prop:bound_conversion}\label{prop:exp_bound_conversion}
Let $X$ be a random variable taking values in $\R_+$ such that
\begin{equation*}
    \P\left(X \geq t \right) \leq c\exp \left(-\frac{1}{2}\frac{t^2}{v^2+bt}\right)
\end{equation*}
$v,b>0$, $c \geq 1$, then for all $\tau \geq 0$
$$
X \leq v\sqrt{2\tau} + b2\tau,
$$
with probability at least $(1-ce^{-\tau})_+$, where $x_+ = \max(0,x)$, $x \in \R$. 
\end{prop}

\begin{proof}
Solving for $\tau = \frac{1}{2}\frac{t^2}{v^2+bt}$ we get as positive solution $t = \sqrt{2 \tau v^2 + \tau^2 b^2} + \tau b$. Observing that 
$$
t \leq v\sqrt{2\tau} + b2\tau,
$$
gives the  bound.
\end{proof}
The next proposition provides a high probability bound on the ``whitened'' difference between the population and empirical covariance on $\hat \cH_s$ in operator norm.

\begin{prop} \label{prop:hsu_bernstein_finite_dim}
For $\la_*>0$, $\tau \geq 2.6$ and $n_T \geq 1$, the following operator norm bound is satisfied with $\mu_T^{n_T}$-probability not less than $1-e^{-\tau}$
\begin{align}
\left\|\Sigma_{\hat P, \la_*}^{-1/2}\hat P(\Sigma_T - \hat{\Sigma}_T)\hat P\Sigma_{\hat P, \la_*}^{-1/2}\right\| &\leq \frac{2 \kappa^2(\tau+ \log(s))}{3 \la_*n_T}+\sqrt{\frac{4 \kappa^2(\tau+ \log(s))}{\la_* n_T}} \nonumber
\end{align}
conditionally on $\mathcal{D}_i = \{(x_{i,j}, y_{i,j})_{j=1}^{2n}\}, i \in [N]$.
\end{prop}

\begin{proof}
    We apply Lemma 24 from \cite{hsu2012random} to $\hat \cH_s$ and $\mu_T$. $\Delta_{\la}$, $n$ and $x$ in their setting corresponds to $\Sigma_{\hat P, \la_*}^{-1/2}\hat P(\Sigma_T - \hat{\Sigma}_T)\hat P\Sigma_{\hat P, \la_*}^{-1/2}$, $n_T$ and $\hat P \phi(X)$ in our setting. As $\hat \cH_s$ is finite-dimensional, their ``Condition 1'' is automatically satisfied and we have  
    $$
        \tilde{d}_{\la_*} \doteq \operatorname{Tr}(\Sigma_{\hat P,\la_*}^{-1}\Sigma_{\hat P}) = \sum_{j=1}^s \frac{\hat \la_j}{\hat \la_j + \la_*} \leq s,
    $$
    where $\{\hat \la_1, \ldots, \hat \la_s \}$ are the eigenvalues of $\Sigma_{\hat P}$. 
    For their ``Condition 2'', we have for all $x \in \cX$, since $\|\hat P\|\leq 1$ and the kernel is bounded,
    \begin{align}
    \left\|\Sigma_{\hat P, \la_*}^{-1/2}\hat{P}\phi(x)\right\|_{\mathcal{H}} \leq \left\|\Sigma_{\hat P, \la_*}^{-1/2}\right\|\|\hat{P}\|\|\phi(x)\|_{\mathcal{H}} \leq \frac{\kappa}{\sqrt{\la_*}}.   \nonumber
    \end{align}
    Hence we take $\rho_{\la_*} = \kappa(\tilde{d}_{\la_*}\la_*)^{-1/2}$ and we have $\rho_{\la_*}^2\tilde{d}_{\la_*} = \kappa^2/\la_*.$ We conclude by Lemma 24 \cite{hsu2012random}. 
\end{proof}
In the next two propositions we omit the index $i \in [N]$ as it applies for any task source task. Let $\cF, \cG$ be separable Hilbert spaces. Let $A: \cH \to \cF$, $B: \cH \to \cG$ be bounded operator. The next result provides a high probability bound on \[\|A\left(\Sigma - \hat \Sigma\right)B^*\|_{HS}.\]

\begin{prop}
    \label{prop:general_cov_concentration} Let $C_A, C_B, \sigma > 0$ be constants such that $\E\|A\phi(X)\|_{\cF}\|B\phi(X)\|_{\cG}\leq \sigma^2$, $\|A\phi(X)\|_{\cF} \leq C_A$ and $\|B\phi(X)\|_{\cG} \leq C_B$ almost surely, then 
    \begin{equation*}
        \P\left(\|A\left(\Sigma - \hat \Sigma\right)B^*\|_{HS} \geq t \right) \leq 2 \exp \left(-\frac{t^2n}{4C_AC_B(2\sigma^2 + t)}\right).
    \end{equation*}
    Alternatively, we get for all $\tau \geq 0$, with probability at least $(1-2e^{-\tau})_+$,
    \begin{equation*}
        \|A\left(\Sigma - \hat \Sigma\right)B^*\|_{HS} \leq \sqrt{\frac{8\tau \sigma^2 C_AC_B}{n}} + \frac{4\tau C_AC_B}{n},
    \end{equation*}
    where $C$ is a universal constant.
\end{prop} 
\begin{proof}
$$
A\left(\Sigma - \hat \Sigma\right)B^* = \frac{1}{n} \sum_{j=1}^n h_{x_j} \otimes l_{x_j}  - \E[h_{X} \otimes l_{X}],
$$
where $h_{x} \doteq A\phi(x)$ and $l_{x} \doteq B\phi(x)$. One one hand, almost surely
$$
\|h_{X} \otimes l_{X} \| = \|h_{X}\|_{\cF}\|l_{X}\|_{\cG} \leq C_AC_B.
$$
On the other hand,
$$
\E\|h_{X} \otimes l_{X} \|^2 \leq C_AC_B\E\|h_{X}\|_{\cF}\|l_{X}\|_{\cG} \leq \sigma^2 C_AC_B.
$$
Hence exploiting Corollary~\ref{corr:pinelis_3} in the Hilbert space $HS$ with $b = C_AC_B$ and $v^2= \sigma^2 C_AC_B$, we get 
\begin{equation}
\P\left(\|A\left(\Sigma - \hat \Sigma\right)B^*\|_{HS} \geq t \right) \leq 2 \exp \left(-\frac{t^2n}{4C_AC_B(2\sigma^2 + t)}\right).\nonumber
\end{equation}
Then by Proposition \ref{prop:bound_conversion}, for all $\tau \geq 0$, with probability at least $(1-2e^{-\tau})_+$,
\[
\left\|A\left(\Sigma - \hat \Sigma\right)B^*\right\|_{HS} \leq \sqrt{\frac{8\tau C_AC_B\sigma^2}{n}} + \frac{4\tau C_AC_B}{n}.
\]
\end{proof}
As a special case we obtain a high probability bound on the ``whitened'' difference between the population and empirical covariance on $\cH$ in Hilbert-Schmidt norm.

The following bound is a Bernstein-like concentration inequality for Hilbert space-valued random variables. It can be deduced from Corollary 1 \cite{pinelis1986remarks}.

\begin{theorem}\label{th:pinelis}
Let $H$ be a separable Hilbert space and $X_1, \ldots, X_n$ be independent random variables with values in $H$. If for some constants $v,b > 0$,
for all $j \in [n]$
\begin{equation*}
\E\left\|X_j - \E[X_j]\right\|^m_H \leq \frac{1}{2}m! v^2 b^{m-2}, \quad m=2,3, \cdots
\end{equation*}
Then 
\begin{equation*}
    \P\left(\left\|\frac{1}{n} \sum_{j=1}^n X_j - \E[X_j]\right\| \geq t \right) \leq 2 \exp \left(-\frac{t^2n}{2v^2+2bt}\right)
\end{equation*}
\end{theorem} 
\begin{corollary}\label{corr:pinelis_3}
Let $H$ be a separable Hilbert space and $X_1, \ldots, X_n$ be independent random variables with values in $H$. If for some constants $v, b > 0$, for all $j \in [n]$, $\|X_j\|_{H} \leq b$ almost surely and $\E\|X_j\|_{H}^2 \leq v^2$, 
\begin{equation*}
    \P\left(\left\|\frac{1}{n} \sum_{j=1}^n X_j - \E[X_j]\right\| \geq t \right) \leq 2 \exp \left(-\frac{t^2n}{8v^2+4bt}\right)
\end{equation*}
\end{corollary}
\begin{proof}
For all $j \in [n]$
$$
\E\left\|X_j - \E[X_j]\right\|^m_H \leq 2^{m-1} \left(\E\left\|X_j\right\|^m_H + \left\|\E[X_j]\right\|^m_H  \right) \leq 2^{m} \E\left\|X_j\right\|^m_H \leq \frac{1}{2}m! (2v)^2 (2b)^{m-2},
$$
hence using Theorem~\ref{th:pinelis} with $v^2 = 4v^2$ and $b = 2b$ gives the result.
\end{proof}

\begin{theorem}[Bounded concentration in Hilbert spaces] \label{th:hoeffding_hilbert}
    Suppose that $(X_i)_{i=1}^n$ are zero-mean independent random variables with values in a Hilbert space $(H,\langle \cdot,\cdot \rangle)$ and such that $\max_{i=1,\ldots,n}\|X_i \| \leq C < \infty$. Then for all $t \geq 0$,
    \begin{equation*} 
    \P\left(\left\|\frac{1}{n} \sum_{i=1}^n X_i\right\| \geq t \right) \leq 2e^{-\frac{nt^2}{2C^2}}.
    \end{equation*}
\end{theorem}
\begin{proof} 
    The inequality can be deduced from Theorem 3.5 \cite{pinelis1994optimum}. Their result applies to martingales $(Z_j)_{j \geq 0}$ of Bochner-integrable random vectors in a $(2,D)-$smooth separable Banach space $(\cX,\|\cdot\|)$. A Banach space $(\cX,\|\cdot\|)$ is $(2,D)-$smooth if for all $(x,y) \in \cX^2$,
    $$
        \|x+y\|^2 + \|x-y\|^2 \leq 2\|x\|^2 + 2D^2\|y\|^2
    $$
    In particular, any Hilbert space is $(2,1)-$smooth by the parallelogram identity. Theorem 3.5 \cite{pinelis1994optimum} states that if the increments of the martingale $(Z_j)_{j \geq 0}$ are such that $\sum_{j=1}^{\infty}\|Z_{j} - Z_{j-1}\|_{\infty}^2 \leq b_{*}^2 $ for some $b_* > 0$. Then for all $r \geq 0$,
    \begin{equation} \label{eq:D2_concentration}
        \P\left(\sup_{j \geq 0}\|Z_j\| \geq t \right) \leq 2\exp\left\{-\frac{t^2}{2D^2b_*^2}\right\}.
    \end{equation}
    Let us fix $n \geq 1$, and consider a sequence $(X_i)_{i=1}^n$ of zero-mean independent random variables with values in a Hilbert space $(H,\langle \cdot,\cdot \rangle)$ such that $\|X_i\|_{\infty} \leq C < \infty$  for all $i=1,\ldots,n$. Then $(Z_j)_{j \geq 0}$ such that 
    $$
    Z_0 = 0, \quad Z_j = \sum_{i=1}^j X_i \quad j=1,\ldots,n, \quad Z_j = Z_n \quad j>n,
    $$
    is a martingale on $H$ and its increments $d_j\doteq Z_{j} - Z_{j-1}$ satisfies 
    $$
    d_j = X_j \quad j=1,\ldots,n, \quad d_j = 0 \quad j>n,
    $$
    hence,
    $$
        \sum_{j=1}^{\infty}\|d_j\|_{\infty}^2 = \sum_{j=1}^{n}\|X_j\|_{\infty}^2 \leq nC^2.
    $$
    Therefore, applying Eq.~\eqref{eq:D2_concentration} to $(Z_j)_{j \geq 0}$ with $\cX = H$, $D=1$ and $b_*^2 = nC^2$ leads to, for all $t \geq 0$,
    \begin{equation*} 
        \P\left(\left\|\sum_{i=1}^n X_i\right\| \geq t \right) = \P\left(\left\|Z_n\right\| \geq t \right) \leq \P\left(\sup_{j \geq 0}\|Z_j\| \geq t \right) \leq 2\exp\left\{-\frac{t^2}{2nC^2}\right\}.
    \end{equation*}
    Rescaling by $1/n$ gives the final result. 
\end{proof}

\section{Additional Experimental Results} \label{sec:appendix_exp} ~

\begin{figure}[h]
    \centering
    \begin{minipage}[b]{0.45\textwidth}
        \centering
        \includegraphics[width=\textwidth]{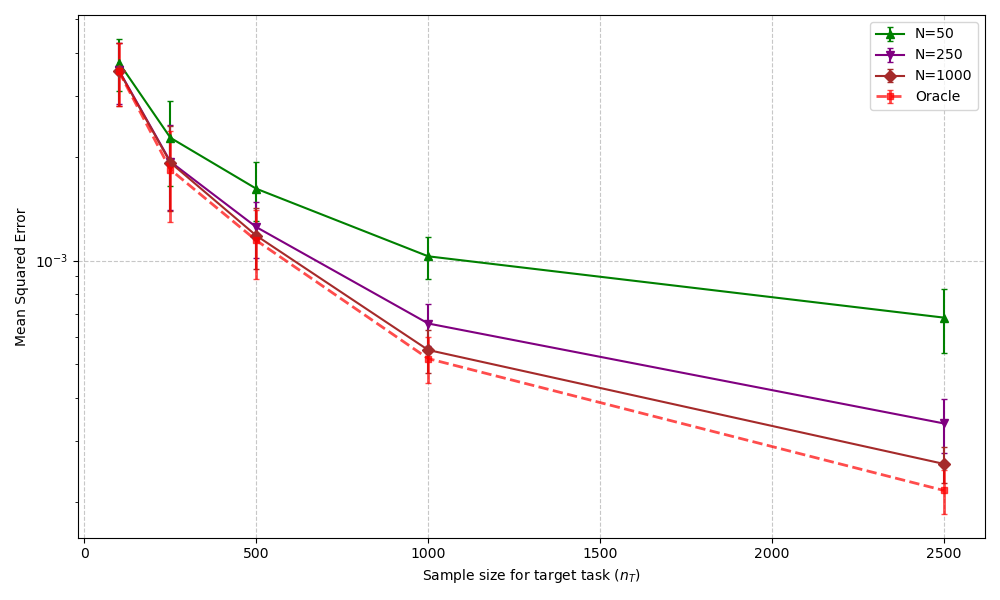}
    \end{minipage}
    \begin{minipage}[b]{0.45\textwidth}
        \centering
     \includegraphics[width=\textwidth]{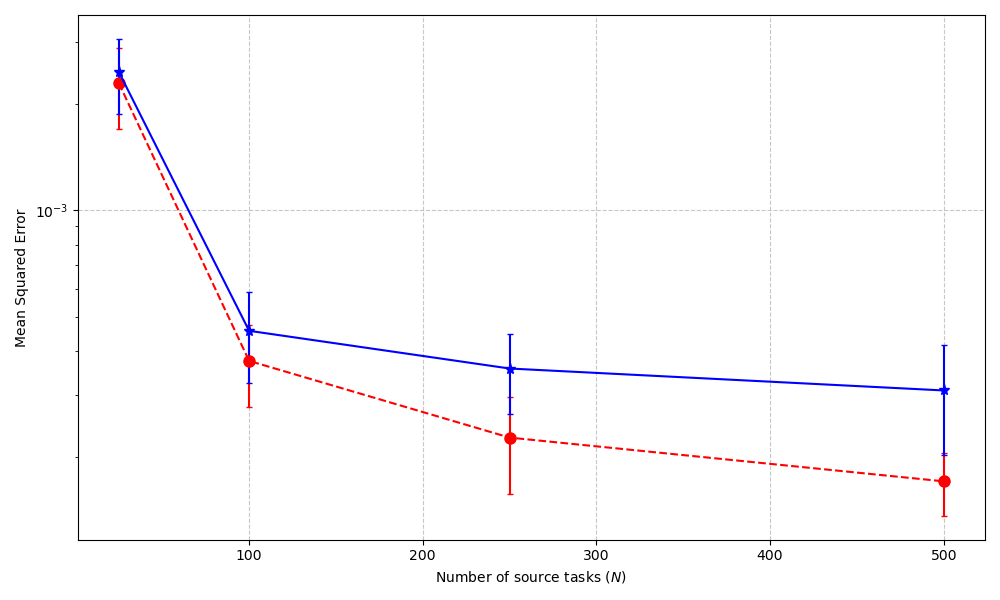}
    \end{minipage}
    \caption{\textbf{(Left)} Meta Learning versus Oracle: Comparison of the squared excess risk on the target task for the oracle estimator $\hat{f}_{\text{oracle}}$ (dotted red line) and the meta learning estimator $\hat{f}_{T,\lambda_*}$ trained with different number of tasks $N$ (solid lines). $x-$axis represents the size of the dataset for the target task $(n_T)$. \textbf{(Right)} Effect of under-regularization: Comparison of the squared excess risk of the meta learning estimator trained with $\lambda = (nN)^{-\frac{2}{5}}$ (red dotted line) and $\lambda = n^{-\frac{2}{5}}$ (blue solid line). $x-$axis represents the number of source tasks $(N)$. For both figures $n=500$, $s=50$ and results are averaged over $20$ generations of the source and target tasks.}
    \label{fig:experiment_additional}
\end{figure}

\end{document}